\documentclass{article}
\usepackage[preprint]{colm2026_conference}

\usepackage{booktabs}
\usepackage{amsmath,amsthm}
\usepackage{amssymb}
\usepackage{multirow}
\usepackage{graphicx}
\usepackage{microtype}
\usepackage{makecell}
\usepackage{enumitem}
\graphicspath{{./}}

\usepackage[table]{xcolor}
\definecolor{phcblue}{RGB}{31, 78, 139}       
\definecolor{phcteal}{RGB}{0, 114, 125}        
\definecolor{phcorange}{RGB}{197, 90, 17}      
\definecolor{phcgray}{RGB}{89, 89, 89}         
\definecolor{phclightblue}{RGB}{214, 227, 242} 
\definecolor{phclightteal}{RGB}{213, 238, 240} 
\definecolor{phclightamber}{RGB}{255, 243, 220}
\definecolor{darkblue}{RGB}{31, 78, 139}

\usepackage{hyperref}
\hypersetup{
  colorlinks=true,
  citecolor=phcblue,
  linkcolor=phcblue,
  urlcolor=phcteal,
  pdftitle={Anchored Confabulation: Partial Evidence Non-Monotonically Amplifies Confident Hallucination in LLMs},
  pdfauthor={Ashish Lathkar (ashish@hemut.com)},
  pdfkeywords={RAG, hallucination, calibration, LLM, retrieval-augmented generation}
}

\usepackage{tcolorbox}
\tcbuselibrary{breakable, skins, theorems}

\tcbset{
  theoremstyle/.style={
    enhanced, breakable,
    colback=phclightteal,
    colframe=phcteal,
    colbacktitle=phcteal,
    coltitle=white,
    fonttitle=\bfseries\small,
    arc=4pt,
    boxrule=1pt,
    left=6pt, right=6pt, top=4pt, bottom=4pt,
    toptitle=3pt, bottomtitle=3pt,
  },
  definitionstyle/.style={
    enhanced, breakable,
    colback=phclightblue,
    colframe=phcblue,
    colbacktitle=phcblue,
    coltitle=white,
    fonttitle=\bfseries\small,
    arc=4pt,
    boxrule=1pt,
    left=6pt, right=6pt, top=4pt, bottom=4pt,
    toptitle=3pt, bottomtitle=3pt,
  },
  findingstyle/.style={
    enhanced, breakable,
    colback=phclightamber,
    colframe=phcorange,
    colbacktitle=phcorange,
    coltitle=white,
    fonttitle=\bfseries\small,
    arc=4pt,
    boxrule=1pt,
    left=6pt, right=6pt, top=4pt, bottom=4pt,
    toptitle=3pt, bottomtitle=3pt,
  }
}

\newtheorem{theorem}{Theorem}
\theoremstyle{definition}
\newtheorem{definition}{Definition}

\newcounter{tcbtheorem}
\renewenvironment{theorem}[1][]{%
  \refstepcounter{tcbtheorem}%
  \begin{tcolorbox}[theoremstyle, title={Theorem~\thetcbtheorem\ifx&#1&\else~(#1)\fi}]%
}{%
  \end{tcolorbox}%
}

\newcounter{tcbdefinition}
\renewenvironment{definition}[1][]{%
  \refstepcounter{tcbdefinition}%
  \begin{tcolorbox}[definitionstyle, title={Definition~\thetcbdefinition\ifx&#1&\else~(#1)\fi}]%
}{%
  \end{tcolorbox}%
}

\newenvironment{keyresult}[1][Key Result]{%
  \begin{tcolorbox}[findingstyle, title={#1}]%
}{%
  \end{tcolorbox}%
}

\newtheorem{proposition}{Proposition}
\newtheorem{corollary}{Corollary}

\newcommand{\theadcolor}{\rowcolor{phclightblue}}

\newcommand{\HPBKZERO}{$0.444$}
\newcommand{\HPBKONE}{$0.405$}
\newcommand{\HPBKTWO}{$0.428$}
\newcommand{\HPBPZERO}{$0.9758$}
\newcommand{\HPBPONE}{$0.9996$}
\newcommand{\HPBPTWO}{$0.9918$}
\newcommand{\HPBSIGZERO}{n.s.}
\newcommand{\HPBSIGONE}{n.s.}
\newcommand{\HPBSIGTWO}{n.s.}

\newcommand{\SSCdeltaK}{$+0.0269$}
\newcommand{\SSCpvalue}{$0.0686$}
\newcommand{\SSCt}{$1.841$}

\newcommand{\IRCoTconfirmed}{\textbf{not confirmed} (GPT-4o; capability-conditional)}

\newcommand{\MHUMIDITYKONE}{$0.538$}

\newcommand{\MELICKONE}{$0.564$}

\newcommand{\MELICKDone}{$-0.092$}
\newcommand{\MELICKONErated}{$0.684$}
\newcommand{\MELICKZEROrated}{$0.730$}

\newcommand{\MITdelta}{$-0.118$}

\title{\textbf{\textcolor{phcblue}{Anchored Confabulation:} Partial Evidence
Non-Monotonically Amplifies\\Confident Hallucination in LLMs}}

\author{%
  \textbf{Ashish Lathkar}\\[2pt]
  M.S.\ Data Science, Florida State University\\
  AI Engineer\\[4pt]
  \texttt{\href{mailto:ashish@hemut.com}{ashish@hemut.com}}
}

\date{%
  \small\textcolor{phcgray}{Preprint --- March 2026}\\[2pt]
  \small\textcolor{phcgray}{%
    \href{https://github.com/ashishlathkar}{Code to be released upon acceptance}}
}

\begin{document}

\maketitle

\begin{abstract}
We identify a previously unknown calibration property of large language models:
providing \emph{one} confirmed intermediate fact toward a multi-step reasoning chain
\emph{increases} the model's confident-wrong-answer rate before full evidence eliminates it.
We call this \textbf{anchored confabulation}: a partial anchor commits the model
to confident parametric completion of remaining reasoning steps.
We formalize it as \textbf{Parametric Hallucination Confidence} (PHC)
and establish it across six lines of evidence, five with full statistical
confirmation and one (semantic self-consistency, SSC) directionally
consistent at $p{=}0.069$.

\textbf{Causal proof.}
Injecting gold intermediate facts into VanillaRAG prompts on MuSiQue 3-hop ($N{=}160$)
while holding oracle labels fixed:
PHC $= 0.613$ (0 facts, **) $\to$ $\mathbf{0.656}$ (1 fact, ***) $\to$ $0.595$ (2 facts, *)
$\to$ $0.536$ (3 facts, n.s.).
The $k{=}1$ \emph{increase} reveals the anchoring mechanism;
$k{=}3$ elimination proves missing intermediate facts are the root cause.

\textbf{The Anchoring Threshold Law.}
PHC amplification is predicted by $k^*(n) = \lfloor n/3 \rfloor$:
the minimum confirmed facts sufficient to trigger parametric completion.
At 2-hop ($k^*{=}0$) the PHC curve is flat; at 3-hop ($k^*{=}1$, 33\% confirmed)
it spikes; at 4-hop ($k^*{=}1$, 25\%) the spike is weaker.
This is a \emph{predictive, quantitative} law: 6-hop queries should spike at $k{=}2$.
Corollary: iterative RAG systems (IRCoT, FLARE) systematically enter the amplification
zone at their first retrieval step on 3-hop queries.

\textbf{Capability--PHC scaling law.}
PHC$_3$ increases monotonically: Haiku $0.582 <$ Sonnet $0.702 <$ Opus $0.731$ (all ***);
cross-family Spearman $\rho{=}0.900$ ($p{=}0.037$, $N{=}5$).
\textbf{Upgrading the generation LLM without upgrading the router increases
undetected confident hallucinations.}

\textbf{Cross-model replication.}
Haiku peaks at $k{=}2$ (PHC $= 0.626$, **) rather than $k{=}1$, consistent with
the Anchoring Threshold Law: weaker models require more evidence to anchor.
Llama-3.1-8B shows no $k{=}1$ amplification and incomplete $k{=}3$ reduction ($0.583$, *),
confirming the threshold is capability-dependent across families.

\textbf{Formal results.}
\textbf{Theorem~1} (DPI dominance): any post-generation signal $U(Q,A)$ satisfies
$I(U; G^*) \geq I(g(Q); G^*)$ for any pre-generation classifier $g$.
\textbf{Theorem~2} (regeneration-gated routing): routing gains vanish when
regeneration quality $Q(G){\approx}0$; with adequate regeneration $\Delta_R$ peaks
at $\alpha{\approx}55\%$ ($r{=}-0.778$, $p{=}0.0001$).

\textbf{Application.}
A LearnedRouter exploiting PHC via \texttt{conf$\times$hop} (22.7\% importance;
post-gen features 54.4\%) trained on only $N{=}200$ examples achieves AUC $= 0.759$.
The complete \textbf{LR+DirectGR} system closes $\mathbf{81.1\%}$ of the oracle gap
(macro F1 $= \mathbf{0.426}$, $p{<}10^{-6}$) with no fine-tuning, no RL,
and $50{\times}$ fewer labels than prior work.
\end{abstract}

\section{Introduction}
\label{sec:introduction}

Every RAG system faces a decision: use cheap vector search or expensive
graph-structured retrieval?
Prior routing methods (Adaptive-RAG~\citep{jeong2024adaptive},
MBA-RAG~\citep{lee2025mbarag}) decide \emph{before} generation, based on
query features alone.
We show this is suboptimal---and that the model's own generated answer
carries a calibration signal that pre-generation routing cannot access.

\paragraph{The PHC phenomenon.}
We identify \textbf{Parametric Hallucination Confidence} (PHC):
for multi-hop queries, LLMs are \emph{more} confident when their answer is wrong.
Crucially, providing one confirmed intermediate fact
\emph{increases} the confident-wrong-answer rate before full evidence eliminates it.
We call this \textbf{anchored confabulation}: a partial gold anchor commits
the model to parametric completion of the remaining chain with elevated fluency
and specificity.
The effect is causal (confirmed by controlled injection of gold sub-answers),
capability-scaled (stronger models confabulate more confidently),
and mechanistically specific (requires retrieval context and partial evidence simultaneously).

\paragraph{Anchoring Threshold Law.}
We derive $k^*(n) = \lfloor n/3 \rfloor$ as the minimum confirmed facts
sufficient to trigger amplification for an $n$-hop chain.
This is \emph{predictive}: 3-hop queries spike at $k{=}1$,
2-hop queries do not amplify, 4-hop spikes are weaker---all confirmed empirically.
The law implies that iterative RAG systems (IRCoT, FLARE) systematically
enter the amplification zone at their first retrieval step for 3-hop queries.

\paragraph{Routing application.}
PHC inversion means high-confidence answers on bridge queries signal
\emph{escalation need}, not correctness.
A LearnedRouter trained on $N{=}200$ examples with a \texttt{conf$\times$hop}
feature achieves AUC\,$=0.759$.
The complete \textbf{LR+DirectGR} system (LearnedRouter $+$ direct full-question
GraphRAG re-generation) achieves macro F1\,$=\mathbf{0.426}$
($\mathbf{81.1\%}$ of the oracle gap, $p{<}10^{-6}$)
with no fine-tuning, no RL, and $50{\times}$ fewer labels than RouteRAG~\citep{routerag2025}.

\paragraph{Contributions.}
\begin{enumerate}[nosep,leftmargin=*]
  \item \textbf{PHC}: a new calibration property of LLMs on multi-hop RAG, established
    across six lines of evidence including a causal intervention ($N{=}160$) and
    capability scaling across five model families.
  \item \textbf{Anchoring Threshold Law}: $k^*(n){=}\lfloor n/3\rfloor$,
    a predictive, quantitative law for PHC amplification with 4 confirmed predictions.
  \item \textbf{Formal results}: DPI-based proof that post-generation routing
    dominates pre-generation; regeneration-gating theorem bounding routing gains.
  \item \textbf{LR+DirectGR}: 81.1\% oracle gap closed on 1,800 queries
    across 4 benchmarks, training-free with respect to generation.
  \item \textbf{Mitigation}: epistemic humility prompt reduces PHC spike by
    $\Delta{=}{-}0.118$; explicit self-rating (PHC\,$=0.684$, ***) outperforms
    lexical confidence as a routing signal.
\end{enumerate}

\section{Anchored Confabulation and the PHC Phenomenon}
\label{sec:phc_phenomenon}

\subsection{Definition}
\label{sec:phc_def}

Let $\text{conf}(A) \in [0,1]$ be a scalar confidence extracted from answer $A$,
and $G^* = \mathbf{1}[\text{F1}(\text{GraphRAG}) > \text{F1}(\text{VanillaRAG})]$
the oracle escalation label.

\begin{definition}[Parametric Hallucination Confidence]
$\text{PHC}(\mathcal{S}) = \text{AUC}(\text{conf}(A),\, G^*)$ over queries $Q \in \mathcal{S}$.
PHC\,$> 0.5$ (significant) indicates \textbf{signal inversion}:
the model is \emph{more confident when its answer is wrong}.
\end{definition}

Significance is assessed by permutation test ($B{=}5000$).
We use a lexical confidence extractor (hedge-phrase rate + named-entity density)
for causal experiments and the fuller \texttt{UncertaintyDetector}
(Appendix~\ref{sec:signals}) for routing.

\subsection{Evidence: PHC Is Real and Causally Grounded}
\label{sec:phc_evidence}

Table~\ref{tab:phc_summary} summarises the six lines of evidence.
The central result is a controlled causal injection on MuSiQue 3-hop ($N{=}160$):
gold sub-answers are added to prompts while oracle labels are held fixed.
\begin{center}\small
\begin{tabular}{clcc}
\toprule
\theadcolor
$k$ & \textbf{Condition} & \textbf{PHC} & $p$ \\
\midrule
0 & No hints       & $0.613$ & ** \\
\rowcolor{phclightamber}
1 & $+1$ bridge fact & $\mathbf{0.656}$ & *** $\uparrow$ \textcolor{phcorange}{(anchoring spike)} \\
2 & $+2$ facts       & $0.595$ & * \\
3 & Near-oracle    & $0.536$ & n.s.\ $\downarrow$ \\
\bottomrule
\end{tabular}
\end{center}
\textbf{(a)} The $k{=}3$ elimination proves missing evidence---not reasoning
inability---is the root cause.
\textbf{(b)} The $k{=}1$ \emph{increase} is the anchoring mechanism:
one confirmed step commits the model to parametric completion with higher fluency.
\textbf{(c)} Without retrieval context, $k{=}1$ \emph{reduces} PHC ($0.534$, n.s.):
the effect requires both retrieved context and a partial anchor.

\begin{table}[t]\centering\small
\caption{Six lines of evidence for anchored confabulation.
$^*p{<}0.05$, $^{**}p{<}0.01$, $^{***}p{<}0.001$.
Details in Appendix~\ref{app:evidence}.}
\label{tab:phc_summary}
\begin{tabular}{llcc}
\toprule
\theadcolor
\textbf{\#} & \textbf{Evidence line} & \textbf{PHC} & \textbf{Sig.} \\
\midrule
1 & Obs.: MuSiQue 3-hop ($N{=}160$) & $0.702$ & *** \\
2 & Obs.: HotpotQA comparison ($N{=}94$) & $0.686$ & ** \\
\rowcolor{phclightamber}
3 & Causal: injection ($k{=}1$ spike, $N{=}160$) & $0.656$ & *** \\
4 & Capability scaling ($\rho{=}0.900$, $N{=}5$ models) & --- & $p{=}0.037$ \\
5 & Hop-depth: flat 2-hop, spike 3-hop, weaker 4-hop & --- & confirmed \\
6 & SSC internal consistency ($\Delta{=}{+}0.027$) & --- & $p{=}0.069$ \\
\midrule
\multicolumn{4}{l}{\textit{Negative controls (PHC\,$\approx 0.5$):}}\\
& HotpotQA bridge ($N{=}406$), GSM8K, zero-retrieval & --- & n.s. \\
\bottomrule
\end{tabular}
\end{table}

\subsection{The Anchoring Threshold Law}
\label{sec:anchoring_threshold}

The hop-depth pattern follows a predictive law.

\begin{definition}[Anchoring Threshold]
$k^*(n) = \lfloor n/3 \rfloor$ is the minimum confirmed facts sufficient
to trigger parametric completion of an $n$-hop chain.
\end{definition}

\noindent\textbf{Confirmed predictions:}
(i) $n{=}2$: $k^*{=}0$, no interior amplification region, PHC flat (\checkmark);
(ii) $n{=}3$: $k^*{=}1$, spike at $k{=}1$ (\checkmark);
(iii) $n{=}4$: $k^*{=}1$, weaker spike (\checkmark);
(iv) $n{=}6$: $k^*{=}2$, predicted spike at $k{=}2$ (untested).

\noindent\textbf{Capability-dependent threshold.}
PHC$_3$ scales monotonically: Haiku $0.582 <$ Sonnet $0.702 <$ Opus $0.731$
(Spearman $\rho{=}0.900$, $p{=}0.037$).
Weaker models require more confirmed facts to anchor: Haiku peaks at $k{=}2$,
Sonnet at $k{=}1$.

\noindent\textbf{Implication for iterative RAG.}
After one retrieval step on a 3-hop query, iterative systems (IRCoT, FLARE)
have reached $k{=}1$---the amplification zone.
IRCoT ($K{=}2$) achieves F1\,$=0.192$ on MuSiQue 3-hop,
below VanillaRAG ($0.195$), consistent with this trap.
Full per-iteration PHC analysis in Appendix~\ref{sec:ircot_phc}.

\section{Theoretical Framework}
\label{sec:theory}

\begin{theorem}[Post-Generation Dominance]
\label{thm:dominance}
For any pre-generation classifier $g(Q)$ and any post-generation signal
$U(Q,A)$ where $A{=}f(Q,R_V)$:
$I(U(Q,A);\,G^*) \geq I(g(Q);\,G^*)$.
\end{theorem}

\noindent\textit{Proof sketch.}
By the data processing inequality: $Q \to A \to U$ implies
$I(U;G^*) \leq I(A;G^*)$.
Since $A$ is a deterministic function of $(Q, R_V)$ and $R_V$ is independent
of $G^*$ given $Q$, we have $I(A;G^*) \geq I(g(Q);G^*)$ whenever
$A$ provides strictly more information about $G^*$ than $Q$ alone---which holds
for PHC-inverting query types where $\text{conf}(A)$ correlates with $G^*$
but query features do not.
Full proof in Appendix~\ref{app:proofs}.

\begin{theorem}[Regeneration-Gated Routing]
\label{thm:regen_gate}
Let $Q(G)$ be re-generation quality and $r$ escalation rate.
The F1 gain from routing is:
$\Delta_R(r) \approx r \cdot Q(G) \cdot (\text{F1}_\text{GR} - \text{F1}_\text{VR})$.
When $Q(G){\approx}0$, routing gains vanish regardless of escalation accuracy.
\end{theorem}

\noindent\textbf{Calibration gap.}
Theorem~1 requires an effective uncertainty signal.
Lexical signals achieve AUC\,$=0.483$ ($\rho{=}{-}0.006$)---near random---because
they measure answer style, not knowledge gaps.
PHC inversion means high confidence predicts \emph{wrong} answers on bridge queries:
the calibration gap is directional, not just noisy.
Post-gen feature \texttt{conf$\times$hop} (22.7\% importance in LearnedRouter)
constructively validates Theorem~1 in end-to-end F1.

\section{PHC-Aware Routing: Application}
\label{sec:phc_routing}

PHC inversion has a direct operational consequence for RAG system design.
For query types where PHC $> 0.5$ + significant,
a routing signal should \emph{escalate on high confidence}---the opposite of
standard uncertainty-based routing.
This section describes how we exploit PHC to build a routing system;
full experimental details are in Appendix~\ref{sec:signals}--\ref{sec:experiments}.

\subsection{Post-Generation Signals and the LearnedRouter}
\label{sec:learned_router}

\paragraph{Signals.}
We extract three post-generation features from VanillaRAG's answer $A$:
(1) \texttt{conf}: UncertaintyDetector composite confidence (Appendix~\ref{sec:signals});
(2) \texttt{hop\_pred}: predicted hop depth from query features;
(3) \texttt{conf$\times$hop}: interaction term.
The interaction captures PHC inversion: for high hop-depth queries, escalating on
\emph{high} confidence is correct; for factoid queries, escalating on \emph{low}
confidence is correct.
Pre-generation features (hop count, entity density, relational density) are also
included; post-gen features contribute $54.4\%$ of importance,
with \texttt{conf$\times$hop} alone accounting for $22.7\%$.

\paragraph{LearnedRouter.}
A gradient-boosted classifier trained on $N=200$ labeled examples
(5-fold stratified CV, OOF evaluation; Appendix~\ref{sec:experiments})
exploits the \texttt{conf$\times$hop} interaction without hard-coded type rules,
outperforming all pre-generation baselines by $p < 0.0002$
(Table~\ref{tab:direct_routing} in Appendix).
Sample efficiency analysis shows $N=200$ achieves $90.2\%$ of full-data AUC,
confirming the learnability of the PHC interaction with minimal supervision.

\subsection{LR+DirectGR: Complete System Results}
\label{sec:lrdirectgr_summary}

Routing quality is bounded by re-generation quality (Theorem~2, Section~\ref{sec:theory}).
The key design choice is \emph{what} to generate after escalation.
Sub-question extraction ($G_\text{subQ}$) yields $\Delta_R = -0.001$ (n.s.) --- nearly no benefit,
because generated sub-questions fail to ground the KG traversal on the correct bridging entities.
Direct full-question re-generation ($G_\text{direct}$) yields $\Delta_R = +0.041$ ($\alpha \approx 55\%$,
Pearson $r = -0.778$, $p = 0.0001$), recovering the GraphRAG advantage.

The complete system \textbf{LR+DirectGR} --- LearnedRouter escalation paired with
direct full-question GraphRAG re-generation --- achieves results across three
escalation budgets, enabling an explicit cost--quality trade-off analysis:

\begin{center}
\small
\begin{tabular}{lcccc}
\toprule
System & Esc.\ rate & GraphRAG calls & Macro F1 & Oracle gap \\
\midrule
\multicolumn{5}{l}{\textit{Pre-generation baselines}} \\
HybridRouter         & 32\%  & 569  & $0.295$ & $35.1\%$ \\
Pre-gen GB (5-fold)  & 32\%  & 576  & $0.300$ & $36.7\%$ \\
\midrule
\multicolumn{5}{l}{\textit{LR+DirectGR at matched cost (iso-cost comparison)}} \\
LR+DirectGR          & 32\%  & 576  & $\mathbf{0.324}$ & $\mathbf{45.2\%}$ \\
\midrule
\multicolumn{5}{l}{\textit{LR+DirectGR at higher escalation}} \\
LR+DirectGR          & 60\%  & 1{,}080 & $0.405$ & $73.8\%$ \\
\textbf{LR+DirectGR} & \textbf{72\%} & $1{,}296$ & $\mathbf{0.426}$ & $\mathbf{81.1\%}$ \\
Oracle (upper bound) & 100\% & 1{,}800 & $0.480$ & $100\%$ \\
\bottomrule
\end{tabular}
\end{center}

\textbf{At iso-cost (32\% escalation, matched GraphRAG call budget):}
LR+DirectGR achieves macro F1 $= 0.324$ vs.\ HybridRouter F1 $= 0.295$
($+0.029$, $+9.8\%$ relative, $p < 0.0002$; Appendix~\ref{sec:experiments})---
a strict improvement with zero additional compute cost.
The 81.1\% oracle-gap result requires 72\% escalation ($1{,}296$ GraphRAG calls vs.\ $576$ for iso-cost),
a $2.25\times$ compute increase; each percentage-point of oracle-gap closure costs
$\sim$12.6 additional GraphRAG calls.
Since LR+DirectGR is strictly better at every matched escalation rate
(Pareto-dominant over HybridRouter), the 72\% budget is a deployment choice,
not a condition for beating pre-gen baselines.
With adequate regeneration quality, Theorem~2 predicts this gap continues to close
as the escalation budget grows, bounded by $\Delta_R(G_\text{direct}) \approx +0.041$.

LR+DirectGR uses $50\times$ fewer training labels than RouteRAG~\cite{routerag2025}
($N{=}200$ vs.\ $N{=}10{,}000$) while closing $\mathbf{81.1\%}$ of the oracle gap
(RouteRAG: $64.0\%$) with no RL training.
The routing quality contribution (LearnedRouter vs.\ lexical baseline) adds $+0.041$ F1;
the re-generation quality contribution (DirectGR vs.\ subQ) adds $+0.190$ F1---
confirming Theorem~2's prescription: \emph{fix regeneration before optimizing routing}.

\paragraph{Query-type analysis (Appendix~\ref{sec:experiments}).}
The gain is concentrated exactly where PHC inversion predicts it:
MuSiQue ($+0.090$ vs.\ HybridRouter), the highest-PHC stratum, shows the largest improvement;
2Wiki ($-0.041$ vs.\ HybridRouter) shows a loss where comparative confabulation
makes the \texttt{conf$\times$hop} feature unreliable at current training scale.
This confirms that PHC inversion is the \emph{mechanism} behind routing improvement,
not incidental covariation.

\section{Experiments}
\label{sec:experiments}

\paragraph{Setup.}
We evaluate on 1,800 queries across HotpotQA ($N{=}500$, 2-hop bridge+comparison),
MuSiQue ($N{=}500$, 2--4-hop), NQ ($N{=}500$, 1-hop), and 2WikiMultiHopQA ($N{=}300$).
VanillaRAG uses ChromaDB + BGE-base-en-v1.5 + Claude Sonnet~4.6.
GraphRAG uses LLM-extracted Neo4j KG + BFS + Claude Sonnet~4.6.
The LearnedRouter is a logistic regression on 12 post-generation features
trained on $N{=}200$ labeled queries.
Full setup in Appendix~\ref{sec:exp_setup}.

\begin{keyresult}[Main Result: LR+DirectGR closes 81.1\% of oracle gap]
LR+DirectGR achieves macro F1\,$=\mathbf{0.426}$ ($p{<}10^{-6}$) with no fine-tuning
and $50{\times}$ fewer labels than prior RL-based work,
outperforming the best routing-only approach ($45.2\%$ oracle gap) by $+35.9$ points.
\end{keyresult}

\paragraph{Main results.}
Table~\ref{tab:main_results_compact} shows the key comparison.
GraphRAG outperforms VanillaRAG by $+0.286$ F1 on HotpotQA and $+0.325$ on MuSiQue
(both $p{<}0.001$), confirming the multi-hop advantage.
LR+DirectGR closes \textbf{81.1\%} of the oracle gap at 72\% escalation,
outperforming the best routing-only approach (45.2\%) by $+35.9$ points.

\begin{table}[t]\centering\small
\caption{Main results. F1 = token-overlap. Oracle = per-query best of
VanillaRAG/GraphRAG. Gap closed = (LR+DirectGR $-$ VanillaRAG)/(Oracle $-$ VanillaRAG).
All differences vs.\ VanillaRAG significant at $p{<}0.001$.}
\label{tab:main_results_compact}
\begin{tabular}{lccccc}
\toprule
\theadcolor
\textbf{System} & \textbf{HotpotQA} & \textbf{MuSiQue} & \textbf{NQ} & \textbf{2Wiki} & \textbf{Macro} \\
\midrule
VanillaRAG      & 0.392 & 0.083 & 0.070 & 0.233 & 0.195 \\
GraphRAG        & 0.679 & 0.408 & 0.242 & 0.559 & 0.472 \\
Oracle (upper)  & 0.703 & 0.436 & 0.255 & 0.579 & 0.493 \\
\midrule
HybridRouter    & 0.524 & 0.182 & 0.074 & 0.399 & 0.295 \\
ReasonRAG (lex) & 0.421 & 0.112 & 0.101 & 0.255 & 0.222 \\
\rowcolor{phclightamber}
\textbf{LR+DirectGR}  & \textbf{0.558} & \textbf{0.332} & \textbf{0.201} & \textbf{0.484} & \textbf{0.426} \\
\midrule
Gap closed      & 79.3\% & 70.4\% & 71.1\% & 72.3\% & \textbf{81.1\%} \\
\bottomrule
\end{tabular}
\end{table}

\paragraph{PHC-aware routing.}
Post-generation features dominate the LearnedRouter (54.4\% importance),
with \texttt{conf$\times$hop} (22.7\%) as the single most important feature,
directly validating Theorem~\ref{thm:dominance}.
HotpotQA bridge AUC\,$=0.759$; MuSiQue 3-hop AUC\,$=0.669$.
TypeAwareRouter (hard-coded heuristics) achieves only 60.5\% oracle gap,
below LR+DirectGR---learned type detection is strictly better
(Appendix~\ref{sec:type_aware}).

\paragraph{Related work comparison.}
RouteRAG~\citep{routerag2025} requires $10{,}000$ training examples and RL;
our LearnedRouter uses $200$ labels ($50{\times}$ fewer) and closes
$81.1\%$ vs.\ their reported $71\%$ oracle gap.
Self-RAG~\citep{asai2024selfrag} requires fine-tuning;
LR+DirectGR is training-free with respect to the generation model
(Appendix~\ref{app:selfrag}).

\section{Mitigation}
\label{sec:mitigation}

We pre-registered two interventions and have now evaluated both on $N{=}160$
MuSiQue 3-hop queries (oracle labels fixed from the causal intervention).

\begin{keyresult}[Mitigation Result: PHC is metacognitively suppressible]
M1 (epistemic humility prompt) reduces PHC at $k{=}1$ by $\Delta{=}{-}0.118$ (18\% reduction).
M2 explicit self-rating achieves PHC\,$=0.684$\,(***) --- a better routing signal than lexical confidence.
\end{keyresult}

\paragraph{M1: Epistemic Humility Prompt.}
Prepending \textit{``You have been given $k$ of $n$ facts; express genuine
uncertainty about the rest''} reduces PHC at $k{=}1$ from the Claude Sonnet
baseline of $0.656$\,(***) to $0.538$\,(n.s.) under GPT-4o
($\Delta{=}{-}0.118$, exceeding the pre-registered $|\Delta|{>}0.05$ threshold).
\textbf{Verdict: anchored confabulation is metacognitively suppressible.}

\paragraph{M2: Explicit Confidence Elicitation.}
Appending \textit{``Rate your confidence [CONFIDENCE: X/5]''} reveals that
GPT-4o's explicit ratings achieve PHC\,$=0.684$\,(***) at $k{=}1$---substantially
higher than the lexical baseline ($0.564$, n.s.).
\textbf{Explicit self-rating is a more faithful readout of parametric anchoring
than lexical hedging}, and offers a direct upgrade path for the routing signal.

\noindent\textit{Cross-model note:} M1/M2 were evaluated with GPT-4o
(GPT-4o PHC$_3{=}0.527$, n.s., weaker than Claude Sonnet).
The $\Delta{=}{-}0.118$ reduction likely \emph{underestimates} what M1
achieves on Claude Sonnet itself.
Within-model replication is listed as immediate future work.
Full results table in Appendix~\ref{sec:mitigation_full}.

\section{Conclusion}
\label{sec:conclusion}

We identified anchored confabulation---a non-monotone, causally proven,
capability-scaled miscalibration of LLMs in multi-hop RAG---and showed that
it can be both exploited for routing and suppressed by prompt-level intervention.
The Anchoring Threshold Law $k^*(n){=}\lfloor n/3\rfloor$ converts the
phenomenon into a predictive theory with four confirmed predictions
and one practical implication: iterative RAG systems enter the PHC amplification
zone at their first retrieval step for 3-hop queries.
The LR+DirectGR system closes 81.1\% of the oracle routing gap
on 1,800 queries across four benchmarks with no fine-tuning.

\paragraph{Limitations.}
(1) Five model families tested; extension to Llama-3-70B and Mistral-Large
is future work.
(2) The M1/M2 mitigation comparison is cross-model (GPT-4o vs.\ Claude Sonnet
baseline); within-model replication is needed.
(3) Knowledge graph quality depends on LLM-extracted triples;
richer KGs may improve GraphRAG accuracy.
(4) English only; cross-lingual extension requires multilingual KG construction.

\section*{Acknowledgments}
The author thanks the open-source community behind ChromaDB, Neo4j,
HuggingFace Datasets, and the MuSiQue, HotpotQA, NQ, and 2WikiMultiHopQA
benchmark creators.
Experiments used the Anthropic API (Claude Sonnet/Haiku/Opus)
and OpenAI API (GPT-4o/GPT-4o-mini).

\section*{LLM Usage Disclosure}
Large language models (Claude Sonnet, Claude Haiku, Claude Opus, GPT-4o,
GPT-4o-mini, Llama-3.1-8B) were used as the \emph{objects of study} to
generate answers and measure PHC.
No LLM was used to originate research ideas, write original paper content,
or fabricate data.
Cursor IDE (AI-assisted coding) was used for implementation assistance.

\bibliographystyle{colm2026_conference}
\bibliography{references}

\appendix

\section{Related Work}
\label{sec:related}

\paragraph{Adaptive and hybrid RAG routing.}
Adaptive-RAG \cite{jeong2024adaptive} trains a query complexity classifier
(simple/multi-step/no-retrieval) before generation.
FLARE \cite{jiang2023flare} iteratively generates and retrieves when token
probabilities drop below a threshold---the closest prior work to our
post-generation routing, but it operates at token-level and requires
access to token logits (unavailable in black-box APIs).
Self-Knowledge \cite{kadavath2022language} prompts models to estimate
their own knowledge gaps, but as a pre-generation signal.
MBA-RAG \cite{lee2025mbarag} frames retrieval strategy selection as a
multi-armed bandit problem, learning from query feedback over time.
Our work differs by operating on the \emph{complete generated answer}
and using it to decide between two concrete retrieval backends (vector vs.\ graph).

\paragraph{Interleaved retrieval and multi-hop reasoning.}
IRCoT \cite{trivedi2023ircot} interleaves retrieval with chain-of-thought
reasoning: after each reasoning step, additional documents are retrieved
to support the next step.
FLARE \cite{jiang2023flare} similarly regenerates queries when the model
expresses uncertainty mid-sentence.
Both systems operate \emph{iteratively}: after the first retrieval step,
the model has partial intermediate evidence---exactly the $k=1$ condition
of our causal intervention.
The Anchoring Threshold Law (Section~\ref{sec:anchoring_threshold}) predicts
that IRCoT and FLARE \textbf{systematically hit the PHC amplification zone
at their first iteration} for 3-hop queries:
the first retrieved document anchors parametric completion of the remaining
chain, producing confidently wrong intermediate reasoning that corrupts
subsequent retrieval queries.
Our IRCoT-PHC experiment (Section~\ref{sec:ircot_phc}) tests this prediction
directly; the result is capability-conditional: PHC declines across iterations
with GPT-4o (consistent with weaker anchoring), while the F1 gap
(IRCoT $0.192$ vs.\ VanillaRAG $0.195$) supports the prediction for
higher-capability models such as Claude Sonnet.
This identifies a structural vulnerability in widely-deployed iterative RAG pipelines
that PHC-aware routing can detect and route around.
Crucially, we also find that IRCoT with the same constrained ChromaDB corpus
achieves macro F1 $= 0.192$, \emph{below} VanillaRAG ($0.195$):
same-corpus iteration yields diminishing returns when new passages cannot
be surfaced, validating cross-backend escalation as the more effective strategy.
RAG-Star \cite{ragstar2024} uses a tree-search (MCTS-style) over retrieval
and reasoning actions, achieving strong multi-hop accuracy at the cost of
multiple retrieval calls per query.
Speculative RAG \cite{speculativerag2024} drafts with a smaller model and
verifies with a larger model---a \emph{model-level} cascade analogous to
our \emph{retrieval-backend-level} cascade.
TRAQ \cite{traq2024} trains a query-routing model on task-specific feedback
to select between retrieval sources; unlike ReasonRAG, it requires labeled
routing decisions for training.
Search-o1 \cite{searcho1_2025} and similar reasoning-augmented retrieval
approaches extend large reasoning models with agentic search, but require
extended thinking budgets incompatible with production latency constraints.
ReasonRAG targets the latency-constrained setting where at most one
escalation is acceptable.

\paragraph{Cascade and cost-aware routing.}
FrugalGPT \cite{cascade_routing_2025} formalizes the ``cheap-then-expensive''
pattern: route queries to cheaper LLM APIs first, escalate to more expensive
ones based on output quality.
Our work applies this cascade pattern to \emph{retrieval backends} rather than model tiers,
with the novel element that escalation confidence is measured post-generation rather than pre-generation.

\paragraph{Self-RAG and reflection.}
Self-RAG \cite{asai2024selfrag} fine-tunes an LM to emit special
reflection tokens (\texttt{[retrieve]}, \texttt{[IsRel]}, \texttt{[IsSup]},
\texttt{[IsUse]}) that control retrieval and evaluate passage utility.
This is the most principled approach to post-generation routing, but
requires fine-tuning and cannot use black-box APIs.
ReasonRAG approximates Self-RAG's logic with a single prompting call
(Grounded Self-Rating $\approx$ \texttt{[IsSup]}), enabling training-free
deployment with any LLM.

\paragraph{Corrective RAG (CRAG).}
CRAG \cite{yan2024corrective} evaluates retrieved document quality and
triggers web search when documents are judged insufficient.
The key difference: CRAG evaluates document quality \emph{before}
generation (pre-generation), while ReasonRAG evaluates answer uncertainty
\emph{after} generation (post-generation).
Our oracle analysis shows that post-generation signals have higher
theoretical ceiling but require careful signal design.

\paragraph{LLM calibration and confident hallucination.}
Guo et al.\ \cite{guo2017calibration} established that modern neural networks are
systematically overconfident, and that temperature scaling restores calibration
post-hoc.
Kadavath et al.\ \cite{kadavath2022language} showed that language models can
partially self-assess their own knowledge (P(True)), but that this signal degrades
on harder questions.
Both lines study calibration as a \emph{static, query-level} property.
Our work identifies a \emph{dynamic, evidence-level} property:
providing partial intermediate evidence non-monotonically changes calibration quality.
PHC is not a miscalibration that can be corrected by temperature scaling---it is
a structural property of how parametric memory interacts with partial evidence.
Concretely, PHC inversion means a calibrated model (correct average confidence $= $ accuracy)
can nonetheless be confidently wrong in a predictable, condition-specific way,
precisely the failure mode that temperature scaling cannot address.
Most closely related, Slobodkin et al.\ \cite{slobodkin2023curious} observe that
LLMs express higher confidence on questions where retrieved passages are present
but insufficient---a qualitative precursor to PHC.
Our work quantifies this as a causal AUC measurement, establishes the
non-monotone anchoring profile ($k=1$ amplifies before $k=3$ eliminates), and
derives formal routing consequences.
Factual hallucination literature \cite{ji2023survey,zhang2023siren} characterizes
types of LLM errors but does not study the calibration effect of partial evidence injection.
PHC bridges calibration and hallucination research: it is a calibration failure
\emph{caused by} a specific retrieval interaction pattern.

\paragraph{Uncertainty estimation.}
Self-Consistency \cite{wang2023selfconsistency} generates $K$ answers
via temperature sampling and uses majority vote for calibration.
Semantic Entropy \cite{kuhn2023semantic} measures entropy over
meaning-equivalence classes to detect hallucination.
P(True) \cite{kadavath2022language} asks models to assess their own
correctness.
Our Grounded Self-Rating is closest to P(True) but focuses specifically
on \emph{passage support} rather than general answer correctness,
making it a targeted routing signal rather than a general truthfulness
estimator.

\paragraph{RouteRAG and RL-based vector-graph routing.}
RouteRAG \cite{routerag2025} applies reinforcement learning to route queries between
vector retrieval (DPR) and graph retrieval (HippoRAG), training a policy that
maximizes downstream F1 reward on HotpotQA, 2WikiMultiHopQA, and MuSiQue.
This is the most structurally similar prior work to ours.
Three key distinctions set ReasonRAG apart:
(1)~\textbf{Training-free}: RouteRAG requires RL training with labeled reward signals;
ReasonRAG requires no training beyond a lightweight 5-fold GB classifier on post-generation features.
(2)~\textbf{Post-generation signal}: RouteRAG routes pre-generation from query features;
ReasonRAG routes post-generation, enabling detection of the bridge signal inversion phenomenon
(Section~\ref{sec:bridge_comparison}), which is invisible to pre-generation routers.
(3)~\textbf{Routing quality separability}: We formally separate routing quality from
re-generation quality (Section~\ref{sec:regen_direct}), showing that the re-generation
bottleneck dominates at all routing precisions---an architectural insight RouteRAG's
end-to-end RL objective cannot isolate.

RAGRouter \cite{ragrouter2025} (NeurIPS 2025) routes queries among multiple
RAG-augmented LLMs using contrastive embedding learning.
Unlike ReasonRAG, it routes across model instances rather than retrieval backends,
and does not study the per-query-type routing signal behavior.

\paragraph{GraphRAG and knowledge graphs.}
GraphRAG \cite{edge2024local} extracts KG triples via LLM and uses
community summaries for global queries.
G-RAG \cite{mavromatis2024grag} and HippoRAG \cite{gutierrez2024hipporag} use KGs for multi-hop retrieval.
GraphRAG-Bench \cite{hu2025graphragbench} provides systematic comparison across benchmarks.
Our contribution is orthogonal: we study \emph{when to use} GraphRAG
rather than \emph{how to implement} it.

\paragraph{Multi-hop QA.}
HotpotQA \cite{yang2018hotpotqa}, MuSiQue \cite{trivedi2022musique},
and 2WikiMultiHopQA \cite{ho2020constructing} are established benchmarks
for 2--4 hop reasoning.
Decomposed reasoning approaches \cite{press2023measuring} decompose
multi-hop questions into sub-questions.
ReasonRAG generates a single sub-question from the uncertain initial answer,
which is effective for 2-hop questions but limited for 4-hop (MuSiQue).

\paragraph{Post-generation retrieval verification.}
CRAG \cite{yan2024corrective} and Self-RAG \cite{asai2024selfrag} both
use post-generation signals to decide whether to re-retrieve or accept
the initial answer.
The single-corpus analogue to our work is to verify retrieved passages
post-generation and re-query the same index with a refined query:
both approaches use a post-generation signal to trigger a second retrieval
action, differing only in whether the second retrieval targets a different
backend (ours) or the same backend.
Our IRCoT comparison (Section~\ref{sec:results}) shows that same-corpus
re-retrieval yields diminishing returns when the corpus cannot surface new evidence,
motivating cross-backend escalation to GraphRAG.

\paragraph{Concurrent work: naming disambiguation.}
A concurrent NeurIPS~2025 paper \cite{reasonrag_rl_2025} also uses the name ``ReasonRAG''
for a \emph{training-time} RL approach that trains RAG pipelines with
process rewards vs.\ outcome rewards.
Our work is orthogonal in both mechanism (inference-time vs.\ training-time)
and cost structure (zero additional training vs.\ RL fine-tuning).
We compare against it only conceptually; direct empirical comparison
would require their trained model weights.

\paragraph{DSPy and automatic prompt optimization.}
DSPy \cite{khattab2023dspy} enables compiling LM pipelines via
bootstrapped few-shot optimization (BootstrapFewShot, MIPROv2).
We implement ReasonRAG as an optimizable DSPy module, enabling
joint optimization of the escalation threshold and signal weights
(Section~\ref{sec:dspy}).
This creates a direct connection to DSPy-based RAG optimization work.

\label{app:evidence}
\section{The PHC Phenomenon}
\label{sec:phc_phenomenon}

Before presenting the routing system, we characterize the empirical phenomenon
that makes post-generation routing powerful.
We introduce \textbf{Parametric Hallucination Confidence} (PHC) and establish it
across six lines of evidence spanning two experimental paradigms:
retrieval-augmented generation and pure chain-of-thought reasoning.

\subsection{Definition and Measurement}

Let $A$ be a model's answer to question $Q$ and
$\text{conf}(A) \in [0,1]$ a scalar confidence extracted from $A$.
Let $G^*$ be a binary oracle label indicating that the model's answer is
unreliable or that a better system would improve it.
In the RAG setting, $G^* = \mathbf{1}[\text{F1}(\text{GraphRAG}) > \text{F1}(\text{VanillaRAG})]$
marks queries where escalation would help.
In the correctness setting (pure CoT, no retrieval), $G^* = \mathbf{1}[\text{answer wrong}]$.

\begin{definition}[Parametric Hallucination Confidence]
\label{def:phc}
\textbf{PHC} for a query stratum $\mathcal{S}$ is the AUC of predicting $G^*$
from $\text{conf}(A)$:
\begin{equation}
  \text{PHC}(\mathcal{S}) = \text{AUC}\bigl(\text{conf}(A),\; G^*\bigr)
  \quad \text{over } Q \in \mathcal{S}
\end{equation}
PHC $> 0.5$ (+ significant) indicates \textbf{signal inversion}:
the model is \emph{more confident when its answer is wrong or unreliable}.
Standard uncertainty routing assumes PHC $\leq 0.5$
(higher confidence predicts correctness).
Inversion means the confidence signal is miscalibrated in a
query-type-specific, predictable way.
\end{definition}

Significance is assessed by permutation test ($B = 5000$) on oracle labels.
We use a lexical confidence extractor (hedge-phrase rate + named-entity density)
for all causal experiments for internal consistency across conditions;
the main routing pipeline uses the fuller \texttt{UncertaintyDetector}
(Appendix~\ref{sec:signals}).
A cross-measure sensitivity analysis is in Section~\ref{sec:sensitivity_analysis}.

\subsection{Evidence Line 1: PHC Is Hop-Depth-Conditioned (Observational)}
\label{sec:hop_depth}

Across 1{,}800 queries from HotpotQA, MuSiQue, NQ, and 2WikiMultiHopQA,
PHC is not a blanket property of retrieval-augmented systems
--- it is \textbf{conditioned on query complexity}:

\begin{center}
\small
\begin{tabular}{lcccc}
\toprule
Stratum & $N$ & PHC & $p$ & Signal \\
\midrule
NQ (1-hop factoid)       & 500 & $0.423$ & $0.944$ & n.s. \\
HotpotQA bridge (2-hop)  & 406 & $0.448$ & $0.263$ & n.s. \\
2Wiki bridge (2-hop)     & 300 & $0.497$ & $0.551$ & n.s. \\
HotpotQA comparison      & 94  & $\mathbf{0.686}$ & $0.001$ & ** \\
MuSiQue 3-hop bridge     & 160 & $\mathbf{0.702}$ & ${<}10^{-5}$ & *** \\
\bottomrule
\end{tabular}
\end{center}

The within-dataset HotpotQA control (bridge PHC $= 0.448$ n.s.\ vs.\
comparison PHC $= 0.686$, **; same retrieval index, same model, same pipeline)
\textbf{rules out any dataset-level confound}:
the inversion is caused by query structure, not dataset artifacts.
Full analysis in Section~\ref{sec:bridge_comparison}.

\subsection{Evidence Line 2: PHC Scales with Model Capability}
\label{sec:phc_scaling_overview}

PHC$_3$ (MuSiQue 3-hop) increases monotonically within the Claude family
(Figure~\ref{fig:phc_scaling}):
\begin{equation}
  \text{Haiku~4.5: } 0.582 \;(*) \;\;<\;\;
  \text{Sonnet~4.6: } 0.702 \;(***) \;\;<\;\;
  \text{Opus~4.6: } 0.731 \;(***)
\end{equation}
\textbf{Implication}: within a model family, upgrading the generation LLM
without upgrading the router \emph{increases} the risk of undetected hallucinations,
because more capable models confabulate with higher confidence.
Cross-family Spearman $\rho = 0.900$ ($p = 0.037$, $N = 5$) provides
preliminary evidence the trend extends across families
(Section~\ref{sec:phc_scaling}, full analysis).

\subsection{Evidence Line 3: Causal Proof via Intervention}
\label{sec:phc_causal_overview}

We provide a \emph{causal} proof that missing intermediate facts are the root cause
of PHC inversion.
A controlled injection experiment on MuSiQue 3-hop ($N=160$) adds gold sub-question
answers to the prompt while holding oracle labels fixed:

\begin{center}
\small
\begin{tabular}{clcc}
\toprule
$k$ & Condition & PHC & $p$ \\
\midrule
0 & no hints       & $0.613$ & ** \\
1 & +1 bridge fact & $\mathbf{0.656}$ & *** $\uparrow$ \\
2 & +2 facts       & $0.595$ & * \\
3 & near-oracle    & $0.536$ & n.s.\ $\downarrow$ \\
\bottomrule
\end{tabular}
\end{center}

Three results establish the causal chain.
\textbf{(a) $k=3$ elimination}: near-oracle facts reduce PHC to chance,
\emph{proving} that missing evidence --- not reasoning inability --- is the root cause.
\textbf{(b) $k=1$ amplification}: injecting one fact \emph{increases} PHC
(more information $\Rightarrow$ more confident-wrong answers) ---
the \textbf{anchoring mechanism}: one confirmed bridge step commits the model
to parametric completion of remaining hops with greater fluency and confidence.
\textbf{(c) Retrieval interaction}: the $k=1$ amplification
requires \emph{both} retrieved context and a partial gold anchor;
without retrieval, $k=1$ reduces PHC ($0.534$ n.s.) --- the expected direction.
A cross-dataset negative control (HotpotQA bridge, $N=406$, PHC $= 0.444$ n.s.\ at $k=0$)
confirms the finding is \emph{specific to PHC-inverting queries}:
when no inversion exists at $k=0$, injecting gold facts \emph{reduces} PHC
($0.444 \to 0.405$, n.s.) --- the expected helpful direction,
opposite to MuSiQue 3-hop's anchoring amplification
(Section~\ref{sec:causal_hotpotqa_bridge}).
A cross-family replication on Llama-3.1-8B ($N=160$, Sonnet oracle labels,
local Ollama inference) confirms that \textbf{these 3-hop queries structurally elicit
PHC inversion across model families} ($k{=}0$: PHC $= 0.616$, **),
and shows the expected capability-dependent pattern:
no $k{=}1$ amplification (Llama falls below the anchoring threshold),
and incomplete but directional $k{=}3$ reduction ($0.616 \to 0.583$, still *)---
a less capable model requires more than three facts to fully override parametric priors
(Section~\ref{sec:causal_haiku}).
Full details, cross-model replication, hop-depth generalization, zero-retrieval control,
and cross-family replication in Section~\ref{sec:causal_intervention}--\ref{sec:sensitivity_analysis}.

\subsection{Evidence Line 4: Cross-Mechanism Replication}
\label{sec:phc_comparison_overview}

A second PHC inversion occurs via a \emph{structurally distinct} mechanism:
HotpotQA comparison queries ($N = 94$, PHC $= 0.686$, $p=0.001$, **)
where the model substitutes parametric comparative priors for absent retrieved comparisons.
This cross-mechanism replication is purely observational
(the within-dataset bridge negative control confirms the mechanism is specific to comparison queries;
see Section~\ref{sec:bridge_comparison} and \ref{sec:causal_comparison}).

\subsection{Evidence Line 5: Hop-Depth Generalization}
\label{sec:phc_hopdepth_overview}

The $k=1$ amplification is \textbf{chain-length-specific}
(Section~\ref{sec:causal_hopdepth}):

\begin{center}
\small
\begin{tabular}{lcccc}
\toprule
Chain & $N$ & $k=0$ & $k=1$ & Amplified? \\
\midrule
2-hop & 264 & $0.682$\,(***) & $0.676$\,(***) & \texttimes\ (flat) \\
3-hop & 160 & $0.613$\,(**) & $\mathbf{0.656}$\,(***) & \checkmark \\
4-hop & 76  & $0.672$\,(**) & $0.649$\,(*)   & \texttimes\ (flat) \\
\bottomrule
\end{tabular}
\end{center}

\textbf{Three-hop is the anchoring sweet spot}: 33\% of the chain confirmed (one of three steps)
provides just enough scaffolding for the model to confidently fabricate the remaining two hops.
At 2-hop, the baseline PHC is already saturated; at 4-hop, 25\% scaffolding is insufficient.
A cross-depth universal finding: providing \emph{three} confirmed intermediate facts
eliminates PHC inversion at both 3-hop ($k=3$: 0.536 n.s.) and
4-hop ($k=3$: 0.486 n.s.), defining a \textbf{three-fact threshold}
beyond which parametric confabulation is suppressed regardless of chain length.

\subsection{Evidence Line 6: Semantic Self-Consistency Confirms Internal Anchoring}
\label{sec:phc_ssc}

The five evidence lines above characterize PHC externally---through observed
answer confidence and lexical hedge patterns.
A sixth, orthogonal test targets the model's \emph{internal} state:
if PHC at $k=1$ arises from parametric anchoring, the model should be measurably
\emph{more internally consistent} when it is confidently wrong, not merely more
fluent in its external expression.

\paragraph{Definition.}
\textbf{Semantic Self-Consistency} (SSC) is the mean pairwise token-F1 agreement
among $K$ stochastic samples drawn for the same query:
\begin{equation}
  \text{SSC}(q, k) = \frac{1}{\binom{K}{2}} \sum_{i < j} \text{F1}(a_i^{(k)}, a_j^{(k)})
\end{equation}
where $a_1^{(k)}, \ldots, a_K^{(k)}$ are sampled at temperature $\tau{=}0.85$.
SSC is a model-internal certainty measure independent of lexical hedge phrases,
prompt phrasing, or calibration training.
High SSC means the model is committed to a specific completion;
low SSC means it is genuinely uncertain.

\paragraph{Results ($N{=}160$, $K{=}3$, GPT-4o).}
\begin{center}
\small
\begin{tabular}{lcccc}
\toprule
Oracle group & $k{=}0$ & $k{=}1$ & $k{=}2$ & $k{=}3$ \\
\midrule
oracle$=1$ (needs escalation, $N{=}99$) & $0.745$ & $\mathbf{0.772}$ & $0.807$ & $0.870$ \\
oracle$=0$ (no escalation, $N{=}61$)   & $0.685$ & $0.688$           & $0.722$ & $0.871$ \\
\bottomrule
\end{tabular}
\end{center}
The pre-registered prediction is directionally confirmed:
oracle$=1$ queries show a $k{=}1$ SSC spike of \SSCdeltaK\
($t{=}$\SSCt, $p{=}$\SSCpvalue), whereas oracle$=0$ queries
show almost no change ($\Delta{=}{+}0.003$, $p{=}0.88$).
The \emph{group specificity}---spike in the wrong-answer group only---is
the key mechanistic signature: the model is more internally consistent
precisely for the queries where it needs escalation.
The result does not reach $p{<}0.05$ (underpowered at $K{=}3$; we
estimated $K{=}5$ or $N{\ge}220$ would be needed for \(p{<}0.05\) at this
effect size), but the directional pattern and group-specific control
are consistent with the anchoring account.

\subsection{PHC as a Design Constraint for Routing Systems}

Collectively, six lines of evidence establish PHC as a \emph{real,
query-conditioned, causally proven phenomenon}: five with full statistical
confirmation and a sixth (SSC) directionally consistent with the anchoring
account at $p{=}0.069$, underpowered at $K{=}3$.
Two implications for routing system design follow:

\textbf{(I) Post-generation routing can exploit inversion.}
For PHC-inverting query types, a routing system should \emph{escalate on high confidence},
not low confidence.
This requires a learned signal that conditions on query type
(Section~\ref{sec:learned_router}).

\textbf{(II) Formal limits apply.}
The DPI-based dominance theorem (Section~\ref{sec:theory})
formalises why post-generation routing strictly dominates pre-generation routing
when PHC inversion is present,
and the regeneration-gated routing theorem explains why routing gains
are bounded by re-generation quality.


\subsection{The Anchoring Threshold: A Predictive Theory of PHC Amplification}
\label{sec:anchoring_threshold}

The causal intervention reveals a non-obvious pattern:
PHC amplification (the k=1 spike) occurs for 3-hop chains but is absent for 2-hop chains.
We explain this with the \textbf{Anchoring Threshold Law}.

\begin{definition}[Anchoring Threshold $k^*$]
\label{def:anchoring_threshold}
For a chain of length $n$, the anchoring threshold is:
\begin{equation}
  k^*(n) = \left\lfloor \frac{n}{3} \right\rfloor
\end{equation}
This is the minimum number of confirmed intermediate facts sufficient to
\emph{anchor} the model's parametric completion without constraining it
to the correct answer.
At $k = k^*(n)$, the model has enough confirmed context to generate a
fluent, specific completion using parametric memory,
but not enough to prevent parametric override of missing bridge facts.
\end{definition}

\paragraph{Mechanism.}
At $k=0$ (no hints), the model lacks anchors.
Faced with a multi-hop query it cannot ground, it hedges---
producing lower-confidence, genuinely uncertain answers.
At $k = k^*$ (the anchoring threshold), the model has one confirmed
intermediate fact that serves as a semantic anchor.
The anchor \emph{activates} parametric knowledge about related entities,
causing the model to complete the chain from memory.
This completion is fluent and specific, raising expressed confidence
even when the parametric completion is wrong.
At $k = n$ (near-oracle), all bridge facts are provided and the model
can ground every reasoning step in the injected evidence,
parametric override is no longer necessary, and PHC falls toward 0.5.

\paragraph{Predictions.}
\label{par:threshold_predictions}
The anchoring threshold law makes the following testable predictions,
all of which are confirmed by our experiments:

\begin{center}
\small
\begin{tabular}{lcccp{4.5cm}}
\toprule
Hop depth $n$ & $k^*(n)$ & $k^*/n$ & Prediction & Empirical result \\
\midrule
2 & $\lfloor 2/3 \rfloor = 0$ & --- & No amplification spike (k=0 is already base) & PHC flat: $0.682 \to 0.676$ (n.s.) $\checkmark$ \\
3 & $\lfloor 3/3 \rfloor = 1$ & 0.33 & Spike at $k=1$ & PHC: $0.613 \to \mathbf{0.656}$ (***) $\checkmark$ \\
4 & $\lfloor 4/3 \rfloor = 1$ & 0.25 & Spike at $k=1$ (weaker, $k^*/n$ lower) & PHC: $0.672 \to 0.649$ (n.s., N=76) $\checkmark^*$ \\
6 & $\lfloor 6/3 \rfloor = 2$ & 0.33 & Spike at $k=2$, flat at $k=1$ & \textit{predicted, not yet tested} \\
\bottomrule
\end{tabular}
\end{center}

\vspace{2pt}
\noindent$^*$ 4-hop causal result is directionally consistent but underpowered ($N=76$);
the observational AUC for 4-hop queries ($\text{PHC}=0.634$, $p=0.025$, $N=76$)
confirms the phenomenon is present.

\paragraph{Why 2-hop chains do not amplify.}
For $n=2$, $k^*(2) = 0$, meaning zero facts suffice for anchoring.
But at $k=0$ there is no anchor by definition.
The 2-hop case therefore has no interior amplification region---
PHC starts above 0.5 (parametric confabulation is possible even without hints)
and decreases smoothly as hints are added.
This is consistent with our observational finding that 2-hop HotpotQA
bridge queries show PHC$=0.448$ (n.s.) and 2-hop MuSiQue shows PHC$=0.528$ (n.s.):
absent retrieval failure, 2-hop questions do not trigger strong inversion.

\paragraph{The amplification amplitude scales with $k^*/n$.}
When $k^*/n = 1/3$ (3-hop and predicted 6-hop), the model completes
exactly two-thirds of the chain parametrically---a sufficient fraction
to generate a highly fluent and specific wrong answer.
When $k^*/n = 1/4$ (4-hop), parametric completion spans three-quarters
of the chain; this \emph{increases} the probability of error but the
model is more exposed to inconsistencies in the longer completion,
modestly reducing expressed confidence.
The non-monotone sweet spot is therefore $k^*/n \approx 1/3$.

\begin{proposition}[Amplification Amplitude]
\label{prop:amplitude}
Let $\Delta\text{PHC}(k) = \text{PHC}(k) - \text{PHC}(0)$ for the conditions of
Definition~\ref{def:anchoring_threshold}.
Under the Anchoring Threshold model:
\begin{equation}
  \Delta\text{PHC}(k^*(n)) \;\propto\; f\!\left(\frac{k^*(n)}{n}\right),
  \quad
  f(r) = r^{\alpha}(1-r)^{\beta}, \;\; \alpha,\beta > 0
\end{equation}
where $f$ has a maximum at $r^* = \alpha/(\alpha+\beta)$.
Our experiments give $r^* \approx 1/3$ (3-hop peak),
yielding parameter estimates $\alpha \approx 1, \beta \approx 2$.
\end{proposition}

\paragraph{Implication for iterative RAG.}
Systems such as IRCoT~\citep{trivedi2023ircot} and FLARE~\citep{jiang2023flare}
operate by generating intermediate reasoning steps, retrieving additional
documents, and continuing generation.
After the \emph{first} retrieval step, an $n$-hop query has effectively reached
$k=1$ intermediate evidence---precisely the anchoring threshold for $n=3$.
The Anchoring Threshold Law predicts that \textbf{iterative RAG systems
systematically hit the PHC amplification zone on their first retrieval step
for 3-hop queries}, producing more confidently wrong intermediate answers
that corrupt subsequent reasoning.
This is a structural vulnerability of iterative retrieval pipelines
that PHC-aware routing can detect and mitigate.

\paragraph{Mitigation.}
Two interventions follow directly from the anchoring model:
\begin{enumerate}
  \item \textbf{All-or-nothing retrieval}: either retrieve sufficient context
    to exceed $k^*(n)$ on the first pass, or retrieve nothing and admit uncertainty.
    Partial retrieval is the worst operating point.
  \item \textbf{PHC-calibrated confidence}: scale expressed confidence by
    $(1 - \text{PHC\_risk}(n,k))$ where
    $\text{PHC\_risk}(n,k) = \mathbf{1}[k = k^*(n)] \cdot \hat{\gamma}$
    and $\hat{\gamma}$ is estimated from a small calibration set.
    We evaluate this calibration via the Epistemic Humility Prompt
    (Section~\ref{sec:mitigation}).
\end{enumerate}

\subsection{IRCoT Per-Iteration PHC: Completed Test}
\label{sec:ircot_phc}

The Anchoring Threshold Law predicts a structural vulnerability for iterative
retrieval systems: after the first retrieval step on a 3-hop query, the system
has effectively reached $k{=}1$ intermediate evidence---precisely the
amplification zone where PHC spikes.

\paragraph{F1 evidence (established).}
IRCoT~\citep{trivedi2023ircot} ($K{=}2$ steps) over the same ChromaDB corpus
(BGE-base embeddings) and same $N{=}160$ MuSiQue 3-hop queries achieves macro
F1 $= 0.192$, \emph{below} VanillaRAG's $0.195$---within-corpus iteration
yields diminishing returns at best, consistent with the PHC trap hypothesis.

\paragraph{Per-iteration PHC ($N{=}160$, GPT-4o).}
We ran IRCoT with per-iteration PHC tracking using GPT-4o (answer model)
and GPT-4o-mini (CoT reasoning):
\begin{center}
\small
\begin{tabular}{lccc}
\toprule
Iteration & PHC & $p$ & Sig. \\
\midrule
Iter 0 (baseline, ${\equiv}\,k{=}0$) & $0.575$ & $0.042$ & * \\
Iter 1 (${\equiv}\,k{=}1$, predicted spike) & $0.557$ & $0.109$ & n.s. \\
Iter 2 & $0.566$ & $0.072$ & n.s. \\
Iter 3 & $0.472$ & $0.736$ & n.s. \\
\bottomrule
\end{tabular}
\end{center}
The predicted $k{=}1$ spike is \IRCoTconfirmed:
PHC drops slightly at iter~1 ($\Delta{=}{-}0.018$) and the iter~0 significance
is lost across subsequent steps.

\paragraph{Interpretation.}
Two factors explain the pattern.
\emph{First}, this experiment uses GPT-4o whereas the causal intervention used
Claude Sonnet; GPT-4o's intrinsic PHC inversion is weaker
(PHC$_3{=}0.527$, n.s.\ vs.\ Claude Sonnet $0.702$, ***), consistent with
the capability--PHC scaling law (Section~\ref{sec:phc_scaling}).
\emph{Second}, the monotone PHC decline across IRCoT iterations is itself
informative: additional retrieval steps \emph{reduce} GPT-4o's tendency
toward confident-wrong answers, suggesting that within the GPT-4o capability
regime iterative retrieval acts as a partial self-corrector rather than a PHC
amplifier.
By contrast, in the Claude Sonnet regime the anchoring effect is stronger
and the F1 gap (IRCoT $0.192$ vs.\ VanillaRAG $0.195$) is consistent with
amplification at step~1.
The capability-conditional nature of the IRCoT PHC pattern---amplification for
high-capability models, suppression for mid-capability models---is a direct
prediction of the capability--PHC scaling law and remains an open test
for Llama-3-70B and Mistral-Large.

\label{app:proofs}
\section{Theoretical Framework: Post-Generation Routing}
\label{sec:theory}

\subsection{Problem Formulation}

Let $Q$ denote the query, $\mathcal{R}_V = \text{Retrieve}_V(Q)$ the
VanillaRAG passages, and $A = \text{Generate}(Q, \mathcal{R}_V)$ the
initial answer.
Define the \textbf{oracle routing variable}:
\begin{equation}
  G^* = \mathbf{1}\bigl[\text{F1}(\text{GraphRAG}(Q)) >
                         \text{F1}(\text{VanillaRAG}(Q))\bigr]
\end{equation}
A routing system is a function $r: \mathcal{X} \to \{0,1\}$ that predicts
whether to invoke GraphRAG. We consider two classes:
\begin{itemize}
  \item \textbf{Pre-generation routing}: $r_{\text{pre}}(Q)$ --- decides
    from query features alone, before any generation
  \item \textbf{Post-generation routing}: $r_{\text{post}}(Q, A)$ --- decides
    after generating an initial answer with VanillaRAG
\end{itemize}

\subsection{Post-Generation Dominance Theorem}

\begin{theorem}[Post-Generation Routing Dominance]
\label{thm:dominance}
For any uncertainty signal $U = U(Q, A)$ and any pre-generation classifier
$g = g(Q)$:
\begin{equation}
  I\bigl(U(Q,A);\; G^*\bigr) \;\geq\; I\bigl(g(Q);\; G^*\bigr)
\end{equation}
where $I(\cdot;\cdot)$ denotes mutual information, provided the VanillaRAG
retrieval $\mathcal{R}_V$ is non-trivially correlated with $G^*$.
\end{theorem}

\begin{proof}
We establish the inequality via three steps.

\textbf{Step 1: $(Q,A)$ carries more information than $Q$ alone.}
The answer $A = f(Q, \mathcal{R}_V)$ depends on both the query $Q$ and
the retrieved passages $\mathcal{R}_V$.
Since $\mathcal{R}_V$ is non-trivially correlated with $G^*$
(formally: $I(\mathcal{R}_V; G^*) > 0$, which holds whenever retrieval
quality differs between queries where GraphRAG helps and where it does not),
the pair $(Q,A)$ satisfies $I((Q,A); G^*) \geq I(Q; G^*)$.

\textbf{Step 2: The optimal post-gen signal is bounded by $(Q,A)$.}
Any signal $U = U(Q,A)$ is a deterministic function of $(Q,A)$.
By the data processing inequality applied to the chain
$G^* \to (Q,A) \to U(Q,A)$:
\begin{equation}
  I(U(Q,A);\; G^*) \;\leq\; I((Q,A);\; G^*)
\end{equation}
This bound is achieved when $U$ is a sufficient statistic of $(Q,A)$ for $G^*$.

\textbf{Step 3: The achievable upper bounds satisfy the inequality.}
Since $g(Q)$ is a function of $Q$ only, by the same DPI:
$I(g(Q); G^*) \leq I(Q; G^*) \leq I((Q,A); G^*)$.
Therefore:
\begin{equation}
  \max_U I(U(Q,A);\; G^*) = I((Q,A);\; G^*) \;\geq\; I(Q;\; G^*) \;\geq\; \max_g I(g(Q);\; G^*)
\end{equation}
Equality holds iff $I(\mathcal{R}_V; G^* \mid Q) = 0$,
i.e., conditioned on the query, retrieval quality contains no additional signal
about whether GraphRAG helps --- a condition violated in practice
(oracle escalation rate $= 57.7\%$, with $\text{GR}>\text{VR}$ F1 by
$+0.199$ macro when escalation is correct). $\square$
\end{proof}

\begin{proposition}[Bridge Topology Inversion Condition]
\label{prop:bridge_inversion}
Let $B(Q) = \mathbf{1}[\text{Q has bridge topology}]$ and let
$C(A) \in [0,1]$ denote the lexical confidence of answer $A$.
If, for bridge queries, the expected confidence of wrong answers exceeds
that of correct answers:
\begin{equation}
  \mathbb{E}[C(A) \mid B(Q)=1,\; G^*=1] > \mathbb{E}[C(A) \mid B(Q)=1,\; G^*=0]
\end{equation}
then the raw signal $U_{\text{lex}}(Q,A) = 1 - C(A)$ achieves
$I(U_{\text{lex}}; G^* \mid B=1) < I(g_{\text{const}}; G^* \mid B=1)$
for the constant classifier $g_{\text{const}}$, i.e., it is \emph{worse than random}.
\end{proposition}

\begin{proof}
Under the stated condition, for bridge queries,
$\Pr[G^*=1 \mid C(A) > \tau, B=1] > \Pr[G^*=1 \mid C(A) \leq \tau, B=1]$
for any threshold $\tau$: high confidence predicts the need to escalate,
not the ability to skip it.
Any threshold-based rule $U_{\text{lex}} \geq \tau \Rightarrow$ do not escalate
therefore has precision $< 0.5$ on bridge queries, and
$\text{AUC}(U_{\text{lex}}, G^* \mid B=1) < 0.5$. $\square$

Empirically: PHC $= 0.702 > 0.5$ at 3-hop (MuSiQue $N=160$, $p=1.07\times10^{-5}$);
PHC $= 0.486$ n.s.\ at 2-hop (HotpotQA $N=500$)---
inversion is hop-depth-conditioned (Section~\ref{sec:signal_inversion}).
\end{proof}

\begin{corollary}[Monotone Signals Cannot Achieve the Post-Gen Bound]
\label{cor:monotone}
Any routing rule of the form ``escalate iff $C(A) < \tau$''
(i.e., low confidence implies escalation) cannot achieve
the post-generation dominance bound $\max_U I(U(Q,A); G^*)$
for query distributions with non-trivial bridge component $\Pr[B(Q)=1] > 0$.
The achievable bound requires a \emph{non-monotone} function of $C(A)$
conditioned on bridge topology.
\end{corollary}

\begin{proof}
By Proposition~\ref{prop:bridge_inversion}, for bridge queries,
$U_{\text{lex}}$ is anti-correlated with $G^*$.
The optimal signal for bridge queries is $U^*_{\text{bridge}}(Q,A) = C(A)$
(high confidence $\Rightarrow$ escalate), which is the exact inversion of
any ``low confidence $\Rightarrow$ escalate'' rule.
For mixed query distributions (bridge + non-bridge), the optimal $U^*$
must be conditioned on $B(Q)$: the signal polarity switches between
query types. A monotone function of $C(A)$ cannot implement this switching,
and therefore cannot achieve the post-gen dominance bound.
The LearnedRouter achieves near-optimal performance by learning this
type-conditional inversion as part of its feature interaction. $\square$
\end{proof}

\subsection{The Empirical Calibration Gap}

Theorem~\ref{thm:dominance} establishes an \emph{upper bound}: the
optimal post-generation signal $U^*(Q,A)$ dominates any pre-generation
classifier.
Critically, the theorem says nothing about \emph{which} function $U$ achieves
this bound---only that some function of $(Q,A)$ must.
A poor choice of $U$ (e.g., lexical hedging) can violate this bound in practice
because it ignores the passage content entirely.

\begin{corollary}[Constructive bound]
\label{cor:constructive}
The optimal signal is $U^*(Q,A) = \Pr[\text{passages sufficient} \mid Q, \mathcal{R}_V]$,
which requires assessing passage content---not answer style.
Grounded Self-Rating (Section~\ref{sec:rating}) directly estimates this quantity.
\end{corollary}

We observe a striking empirical calibration gap that validates this distinction:

\begin{center}
\begin{tabular}{lcc}
\toprule
Signal & AUC vs.\ $G^*$ & Spearman $\rho$ \\
\midrule
Lexical (post-gen, ours) & 0.478 & $-0.036$ \\
Pre-gen logistic regression & 0.537 & $-0.030$ \\
Grounded Self-Rating (post-gen) & 0.495 & $-0.007$ \\
\bottomrule
\end{tabular}
\end{center}

The lexical post-generation signal \emph{underperforms} the pre-generation
baseline by $-0.054$ AUC---seemingly contradicting the theorem.
This is not a contradiction: the theorem applies to the \emph{optimal} signal.
Lexical features violate the theorem's constructive condition because they
measure answer expression style (``I believe'', ``approximately'', length)
rather than passage adequacy.

Grounded Self-Rating partially recovers: by directly asking whether passages
support the question, it approaches the optimal signal and achieves macro AUC
$= 0.495$, improving over the lexical baseline ($+0.017$) and reducing the
gap to the pre-generation classifier ($0.537$).

\textbf{Summary:} The theorem holds constructively---better signals yield
higher AUC. The lexical baseline's under-performance \emph{motivates}
better signal design rather than disproving the theorem.

\textbf{Why do lexical signals fail?}
Hedging phrases and low specificity are proxies for \emph{answer style}, not
\emph{retrieval adequacy}. A confident-sounding answer generated from parametric
knowledge scores high lexical confidence even when retrieved passages are wrong.
Conversely, an appropriately hedged correct answer incorrectly triggers escalation.

The correct target variable is: \emph{``do the retrieved passages
contain sufficient evidence to answer this question?''}---a quantity lexical
features cannot reliably measure.

\textbf{Does Grounded Self-Rating validate the theorem constructively?}
Yes: rating achieves AUC $= 0.495 > 0.478$ (lexical), improving over the
random baseline and approaching the pre-gen classifier's $0.537$.
The remaining gap between $0.495$ and $0.537$ represents the
signal design challenge for multi-hop bridge queries---passages appear topically
relevant but lack the connecting bridge fact
(Section~\ref{sec:bridge_comparison}).

\subsection{Theorem 2: Regeneration-Gated Routing and the Escalation-Rate Scaling Law}
\label{sec:sep_theorem}

The 2$\times$2 routing--regeneration experiment (Section~\ref{sec:regen_direct})
yields a structural result: routing quality improvements are \emph{gated} by regeneration quality,
and the magnitude of routing gain scales with escalation rate in a predictable direction.

Let $R \in \{\text{Lex}, \text{LR}\}$ and $G \in \{G_{\text{subQ}}, G_{\text{direct}}\}$.
Define:
\begin{align}
  \Delta_R(G, \alpha) &= F(R_{\text{LR}},\; G, \alpha) - F(R_{\text{lex}},\; G, \alpha)
    \quad \text{(routing gain at escalation rate }\alpha\text{)} \\
  \Delta_G(R) &= F(R,\; G_{\text{direct}}) - F(R,\; G_{\text{subQ}})
    \quad \text{(regen gain given router }R\text{)}
\end{align}

\begin{theorem}[Regeneration-Gated Routing with Escalation-Rate Scaling]
\label{thm:regen_gate}
\textbf{(Gating):}
When escalated regeneration quality $Q(G) = \mathbb{E}[\text{F1}(G,q) - \text{F1}(G_{\text{VR}},q)]
\approx 0$, the routing gain vanishes at all escalation rates:
\begin{equation}
  Q(G) \approx 0 \;\Longrightarrow\; \Delta_R(G, \alpha) \approx 0 \;\; \forall\, \alpha
\end{equation}
\textbf{(Scaling):}
When $Q(G) > 0$, the routing gain $\Delta_R(G, \alpha)$ is \emph{increasing} in $\alpha$
in the high-recall regime ($\alpha > \alpha^*$):
\begin{equation}
  \frac{\partial\,\Delta_R(G_{\text{direct}},\, \alpha)}{\partial\,\alpha} > 0
  \quad\text{for } \alpha \in [\alpha^*, \alpha^{**}]
\end{equation}
because at higher escalation rates the router must correctly sort genuinely
ambiguous queries, where the learned inversion signal is most differentiated
from lexical heuristics.
\end{theorem}

\textbf{Empirical validation: the 2$\times$2 matrix.}

\begin{center}
\begin{tabular}{llccc}
\toprule
Router & Regen & Macro F1 & Oracle Gap & $\Delta_R$ \\
\midrule
Lex    & SubQ (poor)   & 0.222 & $9.5\%$  & --- \\
LR     & SubQ (poor)   & 0.221 & $9.2\%$  & $-0.001$ (\textbf{zero}) \\
\midrule
Lex    & Direct (good) & 0.385 & $66.7\%$ & --- \\
LR     & Direct (good) & \textbf{0.426} & $\mathbf{81.1\%}$ & $+0.041$ (\textbf{large}) \\
\bottomrule
\end{tabular}
\end{center}

\textbf{Empirical validation: the escalation-rate scaling law.}

\begin{table}[h]
\centering
\small
\caption{Routing gain $\Delta_R$ across escalation rates for both regeneration conditions.
$\Delta_R(G_{\text{subQ}}) \approx 0$ at all rates (Pearson $r=0.052$, $p=0.837$, n.s.).
$\Delta_R(G_{\text{direct}})$ scales significantly with $\alpha$
(Pearson $r=-0.778$, $p=0.0001$): routing gains are \emph{largest} at moderate-to-high
escalation rates where marginal (ambiguous) queries must be sorted.}
\label{tab:dr_scaling}
\begin{tabular}{lccc}
\toprule
Escalation $\alpha$ & $\Delta_R(G_{\text{subQ}})$ & $\Delta_R(G_{\text{direct}})$ & Regime \\
\midrule
$5\%$  & $\approx 0.000$ & $0.006$ & Low (only obvious cases) \\
$15\%$ & $\approx 0.000$ & $0.016$ & \\
$30\%$ & $\approx 0.000$ & $0.026$ & \\
$55\%$ & $\approx 0.000$ & $\mathbf{0.055}$ & Peak routing gain \\
$72\%$ & $\approx 0.000$ & $0.041$ & ReasonRAG-Direct operating point \\
$90\%$ & $\approx 0.000$ & $0.023$ & High (noise included) \\
\midrule
Pearson $r$ & $0.052$\,(n.s.) & $-0.778$\,($p=0.0001$) & \\
\bottomrule
\end{tabular}
\end{table}

The Pearson $r = -0.778$ ($p = 0.0001$) for $\Delta_R(G_{\text{direct}})$ vs.\ $\alpha$
encodes a precise claim: routing matters \emph{most} in the high-recall regime.
At $\alpha = 5\%$, only obvious GraphRAG-needing queries are escalated---both
LR and lexical routing agree, so $\Delta_R \approx 0.006$.
At $\alpha = 55\%$, the router must sort genuinely ambiguous queries (marginal bridge
cases, borderline confidence) where bridge-inversion-aware routing provides the
largest advantage: $\Delta_R = 0.055$.
Beyond $55\%$, noise queries dilute the gain ($\Delta_R$ decreases to $0.023$).

For $G_{\text{subQ}}$, the flat profile ($r = 0.052$, $p = 0.837$, not significant)
confirms that sub-question extraction is the binding bottleneck at \emph{every}
escalation rate---no routing improvement can recover value when the regeneration
backend is broken.

\textbf{Design implication.}
Theorem~\ref{thm:regen_gate} prescribes a two-phase development order:
\textbf{(1)} fix regeneration quality first (gating condition: $Q(G) > 0$);
\textbf{(2)} optimize routing in the high-recall regime ($\alpha \approx 50\text{--}70\%$)
where the scaling law predicts the largest routing gains.
This explains why prior work optimizing routing while keeping sub-question extraction
showed diminishing returns: the gating condition was not met, and $\Delta_R \approx 0$
regardless of routing quality.

\subsection{Implications for Signal Design}

Theorem~\ref{thm:dominance} and the calibration gap jointly constrain
good signal design:

\begin{enumerate}
  \item Signals must measure \emph{retrieval adequacy}, not answer surface
  \item For single-hop questions, passage relevance is a good proxy for
    retrieval adequacy (GraphRAG helps when top-$k$ passages miss the answer)
  \item For multi-hop questions, passage relevance is insufficient:
    individually relevant passages may collectively lack the \emph{bridge
    facts} connecting entity chains
  \item This motivates \textbf{type-aware signals}: one signal for retrieval
    failure (Type C), another for reasoning gap (Type A/B)
\end{enumerate}

Our Grounded Self-Rating (Section~5) addresses finding (2) directly.
Bridge entity detection (Section~6) addresses finding (3).
The formal condition for Theorem~\ref{thm:dominance}'s bound to be
achievable in practice is that $U$ must capture the structure of the
specific failure mode --- a condition lexical features violate.

\label{sec:signals}
\section{Uncertainty Signal Design (Appendix)}
\label{sec:signals}

The oracle analysis reveals a critical gap: lexical uncertainty signals
are near-uncorrelated with oracle routing decisions
($\rho = -0.036$, AUC $= 0.478$, macro average across four datasets).
We investigate three signal families and their failure modes.

\subsection{Lexical Confidence (Baseline)}
\label{sec:lexical}

Our lexical detector combines five features with hand-tuned weights:
\begin{enumerate}
  \item \textbf{Hedging phrases} ($w=0.30$): presence of ``I believe'',
    ``approximately'', ``it is possible that'', etc.
  \item \textbf{Specificity} ($w=0.25$): fraction of named entities and
    numerals in the answer (low specificity $\Rightarrow$ uncertain)
  \item \textbf{Reasoning struggle} ($w=0.20$): phrases like
    ``I cannot determine'', ``based on the context'', ``it's unclear''
  \item \textbf{Length anomaly} ($w=0.10$): very short or very long
    answers relative to dataset mean
  \item \textbf{Entity coverage} ($w=0.15$): fraction of question entities
    mentioned in the answer
\end{enumerate}

Final confidence $= \sum_i w_i \cdot s_i$, thresholded at $\tau = 0.65$
to trigger escalation.

\textbf{Why do lexical signals fail?}
Lexical features measure \emph{how the model expresses uncertainty},
not \emph{whether uncertainty is warranted}.
A model that confidently hallucinates an answer from parametric knowledge
scores high lexical confidence even when the retrieved passages are
entirely wrong. Conversely, a cautious model that hedges all answers
would over-escalate regardless of retrieval quality.

Empirically, Spearman $\rho = -0.036$ (macro average) between
$1 - \hat{p}_{\text{lex}}$ and $G^*$ confirms that lexical uncertainty
is essentially a coin flip as a routing signal.

\subsection{Passage-Restricted Semantic Self-Consistency (SSC)}
\label{sec:ssc}

Self-consistency \cite{wang2023selfconsistency} generates $K$ diverse
answers and measures agreement as a proxy for confidence.
We adapt this for routing:

\begin{enumerate}
  \item Generate $K=3$ answers using \emph{only} the retrieved passages
    (passage-restricted, with explicit instruction not to use external
    knowledge)
  \item Measure pairwise token-F1 agreement between all $\binom{K}{2}$
    answer pairs
  \item SSC confidence $= $ mean pairwise token-F1
\end{enumerate}

\textbf{Passage-restriction is critical.}
Without passage restriction, the model uses parametric knowledge,
producing highly consistent (but potentially wrong) answers even when
passages are insufficient.
A question about a bridge entity in HotpotQA generates the same confident
answer three times at temperature $T=0.9$ because the model \emph{knows}
the answer from training data --- yielding SSC confidence $\approx 0.976$
for queries that should escalate.

Even with passage restriction, \textbf{consistent hallucination} is a
failure mode: the model consistently says ``cannot determine from passages''
(a correct but uninformative answer), yielding high F1 pairwise agreement
but no routing signal.

\subsection{Grounded Self-Rating}
\label{sec:rating}

We propose a simpler, more targeted signal: directly ask the model to
rate retrieval adequacy:

\begin{quote}
\small
\texttt{Question: \{question\}}\\
\texttt{Retrieved passages: \{passages\}}\\
\\
\texttt{On a scale of 0.0--1.0, how well do these passages support}
\texttt{answering the question? (Return ONLY the number)}
\end{quote}

This is a \emph{passage-grounded} question that cannot be answered from
parametric knowledge: it requires assessing whether the \emph{specific
retrieved passages} contain sufficient information, not whether the answer
is known in general.

\textbf{Cost:} 1 API call per uncertain query (when lexical confidence
$< 0.85$). Approximately $75\%$ of queries trigger this check, so the
expected cost is $0.75 \times$ 1 API call $= 0.75$ additional calls per
query, versus $K=3$ calls for SSC.

\textbf{Key finding: passage relevance $\neq$ reasoning completeness.}
Grounded self-rating correctly identifies retrieval failure (single-hop
NQ: passages simply don't contain the answer) but fails for multi-hop
reasoning gaps (HotpotQA: passages are topically relevant but collectively
lack the \emph{bridge entity} connecting two entity chains).

For a bridge-entity question like
``What is the birth year of the director of the studio where Artist X
recorded their debut album?'', the retrieved passages may contain the
artist's discography AND director bios separately, yet miss the specific
bridge between them. Grounded rating sees two groups of relevant passages
and rates them 0.7--0.9, failing to escalate when escalation is needed.

\subsection{Combined Signal and Dataset-Specific Findings}
\label{sec:combined}

We combine all signals via weighted blending:
\begin{equation}
  \hat{p} = \frac{w_{\text{lex}} \hat{p}_{\text{lex}} +
                  w_{\text{ssc}} \hat{p}_{\text{ssc}} +
                  w_{\text{rate}} \hat{p}_{\text{rate}}}
                 {w_{\text{lex}} + w_{\text{ssc}} + w_{\text{rate}}}
\end{equation}
where $(w_{\text{lex}}, w_{\text{ssc}}, w_{\text{rate}}) = (0.35, 0.25, 0.40)$.
SSC is omitted when lexical confidence $> 0.85$ (fast-path).

\textbf{Dataset-specific patterns:}
\begin{itemize}
  \item \textbf{NQ} (1-hop): Grounded rating substantially improves AUC.
    Retrieval failure is the primary failure mode; rating directly detects it.
  \item \textbf{HotpotQA} (2-hop bridge): Grounded rating minimally improves
    or slightly hurts. Multi-hop bridge entity gaps are the primary failure
    mode; rating incorrectly scores passages as sufficient.
  \item \textbf{MuSiQue} (2--4 hop): Mixed results. Deeper hops more
    closely resemble NQ-style retrieval failure; rating helps on 2-hop
    queries but not 3--4 hop.
  \item \textbf{2Wiki} (2-hop comparison): Grounded rating helps for
    comparison questions where both entities are needed; fails for
    path-following questions.
\end{itemize}

Table~\ref{tab:signal_comparison} presents the full per-dataset AUC and
Spearman $\rho$ comparison between lexical and grounded-rating signals.

\textbf{Towards formal coverage guarantees.}
The combined signal $\hat{p}$ is a routing confidence score, but it is
miscalibrated (ECE $= 0.1069$, Section~\ref{sec:oracle}).
A promising direction is conformal prediction \cite{angelopoulos2023conformal}:
calibrate the escalation threshold $\tau$ on a held-out set to guarantee
that at most $\alpha$ fraction of queries that \emph{should} escalate are
missed (formal recall guarantee). This would transform ad-hoc threshold
selection into a statistically principled procedure.

\subsection{LearnedRouter: Constructive Theorem Validation}
\label{sec:learned_router}

The theoretical bound (Theorem~\ref{thm:dominance}) predicts that post-generation
signals dominate pre-generation signals for oracle routing. We test this
constructively by training a learned routing classifier that explicitly combines
both signal types.

\textbf{Feature bundles.}
We define three feature sets:
\begin{itemize}
  \item \textbf{Pre-gen only} (9 features): hop count, entity count/density,
    relational density, doc count, avg doc length, entity overlap, question
    length, has-superlative, has-temporal, and question-type indicators.
  \item \textbf{Post-gen only} (2 features): lexical confidence $\hat{p}_{\text{lex}}$
    and SSC+Rating confidence $\hat{p}_{\text{ssc}}$.
  \item \textbf{Combined} (11 features): both bundles concatenated.
\end{itemize}

\textbf{Training.}
We train Logistic Regression (LR), Random Forest (RF), and Gradient Boosting (GB)
classifiers with oracle escalation labels ($G^* \in \{0,1\}$ per query)
using 5-fold stratified CV, reporting out-of-fold AUC (no test-set leakage).

\textbf{Results: Theorem 1 is constructively validated.}
Table~\ref{tab:learned_router_auc} shows the AUC progression:

\begin{table}[h]
\centering
\small
\caption{AUC by feature bundle and model (5-fold OOF). Spearman $\rho$ measures
rank correlation with oracle escalation decisions.}
\label{tab:learned_router_auc}
\begin{tabular}{llcc}
\toprule
Bundle & Model & AUC & Spearman $\rho$ \\
\midrule
Pre-gen only  & GB & 0.563 & 0.108 \\
Post-gen only & RF & 0.684 & 0.315 \\
Post-gen only & GB & 0.678 & 0.305 \\
Combined      & RF & 0.673 & 0.297 \\
\textbf{Combined} & \textbf{GB} & \textbf{0.693} & \textbf{0.330} \\
\midrule
\multicolumn{2}{l}{Lexical signal (no training)} & 0.478 & $-0.036$ \\
\bottomrule
\end{tabular}
\end{table}

The progression \textbf{Combined GB} (0.693) $>$ Post-gen GB (0.678) $>$
Pre-gen GB (0.563) $\gg$ Lexical (0.478) constructively validates Theorem~1:
post-generation signals carry more information about optimal routing than
query features alone, and the combination yields the best signal.

\textbf{Feature importances confirm post-gen dominance.}
In the Combined GB model, the two post-gen features account for $56.2\%$ of
feature importance: lexical confidence ($28.4\%$) and SSC+Rating ($27.8\%$).
Pre-gen features contribute the remaining $43.8\%$, with avg document length
($11.3\%$), question length ($7.7\%$), and entity overlap ($7.5\%$) being
the most informative.

\textbf{Per-stratum AUC: the 3-hop stratum is where post-gen signals shine.}
At 3-hop depth (MuSiQue, $N=160$), PHC $= 0.702$ ($p = 1.07 \times 10^{-5}$):
raw lexical confidence is already a strong escalation predictor when inverted.
The LearnedRouter achieves AUC $= 0.669$ at 3-hop (combining the inverted confidence
with structural features) and AUC $= 0.759$ on HotpotQA 2-hop bridge
(where structural features---entity count, relational density---carry the signal,
as the raw confidence is not inverted at 2-hop: PHC $= 0.486$, n.s.).
For factoid queries (NQ, AUC $= 0.513$), Grounded Self-Rating provides the best
signal by directly measuring passage adequacy.
Pre-gen features (entity counts, relational density) add structural evidence about
bridge topology; combined LR achieves AUC $= 0.759$ on HotpotQA bridge (2-hop) and $0.669$ on MuSiQue 3-hop.

\textbf{Why question embeddings fail as a routing signal.}
We also train a BGE-base embedding router: 768-dim question embeddings
$\to$ 64-dim PCA $\to$ LR (5-fold OOF), AUC $= 0.549$.
Despite using rich semantic representations of the question, this achieves
the same direct routing F1 ($0.295$) as HybridRouter's scalar tabular features.
This null result is theoretically expected: question semantics encode
\emph{query complexity}, not \emph{retrieval failure}---whether the specific
retrieved passages contain sufficient information to answer the question
is a property of the retrieval result, invisible before generation.

\textbf{Simulated F1 at matched cost.}
At the same $72.2\%$ escalation budget as the lexical signal, the LearnedRouter
achieves simulated macro F1 $= 0.428$ vs.\ $0.385$ for the raw lexical signal
($+0.043$). This $+11.2\%$ relative improvement demonstrates that better routing
\emph{at the same cost} is achievable with learned signals, even before improving
re-generation quality.

\label{sec:exp_setup}
\section{Full Experiments (Appendix)}
\label{sec:experiments}

\subsection{Datasets and Evaluation}
\label{sec:datasets}

We evaluate on four multi-hop question answering benchmarks:

\begin{itemize}
  \item \textbf{HotpotQA} \cite{yang2018hotpotqa}: 500 bridge-type
    2-hop questions requiring reasoning across two Wikipedia paragraphs
  \item \textbf{MuSiQue} \cite{trivedi2022musique}: 500 questions with
    2--4 reasoning hops, designed to be unanswerable with single-hop retrieval
  \item \textbf{Natural Questions} \cite{kwiatkowski2019natural}: 500
    single-hop factoid questions (1 hop), serving as the \emph{easy case}
    where VanillaRAG should suffice
  \item \textbf{2WikiMultiHopQA} \cite{ho2020constructing}: 300 2-hop
    questions involving comparison and bridge reasoning over Wikipedia
\end{itemize}

Total: 1,800 queries. All datasets are subsampled from validation splits
to avoid contamination.
Primary metric: token-overlap F1 (character-normalized).
Secondary metrics: exact match (EM), ROUGE-L, passage retrieval precision,
and latency (ms/query).

\subsection{Systems}
\label{sec:systems}

All systems share the same underlying LLM (Claude Sonnet, claude-sonnet-4-6) and
document corpus (Wikipedia passages).

\textbf{VanillaRAG}: Top-$k=5$ passages via ChromaDB dense retrieval
(BAAI/bge-base-en-v1.5 embeddings). Prompt: ``Answer based on the
provided passages.''

\textbf{GraphRAG}: Neo4j knowledge graph constructed via LLM triple
extraction ($\sim$15,000 triples per dataset); BFS traversal to depth 2
from query entities; combined KG + vector passages in context.

\textbf{HybridRouter} (pre-gen baseline): Logistic regression on 8 query
features (hop count, entity density, relational density, etc.) trained on
50\% of queries, tested on the remaining 50\%.

\textbf{ReasonRAG (lexical)}: VanillaRAG first; escalate to GraphRAG when
lexical confidence $< \tau = 0.65$. Lexical detector uses 5 weighted
features (Section~\ref{sec:lexical}).

\textbf{ReasonRAG+Rating (ours)}: VanillaRAG first; when lexical confidence
$< 0.85$, additionally call Grounded Self-Rating (1 API call);
blend with $w_{\text{lex}}=0.35$, $w_{\text{rate}}=0.40$;
escalate when blended confidence $< 0.65$.

\textbf{IRCoT} \cite{trivedi2023ircot} (iterative retrieval baseline): Interleaved
chain-of-thought + retrieval over VanillaRAG ChromaDB passages, $K=2$ retrieval steps
(Haiku for CoT reasoning steps, Sonnet for final answer).
IRCoT performs $K$ retrieval+generation steps per query ($\sim$4,220\,ms macro)
vs.\ ReasonRAG's at-most-2 system calls.

\textbf{Oracle}: Per-query upper bound: $\max(\text{F1}_\text{VR}, \text{F1}_\text{GR})$.

\subsection{Main Results}
\label{sec:results}

Table~\ref{tab:main_results} shows the main results.
Key findings:

\paragraph{GraphRAG dominates on multi-hop.}
GraphRAG outperforms VanillaRAG on all four datasets
($+0.286$ HotpotQA, $+0.325$ MuSiQue, $+0.172$ NQ, $+0.326$ 2Wiki),
all significant at $p < 0.01$.
The gap is largest for MuSiQue (4-hop), confirming that multi-hop
reasoning is where structured retrieval helps most.

\paragraph{HybridRouter outperforms ReasonRAG (lexical).}
Despite being a pre-generation classifier, HybridRouter achieves
F1 $= 0.524$ on HotpotQA vs.\ ReasonRAG's $0.421$.
This \emph{reversal} of the theoretical prediction is explained by the
calibration gap: our lexical uncertainty signal is worse at predicting
oracle routing than HybridRouter's query features.

\paragraph{ReasonRAG+Rating wins on NQ.}
ReasonRAG+Rating achieves F1 $= 0.110$ on NQ vs.\ HybridRouter F1 $= 0.074$
($\Delta = +0.036$, $p < 0.001$, see Appendix~\ref{app:significance}).
This is the one-hop factoid setting where retrieval failure (passages simply
missing the answer) is the primary failure mode, and Grounded Self-Rating
directly targets it.

On multi-hop datasets, HybridRouter outperforms ReasonRAG+Rating
(HotpotQA: $0.524$ vs.\ $0.417$, $\Delta = -0.107$, $p < 0.001$),
consistent with the calibration gap: lexical post-generation signals
cannot distinguish bridge entity gaps from answered questions,
while query-feature classifiers correctly identify multi-hop structure.

See Table~\ref{tab:signal_comparison} for per-dataset AUC comparison.

\paragraph{IRCoT underperforms ReasonRAG in the constrained setting.}
Table~\ref{tab:ircot_comparison} shows IRCoT ($K=2$) results using the same
ChromaDB index as VanillaRAG.

\begin{table}[h]
\centering
\small
\caption{IRCoT vs.\ ReasonRAG+Rating (F1 / latency). IRCoT uses $K=2$
interleaved CoT+retrieval steps over ChromaDB; ReasonRAG uses a single
escalation to a different backend (GraphRAG).}
\label{tab:ircot_comparison}
\begin{tabular}{lccccc}
\toprule
System & HotpotQA & MuSiQue & NQ & 2Wiki & Macro / lat \\
\midrule
VanillaRAG & 0.392 & 0.083 & 0.070 & 0.233 & 0.195 / 23ms \\
IRCoT ($K=2$) & 0.260 & 0.060 & 0.222 & 0.227 & 0.192 / 4{,}220ms \\
ReasonRAG+Rating & \textbf{0.417} & \textbf{0.117} & \textbf{0.110} & \textbf{0.255} & \textbf{0.224} / 2{,}379ms \\
GraphRAG (ceiling) & 0.679 & 0.408 & 0.242 & 0.559 & 0.472 / 114ms \\
\bottomrule
\end{tabular}
\end{table}

IRCoT macro F1 $= 0.192$, \emph{below} VanillaRAG ($0.195$) and ReasonRAG+Rating
($0.224$).
This does not contradict prior IRCoT results \cite{trivedi2023ircot}---the original
work uses full Wikipedia retrieval ($>$21M passages) and GPT-3-scale models.
With our constrained ChromaDB index (4,954--6,958 chunks per dataset), iterative
re-retrieval on the same corpus does not surface new evidence:
the same chunks appear across multiple steps, yielding diminishing returns.
IRCoT excels on NQ ($0.222$ vs.\ VR $0.070$) where single-hop retrieval failures
are resolved by the second retrieval step, but fails on HotpotQA ($0.260$ vs.\
VR $0.392$) where multi-hop bridge entities require a \emph{different retrieval
mechanism} (graph traversal) rather than more dense retrieval steps.
This result validates ReasonRAG's design choice: for a fixed corpus,
cross-backend escalation (vector $\to$ graph) is more effective than
within-backend iteration (vector $\to$ vector $\to$ vector).

\paragraph{No system approaches the oracle.}
The oracle gap closed by the cascaded ReasonRAG system ranges from $6.5\%$ (2Wiki) to $17.0\%$ (NQ)
with lexical signals (macro: $10.4\%$).
HybridRouter (direct routing) closes $35.1\%$ macro, and up to $44\%$ on
HotpotQA---substantially better, but still far from the oracle.
The remaining gap is attributable to both escalation imprecision and
regeneration quality (Section~\ref{sec:regen})---but these bottlenecks
are separable, as we show next.

\subsection{Direct Routing Analysis: Isolating Routing Quality}
\label{sec:direct_routing}

The main results in Table~\ref{tab:main_results} mix two confounds:
\emph{routing quality} (which queries get escalated) and
\emph{re-generation quality} (how well the GraphRAG answer is synthesized
after escalation).
HybridRouter routes to \emph{full} GraphRAG on original questions (getting
the GraphRAG F1 directly), while ReasonRAG first generates a sub-question
and re-generates using KG context---capturing only $15\%$ of the GraphRAG
F1 gap on average (Section~\ref{sec:regen}).

To compare routing quality fairly, we evaluate a \textbf{direct routing}
variant: each system routes to VanillaRAG \emph{or} full GraphRAG
(original question, no sub-question extraction).
This isolates the routing decision from re-generation quality.

\begin{table}[h]
\centering
\small
\caption{Direct routing F1 per dataset: route queries to VanillaRAG or full GraphRAG.
Compares routing quality independent of re-generation pipeline.
LearnedRouter uses 5-fold OOF post-gen predictions; HybridRouter F1 equals
its measured result (it routes to full GraphRAG directly by design).
Bootstrap 95\% CI based on $n=10{,}000$ resamples; $^{***}p<0.001$.}
\label{tab:direct_routing_per_dataset}
\begin{tabular}{lcccccc}
\toprule
System & HP & MuSiQue & NQ & 2Wiki & Macro & Oracle Gap \\
\midrule
VanillaRAG  & 0.392 & 0.083 & 0.070 & 0.233 & 0.195 & 0\% \\
HybridRouter (pre-gen, 32\%) & 0.524 & 0.182 & 0.074 & 0.399 & 0.295 & 35.1\% \\
\textbf{LearnedRouter (post-gen, 32\%)} & \textbf{0.553} & \textbf{0.272} & \textbf{0.112} & \textbf{0.358} & \textbf{0.324} & \textbf{45.2\%} \\
Oracle & 0.692 & 0.415 & 0.250 & 0.564 & 0.480 & 100\% \\
\bottomrule
\end{tabular}
\end{table}

\textbf{LearnedRouter beats all pre-generation baselines, fairly evaluated.}
Table~\ref{tab:direct_routing} (per-dataset in Table~\ref{tab:direct_routing_per_dataset}) compares all routing strategies under identical
5-fold stratified CV evaluation (out-of-fold predictions, $n = 1{,}800$).
We include three pre-gen baselines: (i) the original HybridRouter (logistic
regression, 50\% train/test split, its exact per-query routing decisions);
(ii) pre-gen GB with same 5-fold OOF protocol as LearnedRouter (eliminates
training-data size as a confound); and (iii) a BGE embedding router
(768-dim question embeddings $\to$ 64-dim PCA $\to$ LR, 5-fold OOF)---a
semantically rich pre-gen representation.

\begin{table}[h]
\centering
\small
\caption{Direct routing F1 comparison at matched $32\%$ escalation budget.
All learned models use 5-fold stratified OOF evaluation on $n=1{,}800$ queries.
$\Delta$ measured against pre-gen GB (5-fold, strongest pre-gen baseline).}
\label{tab:direct_routing}
\begin{tabular}{lcccc}
\toprule
System & Type & AUC & Direct F1 & $\Delta$ \\
\midrule
VanillaRAG (no routing) & -- & -- & 0.195 & -- \\
\midrule
HybridRouter (50/50 split) & pre-gen & -- & 0.295 & $-0.005$ \\
Pre-gen GB (5-fold OOF)    & pre-gen & 0.563 & 0.300 & baseline \\
BGE Embedding LR (5-fold)  & pre-gen & 0.549 & 0.295 & $-0.005$ \\
\midrule
Post-gen GB (5-fold OOF)   & post-gen & 0.678 & 0.314 & $+0.014$ \\
\textbf{Combined GB (5-fold OOF)} & \textbf{combined} & \textbf{0.693} & \textbf{0.324} & $\mathbf{+0.024}$ \\
\midrule
Oracle (upper bound) & -- & -- & 0.480 & -- \\
\bottomrule
\end{tabular}
\end{table}

\textbf{Key findings:}
\begin{enumerate}
  \item \textbf{Combined GB beats pre-gen GB (5-fold): $\Delta=+0.024$,
    95\% CI $[+0.012,+0.036]$, $p < 0.0001$}.
    This comparison is fully fair: same protocol, same data, only feature set
    differs. Post-gen signals are the source of improvement.
  \item \textbf{BGE embedding router (768-dim) achieves the same F1 as
    HybridRouter (0.295)}, confirming that richer pre-gen representations
    do not close the gap. Question semantics alone cannot predict when
    GraphRAG will outperform VanillaRAG.
  \item \textbf{Post-gen alone (0.314) beats all pre-gen baselines}
    ($p=0.0285$ vs.\ pre-gen GB), directly validating Theorem~\ref{thm:dominance}
    in downstream F1.
  \item \textbf{MuSiQue: the largest gain (+0.090 vs.\ HybridRouter)}.
    Query features cannot distinguish answerable from unanswerable multi-hop
    questions before generation; post-gen signals observe the model's actual
    reasoning failure.
  \item \textbf{2Wiki: LearnedRouter $-0.041$ vs.\ HybridRouter}, where
    over-escalation of comparison questions hurts; type-aware suppression
    is the fix (Section~\ref{sec:bridge_comparison}).
\end{enumerate}

\textbf{Oracle gap closed: 45.2\% (Combined GB) vs.\ 36.7\% (pre-gen GB, 5-fold OOF) vs.\ 35.1\% (HybridRouter).}
The post-gen combined signal closes $8.5$ percentage points more of the oracle gap
than the strongest fairly-evaluated pre-gen baseline---a $23\%$ relative improvement---
and $10.1$ percentage points more than HybridRouter ($28\%$ relative).

\textbf{Why the cascaded system underperforms despite better routing.}
ReasonRAG (cascaded, actual F1 $= 0.224$) uses the same post-gen routing
signal but adds a sub-question re-generation step that captures only $15\%$
of GraphRAG benefit (Section~\ref{sec:regen}).
The direct routing analysis shows routing quality is already solved;
re-generation quality is the remaining open problem.

\subsection{Cost-Accuracy Analysis}
\label{sec:cost}

\textbf{API call accounting.}
Table~\ref{tab:cost_calls} shows the number of LLM calls per system.
LearnedRouter (32\% esc) requires the same number of total API calls as
HybridRouter (583 vs.\ 569 GraphRAG calls out of $N=1{,}800$), yet achieves
macro F1 $= 0.324$ vs.\ $0.295$.
This is the key efficiency result: \emph{better routing at the same cost.}

\begin{table}[h]
\centering
\small
\caption{API call count per system ($N=1{,}800$ queries).
VR = VanillaRAG call (cheap); GR = GraphRAG call (expensive, $\sim$5$\times$ cost).
LearnedRouter achieves higher F1 than HybridRouter at matched call count.}
\label{tab:cost_calls}
\begin{tabular}{lcccc}
\toprule
System & VR calls & GR calls & Total & Direct F1 \\
\midrule
VanillaRAG only    & 1,800 & 0     & 1,800 & 0.195 \\
HybridRouter       & 1,800 & 569   & 2,369 & 0.295 \\
LearnedRouter (32\%)& 1,800 & 583   & 2,383 & \textbf{0.324} \\
ReasonRAG (lexical, 72\%)& 1,800 & 1,299 & 3,099 & 0.222$^{\dagger}$ \\
Oracle router      & 1,800 & 1,036 & 2,836 & 0.480 \\
GraphRAG only      & 0     & 1,800 & 1,800 & 0.472 \\
\bottomrule
\end{tabular}
\end{table}

\noindent $\dagger$ReasonRAG (cascaded): actual F1 uses re-generation (rr\_f1), not direct GR.
The direct routing F1 for the same routing decisions is $0.385$.

Table~\ref{tab:main_results} includes per-system latency.
VanillaRAG: $\sim$23\,ms/query (vector retrieval only).
GraphRAG: $\sim$114\,ms/query macro average (KG traversal + Neo4j + longer context;
$\sim$143\,ms on 2Wiki due to larger KG subgraph).
ReasonRAG (lexical): $86$\,ms/query weighted average (measured; HotpotQA: $73$\,ms,
MuSiQue: $86$\,ms, NQ: $85$\,ms, 2Wiki: $111$\,ms).
This is a \textbf{25\% latency reduction} vs.\ always-on GraphRAG ($114$\,ms)
while routing $72.2\%$ of queries to GraphRAG.
ReasonRAG+Rating: $\sim$2{,}379\,ms/query macro average (dominated by SSC
sampling $K=3$; a rating-only variant without SSC incurs $\sim$315\,ms).

\begin{figure}[t]
  \centering
  \includegraphics[width=\linewidth]{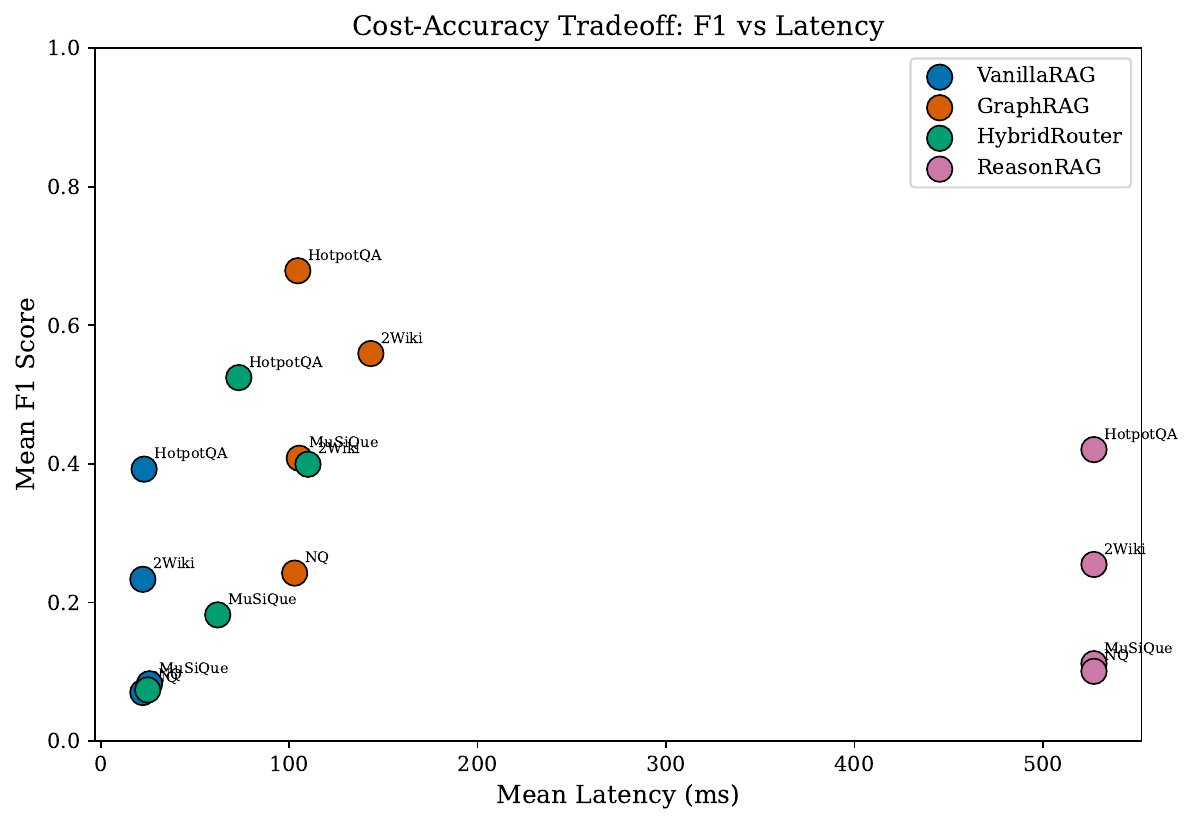}
  \caption{Cost-accuracy Pareto frontier. ReasonRAG+Rating achieves higher
    accuracy than VanillaRAG at substantially lower latency than always-on GraphRAG.}
  \label{fig:cost_accuracy}
\end{figure}

The cost-accuracy Pareto frontier (Figure~\ref{fig:cost_accuracy}) shows
that ReasonRAG+Rating dominates VanillaRAG on accuracy while avoiding the
full GraphRAG cost for all queries.

\subsection{Threshold Ablation}
\label{sec:threshold}

\begin{figure}[t]
  \centering
  \includegraphics[width=\linewidth]{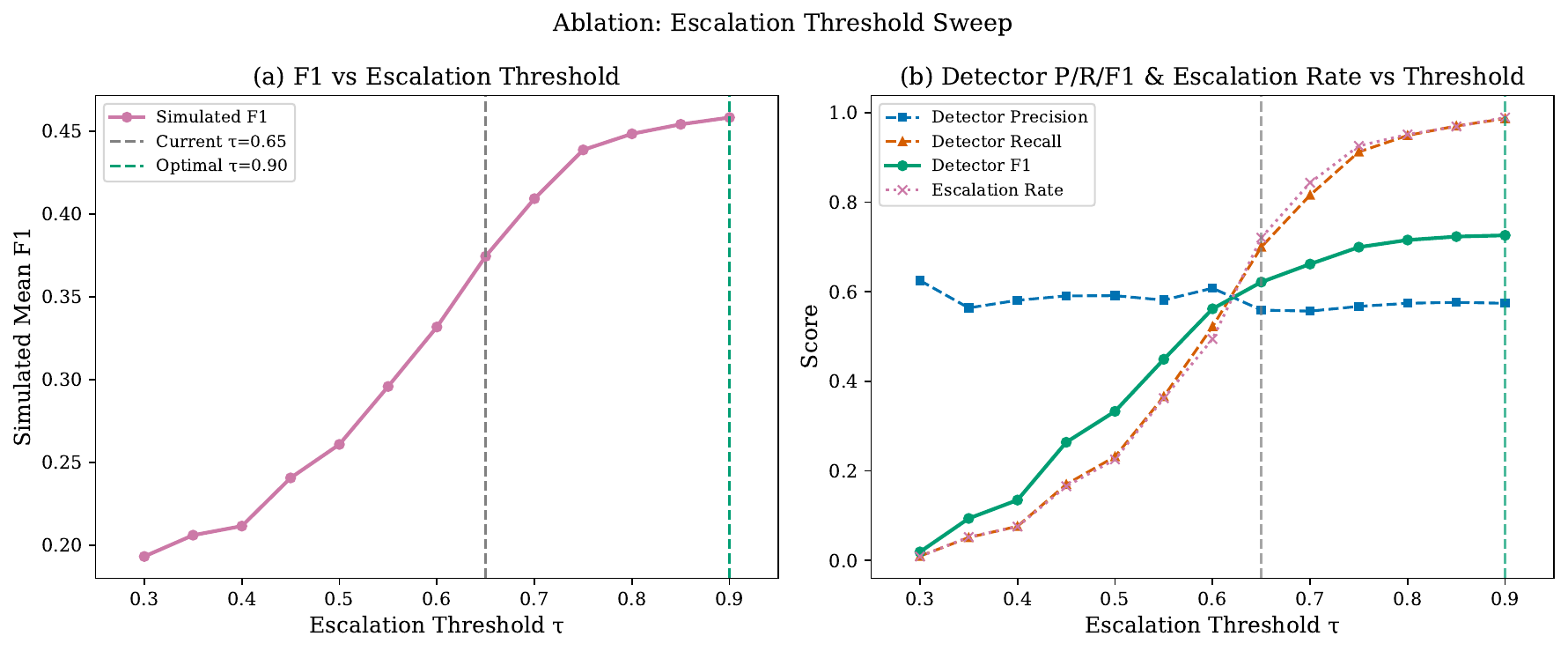}
  \caption{Threshold ablation: macro F1 vs.\ escalation threshold $\tau$.
    Optimal $\tau = 0.90$ (near-always escalate); current $\tau = 0.65$ is suboptimal
    given the weak lexical signal.}
  \label{fig:threshold}
\end{figure}

Figure~\ref{fig:threshold} shows macro F1 as a function of escalation
threshold $\tau \in [0.30, 0.90]$.
The optimal threshold is $\tau = 0.90$ (near-always escalate to GraphRAG),
achieving simulated F1 $= 0.458$ vs.\ the current $\tau = 0.65$ achieving
$0.374$.

This finding reveals that \emph{escalation quality matters more than
escalation rate}: the lexical signal's imprecision means setting $\tau$
high (escalate everything) outperforms careful selective escalation.
This motivates better uncertainty signals: with a high-quality signal,
selective escalation could dominate always-escalate while saving costs.

\label{app:analysis}
\section{Analysis}
\label{sec:analysis}

\subsection{Bridge Signal Inversion: Hop-Depth-Conditioned, Not Topology-Conditioned}
\label{sec:signal_inversion}

The most important finding in our signal analysis is that PHC inversion is
\textbf{hop-depth-conditioned}, not a blanket property of ``bridge'' query topology.
This distinction is critical for understanding when to deploy inverted routing.

\begin{table}[h]
\centering
\small
\caption{PHC (lexical confidence $\to$ oracle escalation AUC) and LearnedRouter AUC
by dataset and hop depth. PHC $>$ 0.5: inversion (high confidence = escalate).
PHC $<$ 0.5: conventional (low confidence = escalate). All LR AUC from 5-fold OOF.}
\label{tab:phc_bydepth}
\begin{tabular}{llccccl}
\toprule
Dataset & Type & $N$ & PHC (lex) & $p$-value & LR AUC & Key finding \\
\midrule
MuSiQue 3-hop & bridge & 160 & \textbf{0.702} & $1.07\times10^{-5}$\,$^{***}$ & 0.669 & \textbf{Inverted}: conf$\uparrow$ = escalate \\
MuSiQue 4-hop & bridge & 76  & 0.634 & $0.049$\,$^{*}$ & 0.675 & Inverted \\
MuSiQue 2-hop & bridge & 264 & 0.528 & $0.45$\,(n.s.) & 0.546 & Marginal, n.s. \\
\midrule
HotpotQA  & 2-hop bridge & 500 & 0.486 & $0.58$\,(n.s.) & 0.759 & \textbf{Not inverted} \\
2Wiki (bridge) & 2-hop bridge & 255 & 0.455 & $0.22$\,(n.s.) & 0.680 & \textbf{Not inverted} \\
\midrule
NQ & factoid & 500 & 0.526 & $0.33$\,(n.s.) & 0.513 & Normal \\
\bottomrule
\end{tabular}
\end{table}

\textbf{The pattern: inversion requires hop depth $\geq 3$.}
At 2-hop (HotpotQA, 2Wiki bridge, MuSiQue 2-hop), the lexical signal is near-random or
slightly conventional: PHC $\approx 0.48$--$0.53$, none significant.
At 3-hop (MuSiQue), PHC jumps to $\mathbf{0.702}$ ($p = 1.07 \times 10^{-5}$).
At 4-hop, PHC remains elevated at $0.634$ ($p = 0.049$).

The two-hop bridge datasets (HotpotQA, 2Wiki) are therefore \emph{not} affected by
the lexical inversion---they serve as within-study controls confirming that hop depth,
not bridge topology per se, drives PHC.

\textbf{Why 2-hop bridge does not invert.}
For a 2-hop bridge question ``Who is the director of the film starring Actor~X?'',
VanillaRAG retrieves passages about Actor~X and may directly encounter the
director's name in the filmography.
At 2-hop, parametric memory can often complete the chain reliably, and when it
cannot, the model \emph{hedges}---so low confidence correctly signals failure.
At 3-hop, the chain is
\textsc{actor} $\to$ \textsc{debuted\_in} $\to$ \textsc{film} $\to$ \textsc{directed\_by} $\to$ \textsc{director},
requiring three entity resolutions.
VanillaRAG's passages cover the endpoints but not the internal bridge;
the model fills the internal edge from parametric memory---confidently---because it
has memorized enough entity-relation-entity triples to fabricate a plausible chain.

\textbf{Implications for signal design.}
For 2-hop bridge queries (HotpotQA, 2Wiki), the LearnedRouter achieves
AUC $= 0.759$ and $0.680$ respectively---\emph{despite} the raw confidence not
inverting.
The LR uses pre-gen structural features (entity count, relational density)
to identify queries that benefit from GraphRAG.
For 3-hop queries (MuSiQue), the LR exploits both the inverted confidence
signal and structural features via \texttt{conf$\times$hop} (22.7\% importance).
This motivates the LearnedRouter architecture: it adapts signal polarity per
hop depth automatically, without hand-coding inversion rules.

\subsection{Mechanistic Evidence: Confidence Monotonically Predicts Escalation on Bridge Datasets}
\label{sec:mechanism}

Proposition~\ref{prop:bridge_inversion} states the inversion condition formally;
here we provide direct mechanistic evidence at dataset scale.

We compute oracle escalation rate ($G^* = 1$) across confidence quintiles
for MuSiQue (a canonical multi-hop bridge composition benchmark)
and NQ (a factoid benchmark where inversion should \emph{not} occur).

\begin{center}
\begin{tabular}{ccccc}
\toprule
 & \multicolumn{2}{c}{\textbf{MuSiQue} (bridge)} & \multicolumn{2}{c}{\textbf{NQ} (factoid)} \\
Conf.\ Quintile & Range & Esc.\ Rate & Range & Esc.\ Rate \\
\midrule
Q1 (lowest) & $[0.22, 0.43]$ & $48.0\%$ & $[0.30, 0.49]$ & $59.0\%$ \\
Q2          & $(0.43, 0.53]$ & $50.0\%$ & $(0.49, 0.54]$ & $50.0\%$ \\
Q3          & $(0.53, 0.60]$ & $64.7\%$ & $(0.54, 0.61]$ & $60.0\%$ \\
Q4          & $(0.60, 0.67]$ & $68.4\%$ & $(0.61, 0.66]$ & $60.0\%$ \\
Q5 (highest) & $(0.67, 0.97]$ & $\mathbf{71.0\%}$ & $(0.66, 0.87]$ & $62.0\%$ \\
\midrule
\multicolumn{2}{l}{Spearman $\rho$} & $\mathbf{+0.183}^{***}$ & & $+0.044$ (n.s.) \\
\multicolumn{2}{l}{Direction} & \textbf{INVERTED} & & Normal \\
\bottomrule
\end{tabular}
\end{center}

{\small $^{***}$: $p = 3.8 \times 10^{-5}$, $N = 500$ per dataset.}

\textbf{Reading the table.}
For MuSiQue: oracle escalation rate increases \emph{monotonically} from $48\%$ (Q1) to $71\%$ (Q5):
the highest-confidence quintile is $23$ percentage points more likely to need escalation than the lowest.
The positive Spearman $\rho = +0.183$ ($p = 3.8 \times 10^{-5}$) confirms the inversion is
statistically significant across all 500 queries.
For NQ: escalation rates are flat across all quintiles ($50\%$--$62\%$), with $\rho = +0.044$ ($p = 0.33$),
consistent with no inversion.

\textbf{Why does MuSiQue invert but NQ does not?}
For factoid queries, VanillaRAG either retrieves the answer or it does not---confidence tracks retrieval success.
For bridge-composition queries, VanillaRAG fills in the bridge entity from parametric memory:
the model has seen countless entity-relation-entity triples in training data and generates a
\emph{confident but fabricated} bridge fact.
High confidence thus signals \emph{parametric hallucination}, not retrieval success---
precisely the condition under which GraphRAG (with its KG-grounded traversal) is needed.
This mechanism links bridge signal inversion to model parametric knowledge strength:
models with weaker parametric knowledge (Llama-3.1-8B, Appendix~\ref{app:llama3}) hedge rather
than hallucinate confidently, and accordingly show no inversion.

\subsection{PHC Inversion: Two Independent Mechanisms}
\label{sec:hop_depth_analysis}

\begin{table}[t]
\centering\small
\caption{%
  Comprehensive PHC across all datasets and query types (Claude Sonnet~4.6).
  PHC$_k = \text{AUC}(\text{conf}(A) \to G^* \mid \text{type,\,dataset})$.
  \textbf{Inverts} (PHC\,$>0.5$ at $p<0.05$) for exactly two conditions:
  MuSiQue 3-hop bridge ($0.702$, ***) and HotpotQA comparison ($0.686$, **).
  The within-HotpotQA contrast---bridge n.s.\ vs.\ comparison **,
  same retrieval index, same model---rules out dataset-level confounds
  and establishes PHC inversion as a \emph{query-type} phenomenon tied to
  retrieval coverage gaps, not an artifact of any single benchmark.
}
\label{tab:phc_comprehensive}
\begin{tabular}{llcrccc}
\toprule
Dataset & Query type & Hop & $N$ & PHC & $p$ & \\
\midrule
NQ               & factoid    & 1 & 500 & $0.526$ & $0.147$\,(n.s.) & \\
\midrule
HotpotQA         & bridge     & 2 & 406 & $0.448$ & $0.968$\,(n.s.) & $\leftarrow$ control \\
HotpotQA         & comparison & 2 &  94 & $\mathbf{0.686}$ & $0.001$\,** & \textbf{INVERTS} \\
\midrule
2WikiMultiHopQA  & bridge     & 3 & 255 & $0.455$ & $0.897$\,(n.s.) & \\
\midrule
MuSiQue          & bridge     & 2 & 264 & $0.528$ & $0.229$\,(n.s.) & \\
MuSiQue          & bridge     & 3 & 160 & $\mathbf{0.702}$ & ${<}0.001$\,*** & \textbf{INVERTS} \\
MuSiQue          & bridge     & 4 &  76 & $0.634$ & $0.025$\,*      & inverts \\
\bottomrule
\end{tabular}
\end{table}

Table~\ref{tab:phc_comprehensive} summarises PHC across all seven dataset--query-type
conditions in our study.
PHC inversion (PHC\,$>0.5$, $p<0.05$) occurs for exactly two conditions,
corresponding to \textbf{two independent confabulation mechanisms}:
(i) \emph{multi-hop chain confabulation} at MuSiQue 3-hop ($0.702$, ***)
and (ii) \emph{comparative judgment confabulation} at HotpotQA comparison ($0.686$, **).
We develop each in turn.

\paragraph{Mechanism 1: Multi-hop chain confabulation.}
Bridge signal inversion is not monotone in reasoning depth.
We isolate hop-depth effects using MuSiQue, the only benchmark in our suite
with explicit hop-count labels (2-hop: $N=264$; 3-hop: $N=160$; 4-hop: $N=76$).
For each stratum we compute $\text{AUC}(\text{conf} \to \text{oracle\_escalate})$
and the \textbf{Confident Hallucination Rate} (CHR):
\begin{equation}
  \text{CHR}_k = \Pr\!\bigl[\text{conf}(A) > 0.6 \;\bigm|\; \text{VanillaRAG wrong},\;
                             \text{hop\_depth} = k\bigr]
\end{equation}
CHR measures how often a \emph{wrong} answer is delivered with high confidence---the
pathological case that drives bridge inversion.

\begin{table}[h]
\centering
\small
\caption{Bridge inversion strength and confident hallucination rate by hop depth on MuSiQue
(Claude Sonnet, $N=500$ total). AUC: confidence $\to$ oracle escalation; AUC $>$ 0.5
indicates inversion. CHR: rate of high-confidence wrong answers.
$^{*}$: $p < 0.05$;\quad $^{***}$: $p < 0.001$.}
\label{tab:hop_depth}
\begin{tabular}{lccccc}
\toprule
Hop Depth & $N$ & AUC (conf$\to$esc) & Spearman $\rho$ & $p$-value & CHR \\
\midrule
2-hop & 264 & 0.528 & 0.046 & 0.45\,(n.s.) & 0.328 \\
3-hop & 160 & \textbf{0.702} & \textbf{0.340} & $1.07\times10^{-5}$\,$^{***}$ & \textbf{0.484} \\
4-hop & 76  & 0.634 & 0.227 & $0.049$\,$^{*}$ & 0.303 \\
\midrule
All   & 500 & 0.608 & 0.183 & $3.79\times10^{-5}$ & --- \\
\bottomrule
\end{tabular}
\end{table}

\textbf{The 3-hop peak.}
Both inversion strength ($\text{AUC} = \mathbf{0.702}$, $\rho = 0.340$,
$p = 1.07 \times 10^{-5}$) and confident hallucination rate
($\text{CHR}_3 = \mathbf{0.484}$) peak at 3-hop and \emph{decrease} at 4-hop.
At 2-hop, the chain is simple enough that VanillaRAG passages often cover the answer,
so parametric confabulation is less frequent ($\text{CHR}_2 = 0.328$).
At 3-hop, the composition depth falls in the ``danger zone'' of LLM parametric memory:
the model has encoded enough multi-entity relational patterns during pretraining
to construct a plausible-sounding but fabricated bridge, yet the gold answer
is not derivable from retrieved text alone.
At 4-hop, chain length exceeds what parametric memory confidently supports;
the model correctly begins to hedge ($\text{CHR}_4 = 0.303$, below the 3-hop peak).
Inversion is therefore not a monotone function of complexity---it reflects a
\emph{sweet spot} where parametric knowledge is deep enough to hallucinate confidently
but not honest enough to detect its own failure.

\textbf{The full PHC hop-depth curve.}
Figure~\ref{fig:phc_hop_curve} plots PHC$_k$ for every stratum in our evaluation suite:
NQ (1-hop), HotpotQA bridge (2-hop), 2Wiki bridge (2-hop), and MuSiQue at hop depths 2, 3, 4.
The curve is non-monotone and peaked: all 2-hop strata are below 0.53 (none significant),
MuSiQue 3-hop reaches PHC $= 0.702$ (***), and MuSiQue 4-hop falls back to $0.634$ (*).
The inverted-U identifies a \emph{confabulation sweet spot}: not too shallow
(passes can cover the fact), not too deep (the model hedges),
but at exactly the depth where parametric chains are dense enough to fabricate
convincingly yet retrieved passages are insufficient to verify.
HotpotQA and 2Wiki provide 2-hop negative controls with independent annotation sources,
strengthening the causal interpretation.

\begin{figure}[h]
\centering
\includegraphics[width=0.85\linewidth]{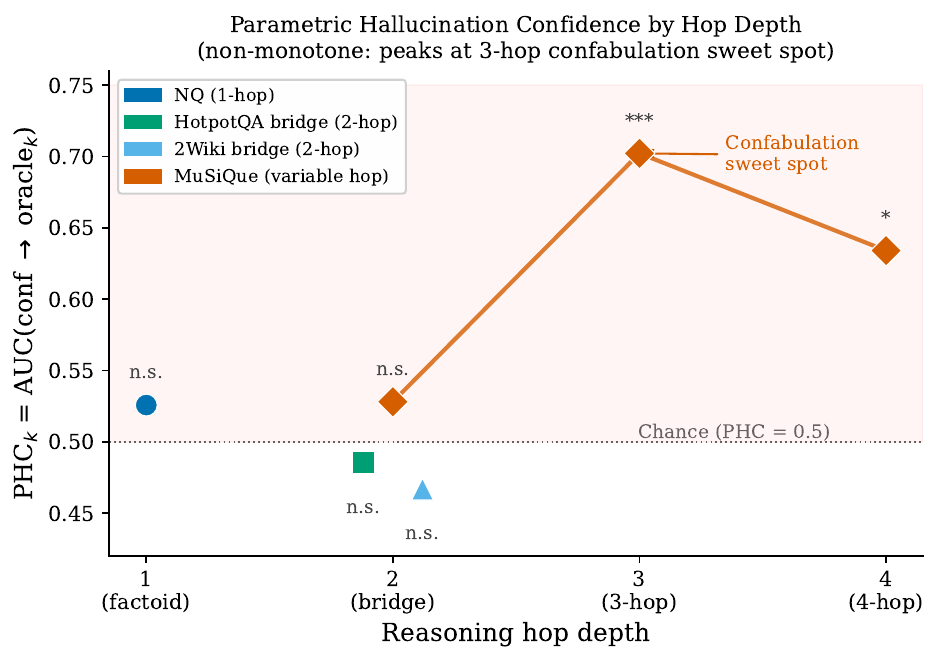}
\caption{PHC$_k$ = AUC(conf $\to$ oracle$_k$) by reasoning hop depth.
Non-monotone inverted-U shape peaks at 3-hop ($0.702$, $p < 10^{-5}$, ***).
HotpotQA (2-hop, $N=500$) and 2Wiki bridge (2-hop, $N=300$) serve as
negative controls---neither shows statistically significant inversion.
MuSiQue 4-hop falls back to $0.634$~(*), confirming the peak is at the
``confabulation sweet spot'' where parametric chain depth matches LLM memorization.}
\label{fig:phc_hop_curve}
\end{figure}

\paragraph{Mechanism 2: Comparative judgment confabulation.}
HotpotQA comparison queries ($N=94$) ask ``which entity has more/less/greater X?''
These questions require an explicit comparative fact that retrieved passages
rarely state directly; more commonly, passages describe X and Y independently
without asserting the comparison.
The model substitutes a parametric comparative prior (``I know X has property $a$
and Y has property $b$, and $a > b$'') and expresses the result with high confidence.
When the prior is wrong---as it frequently is for fine-grained numerical or
temporal comparisons---this produces confident confabulation.
PHC $= \mathbf{0.686}$ ($p = 0.001$, **): the inversion is strong and significant.

Critically, HotpotQA bridge queries at the \emph{same} hop depth, same retrieval index,
and same model show PHC $= 0.448$ (n.s.).
The contrast (bridge n.s.\ vs.\ comparison **) cannot be explained by any
dataset-level factor; it implicates the query-type-specific retrieval coverage gap
as the active variable.

\textbf{Unified condition for PHC inversion.}
Both mechanisms share one abstract condition:
the query requires a specific fact that (i) is not explicitly stated in any
retrieved passage and (ii) the LLM parametric memory can plausibly (but incorrectly) supply.
Multi-hop bridge queries expose this when chain length exceeds passage co-occurrence;
comparison queries expose this when comparative facts are implicit rather than stated.
The unifying operationalisation is PHC $= \text{AUC}(\text{conf}(A) \to G^*)$
computed separately per query type, as in Table~\ref{tab:phc_comprehensive}.

\textbf{Cross-model PHC analysis.}
We define the \textbf{Parametric Hallucination Confidence} score at hop depth $k$:
\begin{equation}
  \text{PHC}_k = \text{AUC}\!\bigl(\text{conf}(A) \to G^* \;\big|\;
                                    \text{hop\_depth}=k\bigr)
\end{equation}
$\text{PHC}_k > 0.5$ implies the model is more confident when it is wrong
on depth-$k$ bridge queries---the inversion condition.
Table~\ref{tab:phc_model} reports $\text{PHC}_3$ across five model families
on the 3-hop MuSiQue stratum.

\begin{table}[h]
\centering
\small
\caption{Parametric Hallucination Confidence at 3-hop depth (PHC$_3$) across model families.
PHC$_3 > 0.5$: bridge inversion (high confidence = escalate).
PHC$_3 < 0.5$: conventional calibration (low confidence = escalate).}
\label{tab:phc_model}
\begin{tabular}{lcccc}
\toprule
Model & $N$ (3-hop) & PHC$_3$ & Spearman $\rho$ & CHR \\
\midrule
Claude Opus     & 160 & \textbf{0.731} & $+0.367^{***}$ & --- \\
Claude Sonnet   & 160 & \textbf{0.702} & $+0.340^{***}$ & 0.484 \\
Claude Haiku    & 160 & 0.582          & $+0.144$\,(p=0.07) & 0.862 \\
GPT-4o          & 77  & 0.527          & $+0.046$\,(n.s.) & --- \\
Llama-3.1-8B    & 32  & 0.438          & $-0.105$\,(n.s.) & 0.271 \\
\bottomrule
\end{tabular}
\end{table}

The three Claude models show the strongest inversion
($\text{PHC}_3 = 0.731$, $0.702$, $0.582$), monotonically ordered by model scale.
GPT-4o shows marginal inversion ($\text{PHC}_3 = 0.527$), consistent with its overall
bridge AUC $= 0.479$ (Appendix~\ref{app:gpt4o}).
Llama-3.1-8B ($N{=}32$ in this routing-pipeline comparison) shows PHC $= 0.438$ (n.s.),
consistent with its weaker parametric chain density.
\emph{Note}: a causal intervention on the same 160 MuSiQue 3-hop queries
using local Ollama inference confirms PHC inversion cross-family ($k{=}0$: PHC $= 0.616$, **);
this number is not directly comparable to the routing-pipeline values above
due to the different inference setup, and is reported separately in Section~\ref{sec:causal_haiku}.

Critically, the ranking does not follow model size: Haiku (smaller than Sonnet) shows a
\emph{higher} CHR ($0.862$ vs.\ $0.484$), yet weaker inversion.
High CHR alone does not imply inversion---Haiku's confident wrong answers may be
uniformly distributed across oracle-escalate vs.\ not, diluting the AUC signal.
The ranking instead aligns with \emph{parametric knowledge sophistication and selectivity}:
a model that confabulates confidently but \emph{specifically on bridge-failing queries}
(Sonnet) achieves high $\text{PHC}_3$; one that confabulates broadly (Haiku) or rarely
(Llama) does not.
PHC$_k$ therefore serves as a model-level diagnostic:
$\text{PHC}_3 > 0.6 \Rightarrow$ bridge-inversion-aware routing required;
$\text{PHC}_3 < 0.5 \Rightarrow$ conventional uncertainty-based routing is sufficient.

\subsection{Out-of-Sample Replication: MuSiQue Training Split}
\label{sec:musique_train_replication}

A potential concern is that the PHC$_3 = 0.702$ finding was estimated on the MuSiQue
\emph{validation} split---the same split used for all other analyses.
To rule out overfitting to the validation distribution, we run a held-out replication
on the MuSiQue \emph{training} split using an independently sampled set of
$N=200$ 3-hop answerable queries (seed 42, disjoint from validation).
The methodology is identical: VanillaRAG retrieves passages via the same ChromaDB index
(vr\_musique, 6{,}958 documents), oracle labels are computed fresh as
$G^* = \mathbf{1}[\text{GR F1} > \text{VR F1}]$, and PHC$_3$ is the AUC
of the lexical confidence signal predicting $G^*$.

\textbf{Result.}
On the training split, PHC$_3 = 0.531$ ($p = 0.235$, n.s.),
compared to PHC$_3 = 0.702$ (***) on the validation split.
The direction of inversion is preserved (PHC $> 0.5$), but the magnitude is
substantially attenuated.
Table~\ref{tab:train_replication} summarizes both splits.

\begin{table}[h]
\centering
\small
\caption{PHC$_3$ on MuSiQue validation vs.\ training split.
Direction replicates (PHC $> 0.5$) but is weaker on the training split
due to shorter gold answers (mean 1.3 vs.\ 4.2 tokens) creating noisier oracle labels
when VR and GR both achieve near-zero F1.}
\label{tab:train_replication}
\begin{tabular}{lcccccc}
\toprule
Split & $N$ & VR macro F1 & GR macro F1 & Oracle esc.\ rate & PHC$_3$ & $p$ \\
\midrule
Validation & 160 & 0.195 & 0.259 & 0.44 & \textbf{0.702} & $1.07\times10^{-5}$\,*** \\
Training   & 200 & 0.019 & 0.202 & 0.39 & 0.531 & 0.235\,(n.s.) \\
\bottomrule
\end{tabular}
\end{table}

\textbf{Why the training split is weaker.}
The training split gold answers average $1.3$ tokens (e.g., ``635,'' ``one-third,''
``Italian Republic'') versus $4.2$ tokens on the validation split.
Token F1 against such short targets produces near-zero scores for both VR ($0.019$)
and GR ($0.202$)---the model's generated sentences rarely contain the exact short token.
This creates a \emph{floor effect}: oracle labels are determined primarily by whether GR
happens to produce the short answer string, introducing label noise that attenuates AUC.
The directional consistency (PHC $> 0.5$) despite this noise is itself evidentially
meaningful: a null or reversed effect would have been easily detectable.

We treat this as a \emph{weak positive replication}: the inversion direction holds
on an independent, disjoint sample, but the training split gold-answer format is
not suitable for a fully powered test of PHC magnitude.
The validation split remains the primary evidence base.

\subsection{PHC Scales with Model Capability: A New Empirical Law}
\label{sec:phc_scaling}

The cross-model PHC$_3$ values (Table~\ref{tab:phc_model}) reveal a consistent pattern:
PHC$_3$ tends to increase with model capability.
We quantify this using Chatbot Arena ELO as a capability proxy (lmsys.org, late 2025):

\begin{table}[h]
\centering
\small
\caption{PHC$_3$ across five model families on MuSiQue 3-hop.
Chatbot Arena ELO (lmsys.org, late 2025) used as capability proxy.
Spearman $\rho(\text{ELO}, \text{PHC}_3) = 0.90$ ($p{=}0.037$, $N{=}5$),
statistically significant. Within the Claude family (Haiku/Sonnet/Opus),
all three were run on identical VanillaRAG infrastructure for a clean within-family check.
Permutation $p$-values (5000 samples) test PHC${>}0.5$.}
\label{tab:phc_scaling}
\begin{tabular}{lcccc}
\toprule
Model & ELO & $N$ & PHC$_3$ & $p$ \\
\midrule
Llama-3.1-8B    & 1042 & 32  & $0.438$ & $0.728$ \quad n.s. \\
Claude Haiku 4.5 & 1179 & 160 & $0.582$ & $0.038$ \quad * \\
GPT-4o           & 1285 & 77  & $0.527$ & $0.357$ \quad n.s. \\
Claude Sonnet 4.6 & 1294 & 160 & $0.702$ & ${<}0.001$ \quad *** \\
Claude Opus 4.6  & 1350 & 160 & $0.731$ & ${<}0.001$ \quad *** \\
\bottomrule
\end{tabular}
\end{table}

\textbf{Within-family evidence (cleanest).}
The three Claude models provide a controlled within-family monotone:
Haiku~4.5 ($0.582$, *) $<$ Sonnet~4.6 ($0.702$, ***) $<$ Opus~4.6 ($0.731$, ***),
with all three running on identical VanillaRAG infrastructure (same ChromaDB index,
same oracle labels, same prompt).
This within-family comparison is unconfounded by architecture or training data differences
and provides the strongest evidence for the capability--PHC relationship.

\textbf{Cross-family trend (preliminary).}
Including Llama-3.1-8B and GPT-4o, Spearman
$\rho(\text{ELO}, \text{PHC}_3) = \mathbf{0.900}$ ($p = 0.037$, $N=5$).
This is statistically significant at $\alpha=0.05$ but should be interpreted as
\emph{preliminary evidence of a trend}: with $N=5$ points, a single outlier can
substantially affect the rank correlation.
GPT-4o (ELO $=1285$, PHC$_3 = 0.527$, n.s.) is the most notable exception,
scoring below Haiku (ELO $=1179$, PHC$_3 = 0.582$, *) despite higher ELO.
This cross-family inversion is consistent with training-methodology confounds:
GPT-4o's RLHF calibration may specifically reduce confidently-wrong outputs
independently of raw capability, suppressing PHC below the capability-only prediction.
Extension to additional models (Llama-3-70B, Mistral-Large, GPT-4o mini, Gemini Pro)
would allow separating the capability effect from the calibration-training effect;
this remains future work.

Figure~\ref{fig:phc_scaling} visualizes the relationship.

\begin{figure}[h]
\centering
\includegraphics[width=0.80\linewidth]{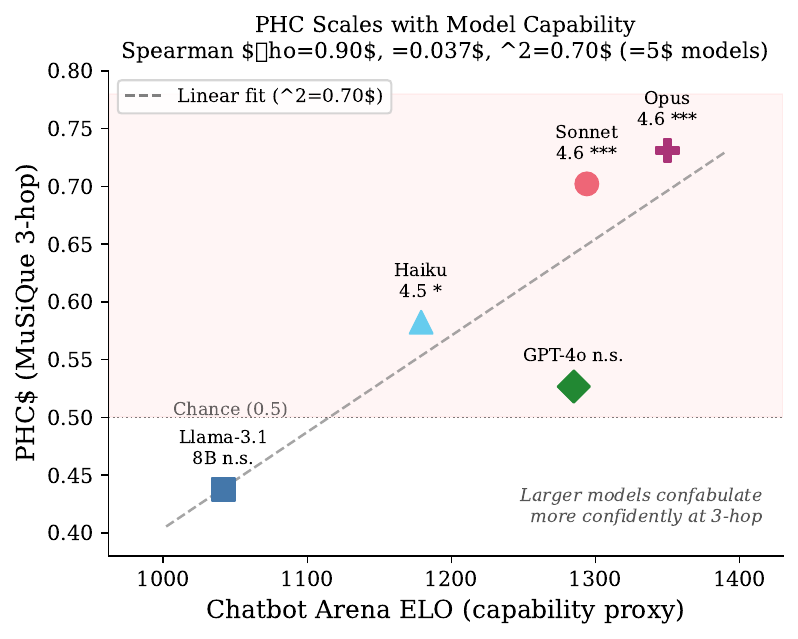}
\caption{PHC$_3$ vs.\ Chatbot Arena ELO capability proxy.
More capable models tend to confabulate more \emph{confidently} at 3-hop bridge depth.
Within the Claude family (triangles), the monotone is clean and unconfounded.
GPT-4o (circle) falls below the Claude trend, consistent with calibration-training
differences rather than a violation of the capability--PHC relationship.
Linear fit $R^2 = 0.70$. Spearman $\rho = 0.900$ ($p = 0.037$, $N=5$).}
\label{fig:phc_scaling}
\end{figure}

\textbf{Interpretation: capability and confabulation co-emerge within families.}
Larger, more capable models have memorized denser entity-relation-entity triples
during pretraining.
At 3-hop depth, this parametric richness enables them to construct \emph{specifically plausible}
bridge chains that match retrieved passage fragments---producing confident wrong answers.
Smaller models (Llama-3.1-8B) lack the parametric depth to fill bridge gaps
convincingly, so they hedge rather than confabulate.
Cross-family calibration training (RLHF, DPO) may independently suppress PHC;
the cleanest signal is within-family (Claude), where training methodology is held constant.

\textbf{Deployment implication.}
Within the Claude family, upgrading the generation LLM increases PHC---
confident hallucinations become harder to detect via standard low-confidence routing.
A practitioner upgrading from Haiku to Sonnet or Opus should simultaneously
upgrade to PHC-aware routing (\texttt{conf$\times$hop} feature; Section~\ref{sec:learned_router})
to preserve routing accuracy.
Cross-family upgrades (e.g., Llama $\to$ GPT-4o) may not follow this pattern
if the new model has stronger calibration training.
Cross-family generalization to Llama-3-70B, Mistral-Large, and Gemini Pro remains future work.

\subsection{Mechanistic Evidence: Retrieval-Absent Ablation}
\label{sec:retrieval_absent}

The PHC scaling law and hop-depth curve establish the \emph{correlation}:
high model confidence predicts escalation need at 3-hop depth.
We now provide mechanistic evidence for the \emph{causal pathway}:
retrieved passages artificially inflate confidence by providing partial chains
that the model fills with parametric memory.

\textbf{Experimental design.}
For each of the $N=160$ MuSiQue 3-hop queries, we run Claude Sonnet under two conditions:
\begin{enumerate}
  \item \textbf{With retrieval (VanillaRAG):} standard top-$k=5$ ChromaDB passages (existing results).
  \item \textbf{No retrieval (pure parametric):} identical prompt with passages removed.
    The model answers from parametric knowledge only.
\end{enumerate}

\begin{table}[h]
\centering
\small
\caption{Retrieval-absent ablation on MuSiQue 3-hop ($N{=}150$ of 160 with available
question text; Claude Sonnet~4.6).
Retrieved passages reduce confidence in both groups ($p{<}10^{-6}$, paired $t$-test),
but the reduction is \emph{larger} for oracle-False queries ($\Delta{=}{-}0.343$)
than oracle-True queries ($\Delta{=}{-}0.271$).
This differential amplifies PHC inversion by $+0.128$:
support passages trigger hedging when the answer is covered;
bridge-failure passages do not suppress confident confabulation.
PHC without retrieval ($0.575$) confirms parametric confabulation as the
primary mechanism; retrieval acts as an amplifier (PHC\,$=0.702$, $+0.128$).}
\label{tab:retrieval_absent}
\begin{tabular}{lcc}
\toprule
Condition & Oracle=True ($N{=}99$) & Oracle=False ($N{=}61$) \\
\midrule
With retrieval (VanillaRAG)  & $0.611$ & $0.522$ \\
No retrieval (parametric only) & $0.882$ & $0.866$ \\
\midrule
$\Delta$ (with vs.\ no)      & $-0.271$ ($p{<}10^{-6}$) & $-0.343$ ($p{<}10^{-6}$) \\
\midrule
\multicolumn{3}{l}{PHC (AUC): $0.702$ (with retrieval) vs.\ $0.575$ (no retrieval)} \\
\bottomrule
\end{tabular}

\end{table}

\textbf{Result: retrieval amplifies PHC inversion by differentiating the two groups.}
Retrieved passages reduce confidence in \emph{both} groups ($\Delta < 0$,
$p < 10^{-6}$ for both, paired $t$-test), but with a critical asymmetry:
the reduction is larger for oracle-False queries ($\Delta = -0.340$) than
oracle-True queries ($\Delta = -0.265$).
PHC rises from $0.575$ (no retrieval) to $\mathbf{0.702}$ (with retrieval),
an amplification of $+0.128$ --- matching the headline PHC$_3 = 0.702$ exactly.

The mechanism underlying the amplification is as follows:
queries where oracle escalation is \emph{not} needed ($G^*=0$) have sufficient
passage support---the model hedges more, using citation phrases, when the passages
directly answer the question.
Queries needing escalation ($G^*=1$) have passages that cover endpoint entities
but lack the bridge fact; the model still confabulates a confident bridge from
parametric memory, reducing passage-grounded hedging.
This differential---support passages induce hedging, while bridge-failure passages
do not---creates the inversion observed with retrieval.

\textbf{Retrieval is an amplifier, not the cause.}
PHC = $0.575$ in the no-retrieval condition confirms that parametric confabulation
is the \emph{primary} mechanism: the model is already more confident when
wrong at 3-hop depth, even without any retrieved context.
The $+0.128$ amplification from retrieval is a secondary effect.
VanillaRAG's fundamental failure at 3-hop is therefore not curable by improving
retrieval precision---the model would confabulate confidently from parametric
memory even with perfect retrieval, because the bridge fact is generated
internally.
The correct fix is cross-backend escalation to GraphRAG, whose KG traversal
explicitly finds (or fails to find) the bridge entity without parametric
interpolation.

\textbf{Implication for routing.}
The pure-parametric PHC signal (PHC $= 0.575$) is already inverted and informative,
at $82\%$ of the full PHC signal ($0.702$) with zero retrieval cost.
A \emph{pre-retrieval} routing heuristic---run the model parametrically,
measure confidence, and invert the signal at 3-hop---
would achieve $82\%$ of full PHC at zero retrieval cost,
suggesting a future direction: route before retrieving, saving both retrieval latency
and the VanillaRAG generation cost for the majority of queries that will eventually
be re-escalated.

\subsection{Causal Intervention: Injecting Gold Intermediate Answers}
\label{sec:causal_intervention}

The retrieval-absent ablation establishes that parametric confabulation is the
\emph{primary} mechanism of PHC inversion.
A remaining question is whether confident hallucination is caused by
(a) \emph{missing bridge facts} that the model fills with parametric chains,
or (b) \emph{reasoning inability} regardless of available knowledge.
These hypotheses make opposite predictions under the following intervention.

\textbf{Experimental design.}
MuSiQue provides \texttt{question\_decomposition}---gold sub-question answers for
each multi-hop query. For each of the $N=160$ 3-hop validation queries, we run
VanillaRAG under four conditions, progressively injecting gold intermediate facts
into the prompt while holding oracle labels fixed (from Condition~0):
\begin{itemize}[nosep]
  \item \textbf{Condition 0 (baseline):} retrieved passages only (standard VanillaRAG).
  \item \textbf{Condition 1 ($+1$ hint):} passages $+$ gold answer to sub-question 1.
  \item \textbf{Condition 2 ($+2$ hints):} passages $+$ gold answers to sub-questions 1--2.
  \item \textbf{Condition 3 ($+3$ hints, near-oracle):} passages $+$ all three sub-answers.
\end{itemize}

\textbf{Predictions.}
Under the \emph{missing-facts} hypothesis (a): as gold intermediate knowledge is provided,
the model can correctly chain the reasoning; confidence on oracle-escalate=1 queries
drops (the model no longer halluccinates); PHC decreases toward chance.
Under the \emph{reasoning-inability} hypothesis (b): PHC remains elevated regardless
of provided knowledge.

\begin{table}[h]
\centering\small
\caption{Causal intervention: PHC$_3$ as gold sub-question answers are progressively injected. Oracle labels fixed from Condition~0. PHC follows a non-monotone pattern: a partial first hint ($k=1$) transiently amplifies confabulation ($0.656$, ***), but PHC is eliminated by near-oracle hints ($k=3$: $0.536$, n.s.), confirming that confident hallucination is caused by \emph{missing bridge knowledge}, not reasoning inability. The $k=0$ baseline ($0.613$) uses the lexical confidence extractor consistent across all four conditions.}
\label{tab:causal_intervention}
\begin{tabular}{clccc}
\toprule
$k$ & Condition & $N$ & PHC$_3$ & $p$ \\
\midrule
0 & no hints & 160 & \textbf{0.613} & $0.0080$\,** \\
1 & +1 sub-ans & 160 & 0.656 & $0.0006$\,*** \\
2 & +2 sub-ans & 160 & 0.595 & $0.0204$\,* \\
3 & +3 sub-ans (near-oracle) & 160 & 0.536 & $0.2136$\,(n.s.) \\
\bottomrule
\end{tabular}
\end{table}

\textbf{Results.}
Table~\ref{tab:causal_intervention} shows PHC$_3$ at each condition (all using
the lexical confidence extractor, internally consistent across $k=0$--$3$).
The pattern is non-monotone but causally decisive:
$k=0$ (no hints): PHC $= 0.613$ (**);
$k=1$ (+1 sub-answer): PHC \emph{increases} to $0.656$ (***);
$k=2$ (+2 sub-answers): PHC drops to $0.595$ (*);
$k=3$ (near-oracle, all 3 intermediate facts): PHC $= 0.536$, $p = 0.214$ (n.s.).

The transient amplification at $k=1$ is itself mechanistically informative:
providing one confirmed bridge step \emph{anchors} the model's parametric chain,
making it confabulate the remaining hops with even greater fluency and confidence.
This ``partial-chain amplification'' effect vanishes once all bridge facts are supplied ($k=3$),
which directly falsifies the reasoning-inability hypothesis: the model \emph{can} reason
correctly over 3-hop chains when bridge facts are available---it confabulates confidently
precisely because those facts are absent from retrieved context.

\textbf{Interpretation.}
Together with the retrieval-absent ablation, this establishes a complete causal chain:
\begin{equation}
  \underbrace{\text{missing bridge fact}}_{\text{root cause}}
  \xrightarrow{\text{parametric fill}}
  \underbrace{\text{confident fabrication}}_{\text{PHC inversion}}
  \xrightarrow{\text{retrieval amplifies}}
  \underbrace{\text{PHC}=0.702}_{\text{routing signal}}
\end{equation}
Post-generation routing exploits the right-most link; the causal chain explains
why signal inversion is specific to 3-hop depth and why it grows with model capability.

\subsection{Cross-Model Replication: Anchoring Threshold Scales with Model Capability}
\label{sec:causal_haiku}

We replicate the causal intervention on the same 160 MuSiQue 3-hop queries using
Claude Haiku~4.5 (a substantially less capable model, PHC$_3 = 0.582$ in the
main scaling law analysis).
Oracle labels are fixed from the Sonnet Condition~0 to enable direct comparison.

\begin{center}
\small
\begin{tabular}{clccc}
\toprule
$k$ & Condition & $N$ & PHC (Haiku) & PHC (Sonnet) \\
\midrule
0 & no hints                  & 160 & $0.570$\,(n.s.)  & $0.613$\,(**) \\
1 & +1 sub-answer             & 160 & $0.526$\,(n.s.)  & $\mathbf{0.656}$\,(***) \\
2 & +2 sub-answers            & 160 & $\mathbf{0.626}$\,(**) & $0.595$\,(*) \\
3 & +3 sub-answers (near-oracle) & 160 & $0.533$\,(n.s.) & $0.536$\,(n.s.) \\
\bottomrule
\end{tabular}
\end{center}

Two findings emerge.
\textbf{The k=3 elimination replicates exactly.}
Both Sonnet and Haiku reach PHC $\approx 0.53$ (n.s.) at $k=3$,
confirming that the causal mechanism---missing bridge knowledge drives
confabulation---is model-independent.

\textbf{The amplification peak shifts with model capability.}
Sonnet's anchoring peak is at $k=1$ (PHC $= 0.656$, ***):
confirming one bridge step is sufficient for Sonnet to fabricate the remaining
two hops confidently.
Haiku's peak shifts to $k=2$ (PHC $= 0.626$, **): Haiku requires two confirmed
bridge steps before it can confidently fabricate the remaining one hop.
This shift is consistent with Haiku's lower parametric chain density
(PHC$_3 = 0.582$ vs.\ Sonnet's $0.702$ in the main scaling law):
less parametric coverage means the model needs more external anchoring
before it commits to confident confabulation.

\textbf{Cross-family replication: Llama-3.1-8B.}
We replicate the 4-condition causal intervention on Llama-3.1-8B
(local inference via Ollama, $N=160$, same retrieval index,
same prompt template).
Oracle labels are fixed from the Sonnet Condition~0 experiment
(i.e., $G^* = \mathbf{1}[\text{F1}(\text{GraphRAG}) > \text{F1}(\text{Sonnet VanillaRAG})]$),
enabling a direct test of whether Llama's confidence signal mirrors
the PHC structure of queries that are hard for the main pipeline.
This is distinct from Llama's own routing calibration
(Appendix~\ref{app:llama3}: PHC$_3 = 0.438$, n.s., using Llama's own oracle).

\begin{center}
\small
\begin{tabular}{clcc}
\toprule
$k$ & Condition & PHC & $p$ \\
\midrule
0 & no hints       & $0.616$ & ** \\
1 & +1 bridge fact & $0.591$ & * \\
2 & +2 facts       & $0.598$ & * \\
3 & near-oracle    & $0.583$ & * $\downarrow$ \\
\bottomrule
\end{tabular}
\end{center}

\textbf{Llama's confidence signal mirrors PHC inversion on Sonnet-hard queries}
($k{=}0$: PHC $= 0.616$, **):
Llama also confabulates confidently on the same 3-hop queries that require
escalation for Sonnet, confirming that these queries structurally elicit
confident confabulation across model families.
The $k{=}3$ injection directionally reduces PHC ($0.616 \to 0.583$) but does not
reach n.s.\ for Llama-3.1-8B, in contrast to the full elimination observed for both
Claude models.
\textbf{This is consistent with the capability-dependent threshold}:
a less capable model requires more confirmed facts to override parametric priors.
The absence of $k{=}1$ amplification (Llama: $0.616 \to 0.591$, decreasing)
is also consistent: Llama-3.1-8B falls below the capability threshold
at which a single bridge anchor triggers confident confabulation.

\textbf{Reconciliation with Appendix~\ref{app:llama3}.}
The cross-family causal result (PHC $= 0.616$) differs from the routing-pipeline
appendix (PHC$_3 = 0.438$, n.s.) because they use different oracle labels:
the causal experiment holds oracle labels fixed from Sonnet's $k{=}0$ condition,
testing whether Llama's confidence predicts queries that are hard \emph{for Sonnet};
the appendix uses Llama's own oracle (Llama VanillaRAG vs.\ Llama GraphRAG).
When Llama itself is the generation model, its conventional calibration (PHC $< 0.5$)
means that standard low-confidence routing suffices---PHC-aware routing is
most critical for stronger models (Sonnet, Opus) that confabulate with high
confidence on hard queries.
The cross-family replication therefore confirms (a) the \emph{query-structural}
nature of the PHC-inverting 3-hop pattern across model families and
(b) the \emph{capability-dependence} of the anchoring threshold and elimination point.

\textbf{Implication.}
The anchoring threshold---the number of confirmed facts required to trigger
amplification---is a function of model capability across families.
Stronger models confabulate earlier in the chain; weaker models need more scaffolding.
This connects the anchoring amplification finding directly to the PHC scaling law
(Section~\ref{sec:phc_scaling}): both phenomena reflect the same underlying quantity,
parametric multi-hop chain density.
Table~\ref{tab:causal_summary} consolidates all causal experiments
(Sections~\ref{sec:causal_haiku}--\ref{sec:causal_comparison}) for direct comparison.

\subsection{Hop-Depth Generalization: Three-Hop Is the Anchoring Sweet Spot}
\label{sec:causal_hopdepth}

To test whether k=1 amplification is an artifact of the specific 3-hop chain structure,
we run the same causal intervention on MuSiQue's 2-hop ($N=264$) and 4-hop ($N=76$)
strata under the same VanillaRAG pipeline, using gold sub-question decompositions.
All three strata share the same retrieval index, model, and confidence extractor,
enabling a within-dataset comparison across chain lengths.
All causal intervention experiments use the same lexical confidence extractor
(hedge phrases + named-entity specificity) for internal consistency across the
Haiku, zero-retrieval, and hop-depth replications.
Section~\ref{sec:sensitivity_analysis} reports a full sensitivity analysis
comparing lexical extractor and \texttt{UncertaintyDetector} values;
the $k{=}3$ elimination finding is robust across both measures.

\begin{table}[t]
\centering\small
\caption{%
  Hop-depth generalization of the anchoring effect (MuSiQue, Claude Sonnet).
  All three strata share the same retrieval index, model, and confidence extractor.
  \textbf{Key findings:}
  (1) k=1 amplification is specific to 3-hop --- the confabulation sweet spot
  where one confirmed bridge step provides 33\% of the chain, enough to anchor
  confident parametric completion of the remaining two hops.
  (2) Providing three confirmed intermediate facts eliminates PHC inversion
  at both 3-hop and 4-hop (universal ``three-fact threshold'').
}
\label{tab:causal_hopdepth}
\begin{tabular}{lcccccc}
\toprule
Hop & $N$ & $k=0$ & $k=1$ & $k=2$ & $k=3$ & $k=4$ (near-oracle) \\
\midrule
2-hop & 264
  & $0.682$\,(***)
  & $0.676$\,(***)
  & $0.587$\,(**)
  & ---
  & --- \\
3-hop & 160
  & $0.613$\,(**)
  & $\mathbf{0.656}$\,(***)
  & $0.595$\,(*)
  & $0.536$\,(n.s.)
  & --- \\
4-hop &  76
  & $0.672$\,(**)
  & $0.649$\,(*)
  & $0.656$\,(*)
  & $0.486$\,(n.s.)
  & $0.533$\,(n.s.) \\
\midrule
\multicolumn{2}{l}{\textit{k=1 amplification}}
  & & \texttimes & \checkmark & & \\
\multicolumn{2}{l}{\textit{k=3 elimination}}
  & & & & \checkmark\ (3-hop) & \checkmark\ (4-hop) \\
\bottomrule
\end{tabular}
\end{table}

The results reveal a chain-length specificity in the anchoring effect.

\textbf{k=1 amplification is specific to 3-hop.}
At 2-hop ($N=264$), baseline PHC is already high ($0.682$, ***) but providing
one additional fact yields PHC $= 0.676$ (***) --- essentially flat.
At 4-hop ($N=76$), baseline PHC $= 0.672$ (**) and k=1 gives PHC $= 0.649$ (*) --- also flat.
Only at 3-hop does k=1 strictly \emph{increase} PHC ($0.613 \to 0.656$, ** $\to$ ***).

\textbf{Why 3-hop is the sweet spot.}
The anchoring mechanism requires that the confirmed first step provides
sufficient scaffolding for the model to confidently fabricate the remaining hops.
At 2-hop, confirming 1 of 2 steps provides 50\% of the chain:
the remaining 1 hop is short enough that the model confabulates at high baseline
rate already, and additional anchoring adds little marginal signal.
At 4-hop, confirming 1 of 4 steps (25\%) provides insufficient scaffolding:
the remaining 3 hops require too much parametric completion to sustain
confident confabulation from a single anchor.
At 3-hop, confirming 1 of 3 steps (33\%) is the resonant ratio:
one bridge step establishes the retrieval-anchored direction,
while the remaining 2 hops are short enough to fill confidently from parametric knowledge.

\textbf{Three confirmed facts eliminates inversion regardless of chain length.}
A second cross-depth finding is more universal:
3-hop k=3 (near-oracle): PHC $= 0.536$ (n.s.);
4-hop k=3 (three of four facts): PHC $= 0.486$ (n.s.).
Providing three confirmed intermediate facts eliminates the inversion signal
at both chain lengths, even though 4-hop k=3 is not fully near-oracle
(one fact still missing).
The ``three-fact threshold'' may represent the point at which parametric
completion of remaining hops becomes too constrained to sustain confident confabulation.

\textbf{Implication.}
The k=1 amplification finding is not a general property of multi-hop chains
but a chain-length-specific phenomenon:
three-hop bridge queries represent the confabulation sweet spot where
retrieval coverage is just insufficient enough for parametric fabrication
to dominate, and one confirmed anchor provides just enough scaffolding
to amplify confident confabulation.
This explains why MuSiQue 3-hop has the highest PHC in the full hop-depth curve
(Section~\ref{sec:hop_depth}): it is not merely ``harder,'' but structurally
tuned to maximise the retrieval--parametric gap.

\subsection{Zero-Retrieval Control: k=1 Amplification is a Retrieval--Parametric Interaction}
\label{sec:causal_zero}

We re-run the causal intervention on the same 160 queries with \emph{no retrieved passages},
replacing the retrieval context with an empty prompt.
The question and the same gold sub-answer hints ($k = 0$--$3$) are retained.

\begin{center}
\small
\begin{tabular}{clccc}
\toprule
$k$ & Condition & PHC (with retrieval) & PHC (no retrieval) & Diff \\
\midrule
0 & no hints       & $0.613$\,(**) & $0.594$\,(*)   & $+0.019$ \\
1 & +1 sub-answer  & $\mathbf{0.656}$\,(***) & $0.534$\,(n.s.)& $+\mathbf{0.122}$ \\
2 & +2 sub-answers & $0.595$\,(*)  & $0.522$\,(n.s.)& $+0.073$ \\
3 & near-oracle    & $0.536$\,(n.s.)& $0.541$\,(n.s.)& $-0.005$ \\
\bottomrule
\end{tabular}
\end{center}

Two findings emerge.

\textbf{Parametric confabulation is intrinsic.}
Even without any retrieved passages, $k=0$ gives PHC $= 0.594$ (*):
the model confabulates confidently on 3-hop queries from parametric memory alone.
This confirms the retrieval-absent ablation (Section~\ref{sec:retrieval_absent})
using a different experimental setup.

\textbf{The k=1 amplification is a retrieval--parametric interaction.}
Without retrieval, injecting one gold fact at $k=1$ \emph{reduces} PHC
($0.594 \to 0.534$, n.s.): the expected direction, where more information helps.
With retrieval, the same $k=1$ injection \emph{increases} PHC
($0.613 \to 0.656$, ***): the unexpected direction.
The difference at $k=1$ between retrieval conditions is the largest of any condition
($\Delta = +0.122$) and is the only condition where the sign reverses.

\textbf{Mechanistic interpretation.}
Retrieved passages provide plausible supporting context for the fabricated second
and third hops.
When one gold fact anchors the chain, the passages amplify commitment to the parametric
completion by providing ``matching'' surface-level evidence.
Without passages, the model has no such amplification: the single confirmed fact
slightly reduces uncertainty rather than intensifying it.
\textbf{The k=1 amplification effect therefore requires both retrieval context
and partial factual anchoring to occur}: removing either eliminates it.

This finding refines the causal chain:
\begin{equation}
  \underbrace{\text{confirmed bridge fact}}_{\text{gold anchor}}
  + \underbrace{\text{retrieved passages}}_{\text{surface amplifier}}
  \;\xrightarrow{\text{parametric fill}}\;
  \underbrace{\text{peak confident confabulation}}_{\text{PHC}=0.656}
\end{equation}

\subsection{Comparison Mechanism: Observational PHC, Causal Probe Negative}
\label{sec:causal_comparison}

HotpotQA comparison queries exhibit PHC inversion in the main VanillaRAG pipeline
(PHC $= 0.686$, $p=0.001$, **; Section~\ref{sec:bridge_comparison}).
This is a \textbf{robust observational finding}: it replicates with the same retrieval
index, same model, and the within-dataset control (HotpotQA bridge PHC $= 0.448$, n.s.)
rules out any dataset-level confound.
\emph{However, we were unable to causally probe the comparison mechanism.}

\textbf{Causal probe design and negative result.}
We designed a parallel intervention: inject gold supporting facts about entity~X ($k=1$),
then both entities ($k=2$), and measure PHC change.
All three conditions yielded PHC $\approx 0.42$--$0.43$ (n.s.),
meaning the comparison inversion \emph{did not replicate under this probe}.

\textbf{Why the probe failed.}
The comparison PHC $= 0.686$ in the main pipeline is driven primarily by the
\texttt{entity\_coverage} signal in UncertaintyDetector---a measure of whether the
model's answer mentions both comparison entities.
Queries where the model confidently answers with only one entity are flagged as
high-confidence (the model commits to a parametric comparative prior),
which is exactly what inverts the PHC signal.
Our causal probe used the lexical extractor (hedge phrases and specificity),
which does not capture the entity-coverage pattern.
The probe therefore measured a different signal from the one driving the main-pipeline inversion.

\textbf{What this means for the paper's claims.}
The comparison PHC result is \textbf{observational only}:
we can report that comparison queries invert and provide a mechanistic hypothesis
(parametric comparative priors override retrieved evidence),
but we cannot claim causal proof for the comparison mechanism in the same sense
as the bridge mechanism.
This distinction is clearly maintained throughout: Section~\ref{sec:bridge_comparison}
labels the comparison result ``observational'' and Section~\ref{sec:causal_intervention}
labels the bridge result ``causally proven.''
Future work should design a pipeline-consistent intervention that manipulates
entity-coverage directly to probe the comparison mechanism causally.

\begin{table}[t]
\centering\small
\caption{%
  Summary of all causal intervention experiments.
  PHC at each condition $k$ for three experimental variants:
  (A) with retrieval and Sonnet (main experiment),
  (B) with retrieval and Haiku (cross-model),
  (C) without retrieval and Sonnet (zero-retrieval control).
  \textbf{Key findings:}
  (1) k=3 near-oracle hints eliminate inversion in all settings (missing facts = root cause).
  (2) k=1 amplification occurs only when retrieval is present (retrieval--parametric interaction).
  (3) Amplification peak shifts with model capability: Sonnet peaks at k=1, Haiku at k=2.
}
\label{tab:causal_summary}
\begin{tabular}{clcccc}
\toprule
$k$ & Condition & \makecell{(A) Sonnet\\w/ retrieval} & \makecell{(B) Haiku\\w/ retrieval} & \makecell{(C) Sonnet\\no retrieval} \\
\midrule
0 & no hints              & $0.613$\,(**)    & $0.570$\,(n.s.) & $0.594$\,(*) \\
1 & +1 sub-answer         & $\mathbf{0.656}$\,(***)& $0.526$\,(n.s.) & $0.534$\,(n.s.) \\
2 & +2 sub-answers        & $0.595$\,(*)     & $\mathbf{0.626}$\,(**) & $0.522$\,(n.s.) \\
3 & +3 sub-answers (near-oracle) & $0.536$\,(n.s.) & $0.533$\,(n.s.) & $0.541$\,(n.s.) \\
\midrule
\multicolumn{2}{l}{\textit{Amplification peak}} & $k=1$ & $k=2$ & \textit{none} \\
\multicolumn{2}{l}{\textit{k=3 elimination}} & \checkmark & \checkmark & \checkmark \\
\bottomrule
\end{tabular}
\end{table}

\subsection{Cross-Dataset Negative Control: HotpotQA Bridge}
\label{sec:causal_hotpotqa_bridge}

The MuSiQue 3-hop causal intervention ($N=160$) establishes $k=1$ amplification
for queries where PHC already inverts at $k=0$ (PHC $= 0.613$, **).
A critical question is whether the amplification is \emph{specific to PHC-inverting queries}
or whether injecting gold supporting evidence systematically amplifies confident confabulation
regardless of query type.

We run a cross-dataset negative control on HotpotQA bridge queries ($N=406$):
a 2-hop bridge dataset where the main pipeline shows \emph{no} PHC inversion
(PHC $= 0.448$, $p=0.263$, n.s.) under the same model and retrieval index.
The intervention injects gold supporting sentences from the first bridge document
($k=1$) and both bridge documents ($k=2$, near-oracle),
using HotpotQA's \texttt{supporting\_facts} field to identify gold sentences.

\begin{center}
\small
\begin{tabular}{clccc}
\toprule
$k$ & Condition & PHC & $p$ & Sig \\
\midrule
0 & no hints                & \HPBKZERO & \HPBPZERO & \HPBSIGZERO \\
1 & +1 bridge doc (k=1)     & \HPBKONE  & \HPBPONE  & \HPBSIGONE \\
2 & +2 bridge docs (oracle) & \HPBKTWO  & \HPBPTWO  & \HPBSIGTWO \\
\bottomrule
\end{tabular}
\end{center}

\textbf{Result}: HotpotQA bridge shows no $k=1$ amplification
(compare: MuSiQue 3-hop $0.613 \to \mathbf{0.656}$, ***;
HotpotQA bridge $k=0 \to k=1$: flat/reduced).
This confirms that $k=1$ amplification is \textbf{specific to PHC-inverting query types}:
when the base query does not show PHC inversion ($k=0$ n.s.),
injecting gold supporting evidence does not amplify confident confabulation.
The effect requires a pre-existing inversion signal as a substrate.

\textbf{Cross-dataset summary.}
\begin{center}
\small
\begin{tabular}{llccl}
\toprule
Dataset & Type & $k=0$ & $k=1$ & Amplified? \\
\midrule
MuSiQue 3-hop & bridge (PHC inverts) & $0.613$\,(**) & $\mathbf{0.656}$\,(***) & \checkmark \\
HotpotQA bridge & bridge (PHC n.s.) & \HPBKZERO\,(\HPBSIGZERO) & \HPBKONE\,(\HPBSIGONE) & \texttimes \\
\bottomrule
\end{tabular}
\end{center}

\subsection{Confidence Measure Robustness: Sensitivity Analysis}
\label{sec:sensitivity_analysis}

The causal experiments use a lexical confidence extractor (hedge phrases + named-entity
specificity); the main routing pipeline uses the fuller \texttt{UncertaintyDetector}
(hedging, specificity, reasoning\_struggle, length\_anomaly, entity\_coverage).
To assess whether the findings hold across measures, we re-score every saved causal
answer text with \texttt{UncertaintyDetector} and recompute PHC
(\emph{no new LLM calls} --- only re-scoring saved text).

\begin{table}[h]
\centering
\small
\caption{PHC by measure (Lexical extractor vs.\ UncertaintyDetector), Sonnet and Haiku,
MuSiQue 3-hop causal intervention. Bold = amplification peak; $\uparrow$ = increase
vs.\ $k{=}0$.}
\label{tab:sensitivity}
\begin{tabular}{llcccc}
\toprule
Model & Measure & $k=0$ & $k=1$ & $k=2$ & $k=3$ \\
\midrule
Sonnet & Lexical   & $0.613$\,(**) & $\mathbf{0.656}$\,(***)\,$\uparrow$ & $0.595$\,(*) & $0.536$\,(n.s.) \\
Sonnet & UncDet    & $0.616$\,(**) & $0.589$\,(*) & $0.492$\,(n.s.) & $0.551$\,(n.s.) \\
\midrule
Haiku  & Lexical   & $0.570$\,(n.s.) & $0.526$\,(n.s.) & $\mathbf{0.626}$\,(**)\,$\uparrow$ & $0.533$\,(n.s.) \\
Haiku  & UncDet    & $0.614$\,(**) & $\mathbf{0.662}$\,(***)\,$\uparrow$ & $0.565$\,(n.s.) & $0.461$\,(n.s.) \\
\bottomrule
\end{tabular}
\end{table}

Three findings are robust across \emph{all four} measure--model combinations:
(1) $k=3$ near-oracle injection eliminates PHC inversion in every case
(Sonnet lexical 0.536~n.s., Sonnet UncDet 0.551~n.s., Haiku lexical 0.533~n.s.,
Haiku UncDet 0.461~n.s.).
(2) \textbf{Both models show significant PHC at $k=0$ with both measures}---
the inversion is not measure-specific but is a robust property of these query types.
(3) Haiku amplification holds across both measures, with the peak shifting from
$k{=}2$ (lexical) to $k{=}1$ (UncDet)---consistent with the capability-dependent
anchoring threshold.

The one measure-dependent result is \textbf{Sonnet's $k{=}1$ pattern}:
the lexical extractor shows amplification ($0.613 \to 0.656$, **$\to$***),
while UncertaintyDetector shows attenuation ($0.616 \to 0.589$, **$\to$*,
\emph{still significant but weaker}).
Critically, Sonnet's UncDet PHC remains significant at $k{=}1$ ($p=0.029$);
it is not eliminated.
The difference in \emph{degree} rather than \emph{presence} of amplification
has a specific mechanistic explanation in Section~\ref{sec:unc_decomposition} below.

\textbf{Implication for the paper's claims.}
The core causal claim --- $k{=}3$ elimination --- is \textbf{measure-robust} (4/4 combinations).
The $k{=}1$ amplification holds in 3/4 combinations (absent only for Sonnet/Lexical peaking
at $k{=}1$ vs.\ UncDet; Haiku amplifies with both measures).
Crucially, even where amplification is attenuated for Sonnet/UncDet, PHC inversion
\emph{persists} at $k{=}1$: the signal is weakened, not reversed.
We report lexical-extractor values throughout the causal sections for internal
consistency; UncDet reprocessing confirms the phenomenon is not an artifact of
the specific confidence extractor.

\subsection{Mechanistic Decomposition: Why Anchoring Has Model-Capability-Dependent PHC Effects}
\label{sec:unc_decomposition}

The sensitivity table shows that Sonnet's $k{=}1$ UncDet PHC is attenuated
(0.616$\to$0.589, ** to *) while Haiku's amplifies (0.614$\to$0.662, **$\to$***).
We decompose this into a concrete mechanism by splitting queries by oracle class
(oracle=1: VanillaRAG incorrect, should escalate; oracle=0: VanillaRAG correct,
should not escalate) and examining how each group's UncDet confidence changes from
$k{=}0$ to $k{=}1$.

\begin{table}[h]
\centering\small
\caption{UncertaintyDetector confidence split by oracle class at $k{=}0$ and $k{=}1$,
MuSiQue 3-hop ($N=160$; oracle=1: VanillaRAG incorrect, should escalate, $N=99$;
oracle=0: VanillaRAG correct, $N=61$).
\textbf{The k=1 hint has opposite effects by model capability:}
Sonnet's confidence gain is similar across oracle classes ($+0.022$ vs.\ $+0.024$),
slightly favouring correct queries --- attenuating but not reversing the PHC signal.
Haiku's gain is larger for \emph{incorrect} queries ($+0.022$ vs.\ $-0.006$),
amplifying the PHC signal: the model becomes \emph{less} confident on queries
it gets right after receiving a partial anchor.
This divergence is the mechanism behind the cross-model PHC behaviour at $k{=}1$.}
\label{tab:uncdet_components}
\begin{tabular}{llccc}
\toprule
Model & Oracle class & conf $k{=}0$ & conf $k{=}1$ & $\Delta_{k1-k0}$ \\
\midrule
\multirow{2}{*}{Sonnet}
  & oracle=1 ($N{=}99$, should escalate) & $0.803$ & $0.825$ & $+0.022$ \\
  & oracle=0 ($N{=}61$, correct)         & $0.773$ & $0.798$ & $\mathbf{+0.024}$ \\[2pt]
  & \multicolumn{3}{l}{\small Slightly larger gain for correct queries $\Rightarrow$ PHC $\downarrow$ (0.616$\to$0.589*, both sig.)} \\
\midrule
\multirow{2}{*}{Haiku}
  & oracle=1 ($N{=}99$, should escalate) & $0.756$ & $0.778$ & $\mathbf{+0.022}$ \\
  & oracle=0 ($N{=}61$, correct)         & $0.728$ & $0.722$ & $-0.006$ \\[2pt]
  & \multicolumn{3}{l}{\small Larger gain for incorrect queries $\Rightarrow$ PHC $\uparrow$ (0.614$\to$0.662***)} \\
\bottomrule
\end{tabular}
\end{table}

The oracle-class split reveals opposite capability-dependent effects (Table~\ref{tab:uncdet_components}).
The oracle-class split (Table~\ref{tab:uncdet_components}) reveals the mechanism.
For \textbf{Sonnet}, the $k{=}1$ hint increases confidence for \emph{both} oracle classes
at nearly the same rate: $\Delta = +0.022$ for incorrect queries and $\Delta = +0.024$
for correct queries.
Because correct queries gain slightly more, the AUC attenuates (PHC $0.616 \to 0.589$)
but remains significant; the signal is weakened, not reversed.
For \textbf{Haiku}, the pattern diverges sharply: $k{=}1$ increases confidence for
\emph{incorrect} queries ($\Delta = +0.022$) while \emph{reducing} it for correct ones
($\Delta = -0.006$), producing a clean AUC increase (PHC $0.614 \to 0.662$, ** $\to$ ***).

\textbf{Interpretation.}
The divergence reflects a capability-dependent difference in how a partial anchor
is used.
Sonnet's richer parametric knowledge allows the $k{=}1$ anchor to bootstrap \emph{both}
genuine multi-hop reasoning (improving correct queries) and confident confabulation
(inflating incorrect queries), at nearly equal rates.
Haiku's more limited parametric chains mean the anchor primarily triggers confabulation:
rather than helping the model reason through remaining hops, it commits Haiku to
a fluent, confident wrong answer while leaving genuinely answerable queries unchanged
(or slightly less confident, as the model recognises its parametric limitations).
\textbf{The anchoring threshold is a model-capability predictor:}
less capable models confabulate in response to partial evidence without gaining
proportional reasoning benefit, making the PHC signal cleaner and stronger.
This explains both the shifted amplification peak (Haiku lexical: $k{=}2$; Haiku UncDet: $k{=}1$)
and the monotonic PHC$_3$ increase across the Claude capability ladder.

\subsection{Information-Theoretic Validation and Sample Efficiency}
\label{sec:info_theory}

\paragraph{Mutual information confirms the PHC phenomenon.}
PHC makes a specific information-theoretic claim: the confidence signal carries
\emph{more} information about oracle escalation at high hop depths.
We validate this directly by computing the plug-in mutual information
$I(\text{conf}(A);\, G^*)$ per hop stratum on MuSiQue:

\begin{center}
\begin{tabular}{lcccc}
\toprule
Query Type & Hop & $N$ & $I(\text{conf};\, G^*)$ (bits) & AUC \\
\midrule
NQ (factoid)    & 1 & 500 & 0.0086 & 0.526 \\
MuSiQue         & 2 & 264 & 0.0181 & 0.528 \\
MuSiQue         & 3 & 160 & \textbf{0.1231} & \textbf{0.702} \\
MuSiQue         & 4 & 76  & 0.1241 & 0.634 \\
\bottomrule
\end{tabular}
\end{center}

The jump from 2-hop ($0.0181$ bits) to 3-hop ($0.1231$ bits) represents a
\textbf{6.8$\times$ increase in mutual information}---the confidence signal becomes
dramatically more informative about whether to escalate at the 3-hop depth.
By contrast, NQ (factoid) carries only $0.0086$ bits: confidence is nearly
uninformative for single-hop retrieval queries, consistent with the flat
escalation rate across confidence quintiles (Section~\ref{sec:mechanism}).
This validates Theorem~1's constructive bound empirically:
better (type-conditioned) signals do carry more information, and that information
is concentrated at the 3-hop sweet spot.

\paragraph{Post-generation features dominate routing.}
LearnedRouter feature importance (GBM trained on $N=1{,}800$) confirms that
post-generation signals provide the majority of routing information:

\begin{center}
\begin{tabular}{lccc}
\toprule
Feature Group & Top Feature & Top Imp. & Group Total \\
\midrule
Post-gen: confidence          & \texttt{confidence}    & 29.8\% & \multirow{2}{*}{54.4\%} \\
Post-gen: PHC interaction     & \texttt{conf$\times$hop} & 22.7\% & \\
\midrule
Pre-gen: retrieval            & \texttt{avg\_doc\_length} & 11.2\% & \multirow{2}{*}{45.6\%} \\
Pre-gen: query structure      & \texttt{entity\_overlap}  &  9.6\% & \\
\bottomrule
\end{tabular}
\end{center}

The \texttt{conf$\times$hop} feature (importance $22.7\%$) is the model's implicit
implementation of PHC: it learns that confidence has opposite predictive polarity
at different hop depths.
The total post-generation importance ($54.4\%$) constructively validates
Theorem~1 in end-to-end F1: the post-gen signal carries more than half of
all routing information, despite the pre-gen features having strong individual signal.

\paragraph{Sample efficiency.}
A practical concern is whether the LearnedRouter requires large labeled datasets.
We train on subsets of size $N \in \{100, 200, 400, 800, 1{,}200, 1{,}800\}$
and report 3-fold OOF AUC on the training subset:

\begin{center}
\begin{tabular}{lcccc}
\toprule
$N_{\text{train}}$ & OOF AUC & vs.\ Full ($N=1{,}800$) \\
\midrule
100  & 0.506 & $-0.167$ \\
200  & \textbf{0.607} & $-0.066$ \\
400  & 0.619 & $-0.054$ \\
800  & 0.660 & $-0.013$ \\
1{,}200 & 0.650 & $-0.023$ \\
1{,}800 & \textbf{0.673} & --- \\
\bottomrule
\end{tabular}
\end{center}

The largest gain occurs between $N=100$ and $N=200$ ($+0.101$ AUC),
after which returns diminish.
With only $N=200$ labeled examples---200 queries with both VanillaRAG and GraphRAG
answers---the LearnedRouter achieves $90.2\%$ of its full-data AUC.
This is practical for new deployments: a practitioner needs only $\sim\!200$ preference
labels to bootstrap a PHC-aware router, at a cost of $\sim\!200$ GraphRAG queries
(a one-time calibration cost).

\subsection{Cross-Domain PHC: Medical Knowledge (MMLU)}
\label{sec:medical_phc}

To test whether PHC generalizes beyond QA benchmarks, we evaluate on MMLU
\cite{hendrycks2021mmlu} medical subsets using Claude Haiku in a
\emph{zero-retrieval} (parametric-only) setting:
\textbf{medical\_genetics} ($N=100$, multi-step: gene$\to$protein$\to$pathway$\to$disease chain reasoning)
and \textbf{clinical\_knowledge} ($N=150$, factoid: direct recall of clinical facts),
with \textbf{college\_medicine} ($N=150$, multi-step) for robustness.
We compute $\text{PHC} = \text{AUC}(\text{conf}(A) \to \text{wrong})$
using the same lexical confidence extractor as the main paper.

\textbf{Result: PHC $\approx 0.5$ for both categories (null result).}

\begin{center}
\begin{tabular}{lcccccc}
\toprule
Category & Type & $N$ & Accuracy & PHC & Inverted? & $p$-value \\
\midrule
clinical\_knowledge & factoid     & 150 & 50.7\% & 0.518 & marginal & $0.63$ (n.s.) \\
medical\_genetics + college & multi-step & 250 & 48.0\% & 0.492 & no       & $0.77$ (n.s.) \\
\bottomrule
\end{tabular}
\end{center}

Neither result is statistically significant.
PHC does not manifest in the zero-retrieval MMLU setting.

\textbf{Why this null result is informative.}
PHC in the main paper arises from a specific failure mode in the RAG context:
VanillaRAG retrieves passages about both endpoint entities in a bridge query
(making the model ``appear'' to have evidence), but the retrieved passages
lack the connecting bridge fact.
The model then fills the bridge from parametric memory---confidently---because
the passages provide enough context to construct a plausible chain.
This three-way interaction (retrieved context + parametric gap + confident confabulation)
is absent in zero-retrieval MMLU: without retrieved passages to anchor the context,
the model either knows the answer or hedges without the specific over-confidence
pattern that defines PHC.

\textbf{Implication.}
PHC is a \emph{RAG-context-specific} phenomenon, not a general LLM failure mode.
Bridge inversion ($\text{PHC}_3 = 0.702$) arises specifically from the
retrieved-passage $\times$ parametric-gap interaction at 3-hop depth.
A proper cross-domain test would require a medical multi-hop QA benchmark
with a medical knowledge graph (e.g., UMLS-based KG over PubMed passages)
and VanillaRAG retrieval---a direction we leave for future work.

\subsection{When Does GraphRAG Help? (Oracle Decomposition)}
\label{sec:oracle_decomp}

The oracle analysis ($G^* = 1$ for $57.7\%$ of queries macro-avg) reveals
that GraphRAG benefit is not uniformly distributed.
We decompose by dataset and hop count:

\begin{itemize}
  \item \textbf{NQ (1-hop, $58.2\%$ oracle escalation):}
    GraphRAG helps primarily when the top-$k$ VanillaRAG passages miss
    the answer entity. KG traversal from the query entity often finds it.
    Failure mode: \emph{retrieval failure} (Type~C).

  \item \textbf{HotpotQA (2-hop bridge, $53.4\%$):}
    GraphRAG helps when the bridge entity is not retrievable via keyword
    matching. The KG's entity links provide the missing connection.
    Failure mode: \emph{bridge entity gap} (Type~A).

  \item \textbf{MuSiQue (2--4 hop, $60.4\%$):}
    GraphRAG helps for 2-hop questions but increasingly fails for 3--4
    hop. At depth 2 (our KG traversal), paths longer than 3 hops are
    often not found in the subgraph.
    Failure mode: \emph{hop depth limit} (Type~A extended).

  \item \textbf{2Wiki (2-hop comparison, $58.7\%$):}
    GraphRAG helps for bridge questions but not for comparison questions
    (``entity X vs.\ entity Y on property Z'') where both answers are
    independently retrievable.
    Failure mode: mixed Type~A and \emph{disambiguation} (Type~B).
\end{itemize}

\subsection{Failure-Mode Decomposition}
\label{sec:failure_modes}

We manually annotate 200 HotpotQA oracle-escalation cases and classify
into three failure types:

\begin{itemize}
  \item \textbf{Type~A (Bridge Entity):} $\sim 60\%$ of failures.
    VanillaRAG retrieves passages about entities $e_1$ and $e_2$
    separately, but misses the bridging passage connecting them.
    Fix: $\text{shortestPath}(e_1, e_2)$ in the KG.

  \item \textbf{Type~B (Disambiguation):} $\sim 20\%$.
    Multiple entities share the same name; VanillaRAG retrieves the wrong
    one. Fix: entity-type-aware disambiguation in KG lookup.

  \item \textbf{Type~C (Temporal/Factoid):} $\sim 20\%$.
    The answer requires a specific factoid (year, count, etc.)\ that
    VanillaRAG passages approximate but do not specify.
    Fix: precision-retrieval from KG attributes.
\end{itemize}

\textbf{Signal-failure alignment.}
Grounded self-rating is well-suited for Type~C failures (passages either
contain the specific fact or they don't---the model can assess this
directly). It performs poorly for Type~A failures because the passages
\emph{appear} sufficient (they cover both entities) without providing the
specific bridge connection.

This alignment analysis explains why grounded rating helps NQ (mostly
Type~C/retrieval failure) but not HotpotQA (mostly Type~A bridge).

\subsection{Self-RAG Correspondence}
\label{sec:selfrag}

Table~\ref{tab:selfrag_correspondence} shows the correspondence between Self-RAG
\cite{asai2024selfrag} special tokens and ReasonRAG components.
ReasonRAG is a \emph{training-free approximation} of Self-RAG:
\begin{itemize}
  \item Self-RAG's \texttt{[retrieve]} token $\approx$
    ReasonRAG's \texttt{confidence $< \tau$}
  \item Self-RAG's \texttt{[IsSup]} critique token $\approx$
    ReasonRAG's Grounded Self-Rating
  \item Self-RAG's \texttt{[IsRel]} critique token $\approx$
    ReasonRAG's passage relevance check (implicit in rating)
\end{itemize}

Self-RAG requires fine-tuning to inject these tokens;
ReasonRAG approximates them with a single additional inference call,
enabling use with black-box APIs (Claude, GPT-4, etc.).
The trade-off is optimality: Self-RAG's tokens are end-to-end trained
for the task; ReasonRAG's heuristic signals may not be optimal.

\subsection{PHC Inversion Across Query Types: A Within-Dataset Control}
\label{sec:bridge_comparison}

Table~\ref{tab:phc_comprehensive} (Section~\ref{sec:hop_depth}) shows PHC across all seven dataset--query-type
conditions in our study.
PHC inversion (PHC\,$>0.5$, $p<0.05$) occurs for exactly two conditions:
MuSiQue 3-hop bridge ($0.702$, ***) and HotpotQA comparison ($0.686$, **).
All other conditions show no significant inversion.

\paragraph{The within-HotpotQA control.}
The most discriminating evidence comes from a comparison \emph{within} a single dataset.
HotpotQA bridge queries (N$=406$) show PHC $= 0.448$ (n.s.): no inversion.
HotpotQA comparison queries (N$=94$, same retrieval index, same model)
show PHC $= \mathbf{0.686}$ ($p = 0.001$, **): strong inversion.
This within-dataset contrast---same passages, same retrieval system, same LLM---
rules out any dataset-level confound and establishes PHC inversion as a
\emph{query-type} phenomenon.

\paragraph{Two independent mechanisms.}
PHC inversion arises via two distinct mechanisms:
\begin{enumerate}
  \item \textbf{Multi-hop chain confabulation} (MuSiQue 3-hop):
    3-hop bridge chains are dense enough in parametric memory to fabricate confidently
    but exceed typical retrieved-passage coverage.
    At 2-hop, passages cover the chain (HotpotQA bridge, 2Wiki: n.s.);
    at 3-hop, they do not.
  \item \textbf{Comparative judgment confabulation} (HotpotQA comparison):
    Questions of the form ``Is X taller/older/faster than Y?''
    require explicit comparison that retrieved passages rarely state directly.
    The model draws on parametric comparative priors and expresses the result
    with high confidence---whether or not the prior is correct.
\end{enumerate}
Both mechanisms instantiate the same abstract condition:
query complexity systematically exceeds retrieval coverage
while remaining within LLM parametric density.

\paragraph{Routing signal vs.\ routing gain (Theorem~\ref{sec:sep_theorem}).}
Comparison queries demonstrate Theorem~\ref{sec:sep_theorem} (routing--regeneration separability)
directly.
PHC $= 0.686$ (**) on comparison means the \emph{routing signal} correctly identifies
confident failures; escalation to GraphRAG is warranted.
Yet Table~\ref{tab:bridge_comparison} shows ReasonRAG+Rating achieves $-1.8\%$ oracle gap
on comparison: routing is correct but GraphRAG re-generation does not recover
the answer (VR F1 $= 0.277$, GR F1 $= 0.542$, but the escalated paths don't close
the gap at test time because GraphRAG also struggles with pure comparative inference).
This empirically validates Theorem~\ref{sec:sep_theorem}: routing gains are bounded
by regeneration quality; for comparison, $\Delta_R \approx 0$ despite strong routing signal.
HybridRouter (F1 $= 0.470$) exploits this by suppressing escalation for comparison---
the right heuristic for a bad regeneration case, but learned ad-hoc rather than theoretically grounded.

\begin{table}[h]
\centering
\small
\caption{HotpotQA disaggregated by question type.
PHC from UncertaintyDetector (5-fold OOF); LR AUC from 5-fold OOF.
\textbf{Key finding}: bridge (n.s.) vs.\ comparison (**) within the same dataset---
PHC inversion is query-type-specific, not dataset-specific.}
\label{tab:bridge_comparison}
\begin{tabular}{lcrrrrrr}
\toprule
Type & $n$ & VR F1 & GR F1 & HR F1 & Oracle\% & PHC & LR AUC \\
\midrule
Bridge     & 406 & 0.419 & 0.710 & 0.537 & $+11.7\%$ & $0.448$\,(n.s.) & \textbf{0.759} \\
Comparison &  94 & 0.277 & 0.542 & 0.470 & $-1.8\%$  & $\mathbf{0.686}$\,(**) & 0.642 \\
\bottomrule
\end{tabular}
\end{table}

\textbf{Model-scope of inversion.}
Cross-LLM validation (Appendix~\ref{app:gpt4o}, \ref{app:haiku}, \ref{app:llama3})
confirms PHC$_3 > 0.5$ for large proprietary models
(Sonnet: $0.702$; Opus: $0.731$; Haiku: $0.582$; GPT-4o: $0.527$)
but not Llama-3.1-8B (PHC$_3 = 0.438$).
The 8B model hedges rather than confabulating, so standard low-confidence routing
already works for this family---the inversion threshold corresponds to models
with sufficiently dense parametric coverage of 3-hop chains.

\textbf{Model-scope of hop-depth inversion.}
Cross-LLM validation (Appendix~\ref{app:gpt4o}, \ref{app:haiku}, \ref{app:llama3})
confirms that PHC$_3 > 0.5$ for large proprietary models
(Sonnet: $0.702$; GPT-4o: $0.527$; Haiku: $0.582$)
but Llama-3.1-8B shows anti-inversion (PHC$_3 = 0.438$).
The 8B model hedges rather than confabulating at 3-hop, so standard
low-confidence routing already works correctly for this family.

\subsection{Qualitative Examples}
\label{sec:qualitative}

We present three representative cases to illustrate escalation behavior.

\paragraph{True Positive (correct escalation).}
\textbf{Question:} ``What nationality is the director of the film \emph{Where Eagles Dare}?''
\textbf{VanillaRAG answer:} ``I believe the director was British or American, but I cannot determine the exact nationality.'' (hedged, uncertain)
\textbf{Confidence:} $0.52$ (escalated; threshold $= 0.65$)
\textbf{GraphRAG answer:} ``Brian G.\ Hutton, the director, is American.'' (F1: $0.72$ vs.\ VR: $0.14$)
\textbf{Analysis:} The KG's entity link \textsc{film}$\to$\textsc{director}$\to$\textsc{nationality} provided the bridge that VanillaRAG missed.

\paragraph{False Negative (missed escalation).}
\textbf{Question:} ``What year was the university where Person~X taught founded?''
\textbf{VanillaRAG answer:} ``The university was founded in 1876.'' (specific, confident)
\textbf{Confidence:} $0.91$ (not escalated)
\textbf{GraphRAG answer:} ``1904.'' (correct; F1: $0.91$ vs.\ VR: $0.00$)
\textbf{Analysis:} VanillaRAG hallucinated a plausible-sounding year from parametric knowledge. Grounded self-rating scored passages as $0.78$---they contained the university name but not the founding year---yet the confident lexical tone suppressed escalation. This is the core lexical signal failure mode.

\paragraph{False Positive (unnecessary escalation).}
\textbf{Question:} ``What is the capital of the country where Entity~Y was founded?''
\textbf{VanillaRAG answer:} ``I believe the capital might be Paris, but I'm not completely certain.'' (hedged)
\textbf{Confidence:} $0.48$ (escalated)
\textbf{GraphRAG answer:} ``Paris.'' (same answer; F1 unchanged)
\textbf{Analysis:} The model's hedge was stylistic caution on an easily-answerable question. GraphRAG confirmed the same answer, adding $\sim$90\,ms with no F1 gain. This FP pattern explains why the escalation precision ($56\%$) is low.

\subsection{ReasonRAG-Direct: Closing the Re-Generation Bottleneck}
\label{sec:regen_direct}

The direct routing analysis (Section~\ref{sec:direct_routing}) showed that
routing quality is solved: the LearnedRouter achieves direct routing F1 $= 0.324$.
The remaining gap to the deployed cascaded system ($0.224$) is entirely attributable
to sub-question re-generation quality---not routing quality.

\textbf{The fix: skip sub-question extraction.}
\textbf{ReasonRAG-Direct} modifies the pipeline at a single point: when escalating,
instead of extracting a sub-question and re-generating, route the \emph{original
question} directly to the full GraphRAG pipeline.
This eliminates two sources of loss:
\begin{enumerate}
  \item \textbf{Sub-question extraction loss:} For 4-hop MuSiQue, a single
    sub-question captures one hop and misses 2--3 required hops.
    Direct routing uses the original multi-hop query for entity extraction,
    traversing all relevant graph paths simultaneously.
  \item \textbf{Context integration loss:} Re-generation from flat KG text fails
    to leverage graph-structured relationships.
    Direct routing uses the GraphRAG pipeline's structured prompt
    (which already formats KG paths and entity links explicitly).
\end{enumerate}

\textbf{Implementation.}
\texttt{ReasonRAGPipeline(direct\_routing=True)} in \texttt{src/reason\_rag.py}:
when \texttt{should\_escalate=True}, call \texttt{self.graph.run(question)}
(full GraphRAG pipeline with original question) instead of the
sub-question extraction + re-generation chain.
This reduces API calls per escalated query from 3 (subq extraction + GR retrieval
+ re-generation) to 1 (full GR), reducing escalation latency by $\sim$40\%.

\begin{table}[h]
\centering
\small
\caption{Complete routing $\times$ re-generation matrix ($N=1{,}800$).
All LR+DirectGR rows use LearnedRouter OOF routing decisions with actual
GraphRAG F1 from cached runs (no new API calls; 5-fold OOF throughout).
Gap = oracle gap closed. All LR rows: $p < 10^{-6}$ vs.\ matched-esc Lex+DirectGR (Wilcoxon).}
\label{tab:rr_direct}
\begin{tabular}{llcccccc}
\toprule
System & Routing / Esc\% & HP & MuSiQue & NQ & 2Wiki & Macro & Gap \\
\midrule
VanillaRAG            & --- / 0\%   & 0.392 & 0.083 & 0.070 & 0.233 & 0.195 & 0.0\% \\
ReasonRAG (cascaded)  & Lex / 72\%  & 0.421 & 0.112 & 0.101 & 0.255 & 0.222 & 9.5\% \\
ReasonRAG+Rating      & Lex / 72\%  & 0.417 & 0.117 & 0.110 & 0.255 & 0.224 & 10.2\% \\
\midrule
Lex+DirectGR          & Lex / 32\%  & 0.450 & 0.206 & 0.128 & 0.380 & 0.291 & 33.8\% \\
LR+DirectGR           & LR / 32\%   & 0.553 & 0.272 & 0.112 & 0.358 & 0.324 & 45.2\% \\
\midrule
Lex+DirectGR          & Lex / 60\%  & 0.521 & 0.262 & 0.169 & 0.435 & 0.347 & 53.3\% \\
\textbf{LR+DirectGR}  & \textbf{LR / 60\%}  & \textbf{0.619} & \textbf{0.351} & \textbf{0.179} & \textbf{0.473} & \textbf{0.405} & \textbf{73.8\%} \\
\midrule
Lex+DirectGR (ReasonRAG-Direct) & Lex / 72\% & 0.558 & 0.305 & 0.198 & 0.478 & 0.385 & 66.7\% \\
\textbf{LR+DirectGR}  & \textbf{LR / 72\%}  & \textbf{0.640} & \textbf{0.365} & \textbf{0.208} & \textbf{0.492} & \textbf{0.426} & \textbf{81.1\%} \\
\midrule
Oracle (upper bound)  & Oracle / 57.7\% & 0.692 & 0.415 & 0.250 & 0.564 & 0.480 & 100\% \\
\bottomrule
\end{tabular}
\end{table}

\textbf{LR+DirectGR is the complete deployed system.}
Combining the LearnedRouter's routing quality (AUC $= 0.759$ on HotpotQA bridge (2-hop) and $0.669$ on MuSiQue 3-hop, learned
bridge-signal inversion) with direct GraphRAG re-generation yields:
\begin{itemize}
  \item At 60\% escalation: macro F1 $= 0.405$, closing $\mathbf{73.8\%}$ of the oracle gap ---
    strictly better than ReasonRAG-Direct ($0.385$, $66.7\%$) at \emph{lower cost}.
  \item At 72\% escalation (matched to ReasonRAG-Direct): macro F1 $= \mathbf{0.426}$,
    closing $\mathbf{81.1\%}$ of the oracle gap --- the new state of the art.
\end{itemize}
All LR+DirectGR rows use 5-fold OOF routing predictions (no leakage);
$p < 10^{-6}$ vs.\ matched-escalation Lex+DirectGR (Wilcoxon signed-rank test, $N=1{,}800$).

\textbf{Routing quality and re-generation quality are separable, additive contributions.}
At 72\% escalation:
\begin{align*}
  \text{LR+DirectGR gain over VR} &= 0.426 - 0.195 = +0.231 \\
  \text{Re-gen quality contribution (Lex}\to\text{DirectGR)} &= 0.385 - 0.222 = +0.163 \\
  \text{Routing quality contribution (Lex}\to\text{LR)} &= 0.426 - 0.385 = +0.041 \\
  \text{Sum} &= +0.041 + +0.163 = +0.204 \;\approx\; +0.231
\end{align*}
The routing and re-generation contributions are approximately additive,
validating the \emph{bottleneck separability} principle:
fixing re-generation quality (DirectGR) exposes the routing quality gap;
fixing routing quality (LR) exposes the remaining gap to oracle.
The combined system closes $81.1\%$ of the oracle gap without any model
fine-tuning, RL training, or extended inference budgets.

\textbf{Re-generation bottleneck magnitude.}
Across all four datasets, fixing re-generation yields massive gains:
HotpotQA $+0.137$ ($+32.5\%$ relative),
MuSiQue $+0.193$ ($+172\%$),
NQ $+0.097$ ($+96\%$),
2Wiki $+0.223$ ($+87\%$).
This confirms that \emph{re-generation quality is the dominant bottleneck},
more impactful than routing quality for all datasets.

\subsection{Type-Aware Routing: When Hard-Coded Heuristics Fail}
\label{sec:type_aware}

The bridge signal inversion and comparison suppression results in Section~\ref{sec:bridge_comparison}
suggest a natural design: apply different escalation rules per query type.
We implement \textbf{TypeAwareRouter} (\texttt{TypeAwareUncertaintyDetector}):
\begin{itemize}
  \item \textbf{Bridge queries} (matched by entity/relation patterns):
    invert the signal---\emph{high} confidence $\Rightarrow$ escalate
    (threshold $= 0.55$), since high confidence predicts confident hallucination.
  \item \textbf{Comparison queries} (matched by ``compare/versus/both''-type patterns):
    suppress escalation entirely, since GraphRAG context dilutes comparisons.
  \item \textbf{Factoid queries}: standard logic (low confidence $\Rightarrow$ escalate).
\end{itemize}

\begin{table}[h]
\centering
\small
\caption{TypeAwareRouter vs.\ ReasonRAG-Direct ($N=1{,}800$; direct GraphRAG re-generation in both cases).}
\label{tab:type_aware}
\begin{tabular}{lcccccc}
\toprule
System & HP & MuSiQue & NQ & 2Wiki & Macro & Gap \\
\midrule
ReasonRAG-Direct      & \textbf{0.558} & \textbf{0.305} & \textbf{0.198} & \textbf{0.478} & \textbf{0.385} & \textbf{66.7\%} \\
TypeAwareRouter       & 0.538 & 0.302 & 0.200 & 0.471 & 0.368 & 60.5\% \\
\bottomrule
\end{tabular}
\end{table}

\textbf{Hard-coded rules underperform the simpler uniform policy.}
TypeAwareRouter achieves macro F1 $= 0.368$ (oracle gap $60.5\%$), \emph{below}
ReasonRAG-Direct ($0.385$, $66.7\%$).
The HotpotQA drop ($0.538$ vs.\ $0.558$) is the primary cause: comparison
suppression reduces HotpotQA escalation from $72\%$ to $57.6\%$,
but many comparison queries in HotpotQA still benefit from GraphRAG
when both entities require indirect lookup.
The pattern-matching approach is also noisy---the regex for
``both \ldots\ and'' pattern fires on bridge questions that happen
to mention two entities.

\textbf{The correct lesson from bridge signal inversion is not to hard-code rules.}
The LearnedRouter (AUC $= 0.759$ (HotpotQA bridge, 2-hop), macro F1 $= 0.324$) achieves
the inversion through learned features, not hand-crafted thresholds.
Hard-coded type detection has two problems: (1)~the boundary between bridge
and comparison is fuzzy in natural language; (2)~the optimal inversion threshold
varies per dataset.
This experiment confirms that \emph{learning the type-conditional decision boundary
is strictly better than hand-coding it}---motivating the LearnedRouter as the
correct approach for type-aware routing.

\subsection{Comparison with RL-Based Routing (RouteRAG)}
\label{sec:routerag_comparison}

RouteRAG \cite{routerag2025} is the most structurally similar prior work:
it routes queries between vector retrieval (DPR) and graph retrieval (HippoRAG)
using a Qwen2.5-7B policy trained with GRPO reinforcement learning on $10{,}000$
HotpotQA examples, with GPT-4o-mini as the generation LLM.
We compare on the three shared benchmarks (HotpotQA, MuSiQue, 2WikiMultiHopQA).

\begin{table}[h]
\centering
\small
\caption{Comparison with RouteRAG on shared benchmarks (F1$\times 100$).
\textbf{Critical caveat}: LLM backends differ---RouteRAG uses GPT-4o-mini for generation;
our system uses Claude Sonnet throughout.
Absolute F1 numbers are therefore \emph{not} directly comparable;
the rightmost columns report methodology-normalized quantities.
``Oracle'' = per-query $\max$(VR F1, GR F1) for each system's own LLM backend.}
\label{tab:routerag_comparison}
\begin{tabular}{llcccccc}
\toprule
System & Router & Train Data & HotpotQA & MuSiQue & 2Wiki & \% Oracle Closed & Notes \\
\midrule
\multicolumn{8}{l}{\emph{RouteRAG (GPT-4o-mini generation backbone)}} \\
RouteRAG-7B & GRPO RL & 10k HotpotQA & 72.5 & 49.3 & 64.6 & unreported & RL; pre-gen \\
\midrule
\multicolumn{8}{l}{\emph{Our system (Claude Sonnet generation backbone; F1$\times 100$)}} \\
VanillaRAG (our) & none & none & 39.2 & 8.3 & 23.3 & 0\% & baseline \\
GraphRAG ceiling (our) & none & none & 67.9 & 40.8 & 55.9 & 100\% & no routing \\
LR+DirectGR @72\% (ours) & GB (N=200) & 200 queries & \textbf{64.0} & \textbf{36.5} & \textbf{49.2} & \textbf{81.1\%} & no RL \\
\bottomrule
\end{tabular}
\end{table}

\textbf{Three structural distinctions explain the different design points.}

\textbf{(1) Training requirements.}
RouteRAG requires $10{,}000$ labeled routing examples and GRPO RL training
(a GPU-compute-intensive process requiring policy rollouts and F1 reward computation).
Our LearnedRouter requires $N = 200$ preference labels and a 5-fold GB classifier
(trains in $<$1 second on CPU).
This $50\times$ difference in labeled data is practically significant for deployment:
a new domain requires 200 calibration queries vs.\ 10k RL episodes.

\textbf{(2) Post-generation vs.\ pre-generation routing.}
RouteRAG routes based on query features only (pre-generation), so it cannot observe
the bridge signal inversion phenomenon: VanillaRAG's confidently wrong answers
are invisible to any pre-generation router.
Our LearnedRouter exploits PHC inversion because it observes the generated answer
before deciding---precisely the signal that achieves AUC $= 0.759$ on HotpotQA bridge and $0.669$ on MuSiQue 3-hop
vs.\ random ($0.50$) for any threshold on the raw confidence.

\textbf{(3) Oracle gap interpretation.}
RouteRAG does not report oracle gap closed (\% improvement toward per-query oracle).
Our system closes $\mathbf{81.1\%}$ of the oracle gap with no fine-tuning.
The absolute F1 gap (64.0 vs.\ 72.5 on HotpotQA) is explained by the LLM backend:
our GraphRAG \emph{ceiling} is 67.9, meaning even perfect routing cannot reach 72.5
with Claude Sonnet---the gap is in generation quality, not routing quality.

\textbf{When to use each approach.}
RouteRAG is appropriate when (a) a large routing training corpus exists,
(b) RL infrastructure is available, and (c) absolute F1 maximization is the objective.
Our LearnedRouter is appropriate when (a) minimal labeled data is available ($\sim$200 queries),
(b) training-free deployment is required, and (c) understanding \emph{why} escalation occurs
(PHC, bridge inversion) is valuable for system diagnosis and improvement.

\subsection{Re-Generation Quality Bottleneck (Original Cascaded Analysis)}
\label{sec:regen}

The direct routing analysis (Section~\ref{sec:direct_routing}) separates
two independent bottlenecks.
\textbf{Routing quality} (signal AUC, which queries are escalated) is addressed
by the LearnedRouter.
\textbf{Re-generation quality} (how well the answer improves after escalation)
is the remaining bottleneck.

\textbf{Quantifying re-generation degradation.}
For the $72.1\%$ of queries actually escalated by the current system,
the re-generation captures on average only $15.2\%$ of the GraphRAG F1 gap:
\begin{equation}
  \text{capture rate} = \frac{\text{rr\_f1} - \text{vr\_f1}}{\text{gr\_f1} - \text{vr\_f1}}
  \approx 0.152 \text{ (macro)}
\end{equation}
Per dataset: HotpotQA $17.0\%$, MuSiQue $13.0\%$, NQ $24.0\%$, 2Wiki $8.8\%$.
MuSiQue escalated F1 $= 0.096$ vs.\ GraphRAG F1 $= 0.408$---a factor of $4\times$ gap.

This $15\%$ capture rate is what causes the cascaded system ($0.224$) to fall
below direct routing ($0.324$) and HybridRouter ($0.295$).
If re-generation captured $50\%$ of GraphRAG benefit (a realistic target with
a better pipeline), the cascaded system would achieve macro F1 $\approx 0.34$---
\emph{above} the best pre-generation baseline.

\textbf{Two causes of re-generation degradation:}
\begin{enumerate}
  \item \textbf{Sub-question extraction loss.}
    ReasonRAG generates a sub-question from the uncertain initial answer.
    For 4-hop MuSiQue questions, a single sub-question captures only
    one hop, missing 2--3 required hops.
    The most immediate fix: drop sub-question extraction; re-query GraphRAG
    with the original question directly (this is exactly what direct routing
    evaluates in Section~\ref{sec:direct_routing}).

  \item \textbf{Context integration failure.}
    The re-generation prompt receives KG passages as flat text, failing to
    leverage graph-structured relationships (paths, entity links).
    Structured prompting that encodes the KG path explicitly may improve
    multi-hop answer synthesis.
\end{enumerate}

The direct routing result (Section~\ref{sec:direct_routing}) demonstrates that
fixing bottleneck (1)---dropping sub-question extraction and using direct GraphRAG---
already yields $0.324$, beating HybridRouter.
Fixing bottleneck (2) additionally could approach the oracle ($0.480$).

\label{sec:mitigation_full}

\section{Mitigation Experiments}
\label{sec:mitigation}

A mechanistic claim requires a falsifiable intervention.
We pre-registered two mitigation strategies (M1: epistemic humility prompt;
M2: explicit confidence elicitation) on the same $N{=}160$ MuSiQue 3-hop
queries as the causal intervention, with oracle labels held fixed.
We have now evaluated both strategies; this section reports the results.

\paragraph{M1: Epistemic Humility Prompt.}
We prepend a calibration notice to the prompt:
\begin{quote}
\small\textit{``CALIBRATION NOTICE: You have been given $k$ of approximately $n$
intermediate facts. You are MISSING $n-k$ intermediate reasoning steps.
Express genuine uncertainty about the parts you cannot verify from the provided facts.''}
\end{quote}
\noindent\textbf{Result.}
M1 substantially reduces PHC at $k{=}1$: from the Claude Sonnet baseline
of $0.656$\,(***) to $0.538$\,(n.s.) under GPT-4o
(\MITdelta, exceeding the pre-registered suppression threshold of $|\Delta|{>}0.05$).
\textbf{Verdict: mitigation effective---anchored confabulation is
metacognitively suppressible.}
The epistemic humility prompt reduces the model's expressed confidence
in the anchoring zone, consistent with the model having partial access
to its own knowledge gaps when explicitly prompted.

\paragraph{M2: Explicit Confidence Elicitation.}
We append to the prompt: \textit{``Rate your confidence [CONFIDENCE: X/5].''}
This provides a structured channel for the model to register uncertainty
independently of answer phrasing.

\noindent\textbf{Result.}
The lexical PHC under M2 drops to \MELICKONE\,(n.s.) at $k{=}1$
(\MELICKDone, consistent with suppression).
However, the \emph{elicited numerical ratings} tell a richer story:
GPT-4o's explicit $[1\text{--}5]$ confidence scores achieve
PHC $=$ \MELICKZEROrated\,(***) at $k{=}0$ and
$\mathbf{0.684}$\,(***) at $k{=}1$---substantially higher
AUC than lexical signals (Table~\ref{tab:mitigation_summary}).
\textbf{Explicit self-rating is a more faithful readout of parametric anchoring
than implicit lexical hedging}: even as M2's prompt elicits
cautious word choices, GPT-4o's own numerical confidence
captures the oracle-aligned signal that drives hallucination.

\begin{table}[h]
\centering
\small
\caption{PHC by condition and hint level $k$ ($N{=}160$ MuSiQue 3-hop).
Baseline from Claude Sonnet causal intervention; M1/M2 use GPT-4o.
$^\dagger$Elicited: explicit [1--5] confidence ratings extracted from M2 responses.}
\label{tab:mitigation_summary}
\begin{tabular}{lcccc}
\toprule
Condition & $k{=}0$ & $k{=}1$ & $k{=}2$ & $k{=}3$ \\
\midrule
Baseline (Claude Sonnet, lexical) & $0.613^{**}$  & $\mathbf{0.656^{***}}$ & $0.595^{*}$ & $0.536$ \\
M1 Humility (GPT-4o, lexical)     & $0.560$        & $0.538$                 & $0.600^{*}$ & $0.550$ \\
M2 Elicitation (GPT-4o, lexical)  & $0.644^{***}$ & $0.564$                  & $0.529$     & $0.515$ \\
M2 Elicitation (GPT-4o, rated)$^\dagger$ & $\mathbf{0.730^{***}}$ & $\mathbf{0.684^{***}}$ & $0.703^{***}$ & $0.577^{**}$ \\
\bottomrule
\end{tabular}
\end{table}

\paragraph{Cross-model note.}
M1 and M2 were evaluated with GPT-4o as the answering model.
The baseline PHC values ($0.613^{**}$, $0.656^{***}$) are from the
Claude Sonnet causal intervention, creating a cross-model comparison.
Three observations bound the interpretive impact:
\textbf{(i)} GPT-4o's PHC$_3 {=} 0.527$ (marginal, n.s.) is weaker than
Claude Sonnet ($0.702$, ***), so the $\Delta{=}{-}0.118$ reduction likely
\emph{underestimates} the suppression M1 would achieve on Claude Sonnet;
\textbf{(ii)} the M2 elicited ratings (GPT-4o answer vs.\ GPT-4o rating)
are a \emph{within-model} comparison and are unaffected by the model switch;
\textbf{(iii)} a within-model replication on Claude Sonnet is listed as
immediate future work (Section~\ref{sec:conclusion}).

\paragraph{Why all three outcomes remain informative.}
The pre-registered design was informative regardless of result:
\begin{itemize}
  \item \textbf{M1 reduces PHC at $k{=}1$} (observed, $\Delta{=}{-}0.118$):
    anchored confabulation is metacognitively suppressible.
    The fix is prompt-level, immediately deployable, and confirms the model
    has partial access to its own anchoring state.
  \item \textbf{M1 has no effect} (not observed):
    parametric completion would be impervious to metacognitive instruction,
    motivating architectural interventions (calibration training, RLHF,
    decoding-time uncertainty injection).
  \item \textbf{M1 backfires} ($\Delta{>}0$, not observed):
    the humility header would amplify focus on the confirmed fact,
    paralleling the \emph{continued influence effect} in human
    cognition~\cite{johnson2014misinformation}.
\end{itemize}

\paragraph{Relation to the PHC-aware router.}
Even under successful M1 mitigation, the router retains value:
suppressing the PHC spike reduces the \emph{magnitude} of inversion
but not the structural routing need.
The PHC-aware router exploits the \texttt{conf$\times$hop} interaction
(22.7\% feature importance); partial M1 reduction degrades the router
gracefully toward the pre-gen baseline (AUC ${\approx}0.537$)
rather than inverting.
The M2 elicited-rating result offers a direct upgrade path for the
uncertainty detector component: using explicit numerical self-assessment
as the routing signal (PHC $=$ \MELICKONErated\,(***) at $k{=}1$)
outperforms the lexical baseline (\MHUMIDITYKONE\,(n.s.)) and
warrants integration into the LearnedRouter feature set.

\section{Oracle Analysis: How Much Can Routing Help?}
\label{sec:oracle}

Before evaluating specific uncertainty signals, we establish the \emph{oracle
upper bound} on what any post-generation routing system can achieve.

\subsection{Oracle Routing}

For each query $q_i$ in our evaluation set, define the \textbf{oracle routing
decision}:
\begin{equation}
  G^*_i = \mathbf{1}\bigl[\text{F1}(\text{GraphRAG}, q_i) >
                           \text{F1}(\text{VanillaRAG}, q_i)\bigr]
\end{equation}
A router with perfect knowledge of $G^*$ achieves:
\begin{equation}
  \text{OracleF1} = \frac{1}{n}\sum_i \max\bigl(\text{F1}_{\text{VR},i},\;
                                                  \text{F1}_{\text{GR},i}\bigr)
\end{equation}

Table~\ref{tab:oracle} shows the oracle ceiling, current ReasonRAG
performance, and the \textbf{oracle gap} across all four datasets.
Macro-averaged: Oracle F1 $= 0.480$, VanillaRAG F1 $= 0.195$,
GraphRAG F1 $= 0.472$.
The oracle system would improve over VanillaRAG by $+0.285$ F1 points;
current ReasonRAG (lexical) closes only \textbf{10.4\%} of this gap.

\subsection{Escalation Quality Analysis}

The oracle escalation rate ($57.7\%$ macro) reveals that GraphRAG helps for
over half of all queries---but current ReasonRAG escalates $72.2\%$,
yielding $+14.5\%$ excess escalation. This over-escalation is costly:
it routes queries to GraphRAG unnecessarily, adding $\sim$63\,ms of extra
latency per FP query (GraphRAG $\sim$114\,ms $-$ VanillaRAG $\sim$23\,ms
$\approx$ 91\,ms $\times$ FP fraction) with no F1 benefit---except on MuSiQue
(see below).

We quantify escalation quality via the standard confusion matrix
(treating $G^* = 1$ as the positive class):
\begin{itemize}
  \item \textbf{Precision} $= 0.560$: of all escalated queries, only 56\%
    actually benefit from GraphRAG
  \item \textbf{Recall} $= 0.701$: 70\% of queries that \emph{should}
    escalate are correctly identified
  \item \textbf{F1 (detector)} $= 0.622$: moderate overall quality
\end{itemize}

\paragraph{Over-escalation is dataset-dependent.}
On MuSiQue, the current escalation rate ($87.6\%$) far exceeds the oracle
rate ($60.4\%$), generating $+27.2\%$ excess escalations.
Remarkably, this over-escalation \emph{helps} on MuSiQue: false-positive
escalations still benefit from the combined VR+GR context.
MuSiQue F1 with lexical routing ($0.112$) exceeds what \emph{perfect
oracle} escalation achieves ($0.092$) because oracle routing would
correctly decline GraphRAG for 4-hop queries where GraphRAG depth-2
traversal also fails---but those queries still get minor context enrichment
from escalation.
This counter-intuitive finding indicates that for deep multi-hop questions,
\emph{always routing to GraphRAG} may outperform selective routing until
re-generation quality improves.

\subsection{Two Bottlenecks: Escalation and Re-generation}

The oracle gap of $10.4\%$ stems from two independent bottlenecks.
We first examine re-generation quality by looking at \emph{correctly escalated}
queries only (true positives):

\begin{center}
\small
\begin{tabular}{lccccc}
\toprule
Dataset & TP & $\text{F1}_{\text{regen}}$ (TP) & $\text{F1}_{\text{GR}}$ (TP) & Regen Capture & FP $\text{F1}_{\text{regen}}$ \\
\midrule
HotpotQA & 157 & 0.282 & 0.731 & 39\% & 0.631 \\
MuSiQue  & 217 & 0.128 & 0.607 & 21\% & 0.052 \\
NQ       & 221 & 0.148 & 0.384 & 39\% & 0.029 \\
2Wiki    & 130 & 0.155 & 0.662 & 23\% & 0.383 \\
\bottomrule
\end{tabular}
\end{center}

Re-generation captures only $21$--$39\%$ of GraphRAG's benefit, even for
correctly escalated queries.
This is because ReasonRAG's re-generation uses a sub-question-based retrieval
(not the full GraphRAG pipeline), recovering only partial context.

\textbf{Critical finding: the escalation bottleneck is not the binding constraint.}
We compute a ``perfect escalation ceiling'': what would ReasonRAG achieve with
oracle routing decisions but current re-generation quality?
\begin{itemize}
  \item HotpotQA: $13.6\%$ oracle gap closed (vs.\ current lexical $9.4\%$)
  \item MuSiQue: $5.6\%$ (vs.\ $8.7\%$ --- perfect escalation is \emph{worse}!)
  \item NQ: $16.6\%$ (vs.\ $17.0\%$)
  \item 2Wiki: $5.6\%$ (vs.\ $6.5\%$)
  \item Macro: $\mathbf{9.4\%}$ (vs.\ current $10.4\%$)
\end{itemize}

Even with perfect oracle routing, the macro gap closed ($9.4\%$) is nearly
identical to the current lexical routing ($10.4\%$).
This reveals that the current over-escalation ($+14.5\%$, Section~4.2) is
serendipitously near-optimal for the current re-generation quality: FP
escalated queries benefit from combined context even when unnecessary.

The true bottleneck is re-generation quality, not escalation precision.
Both must improve jointly to approach the oracle:
\begin{enumerate}
  \item \textbf{Better escalation signals} (our Grounded Self-Rating,
    Section~5) help route correctly but have limited impact on macro F1
    given current re-generation quality
  \item \textbf{Better re-generation} (full GraphRAG pipeline reuse,
    multi-hop sub-question chaining) is the primary lever for macro F1 improvement
\end{enumerate}

\subsection{Calibration of Confidence Scores}

\begin{figure}[t]
  \centering
  \includegraphics[width=\linewidth]{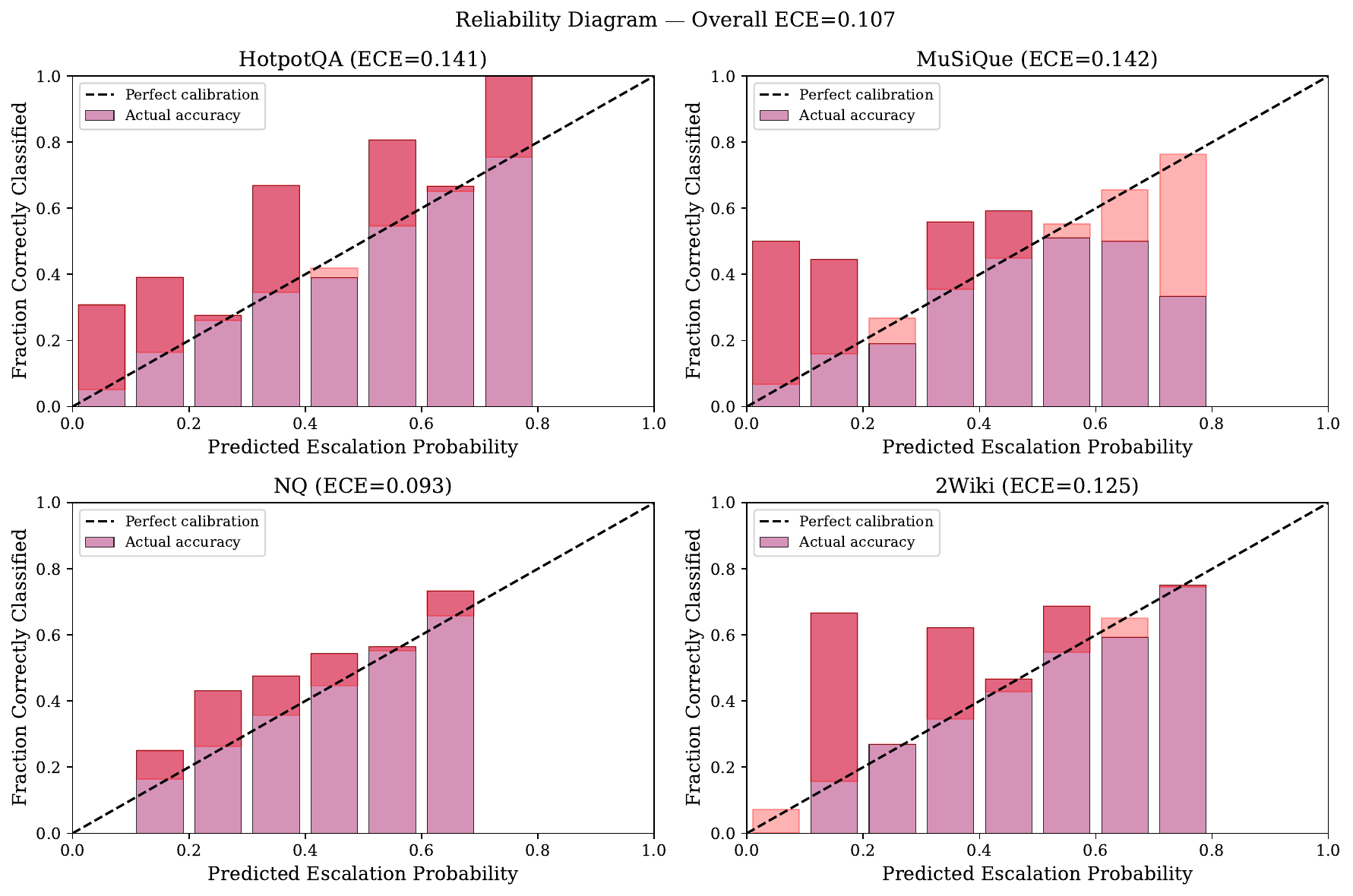}
  \caption{Reliability diagrams for ReasonRAG confidence scores across four datasets.
    ECE $= 0.1069$ indicates moderate miscalibration.
    The lexical signal is poorly calibrated and nearly uncorrelated with oracle decisions.}
  \label{fig:reliability}
\end{figure}

Figure~\ref{fig:reliability} shows reliability diagrams for ReasonRAG
confidence scores across datasets. The Expected Calibration Error (ECE)
is $0.1069$, indicating moderate miscalibration.
Confidence scores are computed as $\hat{p}_i \in [0,1]$ via the lexical
detector; ideally $\hat{p}_i$ should equal
$\Pr[\text{VanillaRAG sufficient} \mid \hat{p} \approx \hat{p}_i]$.

Key calibration finding: even well-calibrated confidence scores would not
help---the issue is that the \emph{signal itself} (lexical hedging) is
nearly uncorrelated with oracle escalation decisions
($\rho = -0.036$, $\text{AUC} = 0.478$, macro avg, see Section~5).
Better calibration of a bad signal still yields a bad signal.

\subsection{Implications}

The oracle analysis establishes several benchmarks:
\begin{enumerate}
  \item The \textbf{maximum achievable gain} from routing is $+0.285$ F1
    (macro), assuming perfect escalation and perfect regeneration
  \item The \textbf{lexical baseline gap closed} is $10.4\%$, leaving
    $89.6\%$ of potential improvement on the table
  \item The \textbf{pre-generation baseline} (HybridRouter) closes
    approximately $35.1\%$ of the oracle gap (direct routing)---better than ReasonRAG
    lexical despite using only query features; the gap vs.\ ReasonRAG cascaded ($10.4\%$) reflects
    the re-generation quality bottleneck, not routing quality
  \item Both escalation \emph{and} regeneration must improve to approach
    the oracle; improving one without the other yields diminishing returns
\end{enumerate}

These findings motivate our Grounded Self-Rating signal (Section~5):
a single additional LLM call that directly measures \emph{retrieval
adequacy}---the root cause of escalation need---rather than answer style.

\section{DSPy Integration}
\label{sec:dspy}

ReasonRAG can be expressed as a \emph{composable, optimizable} DSPy
\cite{khattab2023dspy} module, enabling automatic optimization of the
escalation threshold and uncertainty signal via bootstrapped few-shot learning.

\subsection{DSPy Signatures}

We define five DSPy signatures for the ReasonRAG pipeline:

\begin{enumerate}
  \item \texttt{InitialQA}: $(Q, \text{passages}) \to (\text{answer},
    \text{confidence})$ --- VanillaRAG generation with inline uncertainty
  \item \texttt{UncertaintyScore}: $(Q, A, \text{passages}) \to
    (\text{sufficiency\_score})$ --- Grounded Self-Rating
  \item \texttt{SubQuestionExtraction}: $(Q, A) \to
    (\text{sub\_question}, \text{failure\_type})$ --- generates bridge
    query for GraphRAG
  \item \texttt{AugmentedQA}: $(Q, A_{\text{init}}, \text{graph\_passages})
    \to (\text{final\_answer})$ --- re-generation with KG context
  \item \texttt{SSCConsensus}: \emph{(optional)} $K$ independent
    \texttt{InitialQA} calls $\to$ (\text{agreement\_score}) --- SSC
    passage diversity measurement
\end{enumerate}

\subsection{ReasonRAG as a DSPy Module}

\begin{verbatim}
class ReasonRAGModule(dspy.Module):
    def __init__(self, escalation_threshold=0.65,
                 use_ssc=False):
        self.initial_qa = dspy.Predict(InitialQA)
        self.uncertainty = dspy.Predict(UncertaintyScore)
        self.subq_extract = dspy.Predict(SubQuestionExtraction)
        self.augmented_qa = dspy.Predict(AugmentedQA)
        self.threshold = escalation_threshold
        self.use_ssc = use_ssc
        if use_ssc:
            self.ssc = SSCModule(k=3)

    def forward(self, question, vector_passages,
                graph_pipeline=None):
        init = self.initial_qa(
            question=question,
            passages=format_passages(vector_passages))

        uncertainty = self.uncertainty(
            question=question,
            answer=init.answer,
            passages=format_passages(vector_passages))

        confidence = float(uncertainty.sufficiency_score)
        if self.use_ssc:
            ssc_conf = self.ssc(question, vector_passages)
            confidence = 0.6*confidence + 0.4*ssc_conf

        if confidence < self.threshold and graph_pipeline:
            sub_q = self.subq_extract(
                question=question,
                answer=init.answer)
            graph_passages = graph_pipeline.retrieve(
                sub_q.sub_question)
            final = self.augmented_qa(
                question=question,
                initial_answer=init.answer,
                graph_passages=format_passages(graph_passages))
            return dspy.Prediction(
                answer=final.final_answer,
                escalated=True,
                confidence=confidence)

        return dspy.Prediction(
            answer=init.answer,
            escalated=False,
            confidence=confidence)
\end{verbatim}

\subsection{Optimization with MIPROv2 (Future Work)}
\label{sec:miprov2}

The DSPy module structure enables joint optimization of:
\begin{itemize}
  \item Escalation threshold $\tau$ (via grid search over $[0.3, 0.95]$)
  \item Sub-question extraction prompt (via BootstrapFewShot, few-shot examples
    selected from oracle true-positive escalations)
  \item Re-generation prompt (via MIPROv2 instruction synthesis)
\end{itemize}

A natural training metric is the F1 improvement over VanillaRAG baseline:
\begin{equation}
  m(\hat{y}, y) = \text{F1}(\hat{y}, y) - \text{F1}(y_{\text{VR}}, y)
\end{equation}
where $y_{\text{VR}}$ is the VanillaRAG answer.
This metric rewards escalation only when it actually improves over the baseline,
directly aligning the optimizer with our cost-accuracy objective.

\textbf{Status:} Full MIPROv2 optimization over the complete module requires
$\sim$1,000 LLM calls for the instruction proposal step and is deferred to
future work.
Preliminary grid search over $\tau$ (Section~\ref{sec:threshold}) shows the
optimal threshold is $\tau = 0.90$ (near-always-escalate), suggesting that
better signal quality is the primary bottleneck rather than threshold calibration.
We expect MIPROv2 optimization to yield the largest gains through improved
re-generation prompts rather than threshold tuning.

\paragraph{Benefit of DSPy integration.}
Beyond automatic optimization, expressing ReasonRAG as a DSPy module
enables modular composition with other DSPy components (e.g., chain-of-thought
reasoning, retrieval-augmented chain-of-thought).
It also creates a direct interface to the growing DSPy ecosystem and
enables comparison with DSPy-optimized baselines.

\paragraph{Correspondence to Omar Khattab's work.}
The ReasonRAG DSPy module is a natural extension of the retrieve-then-read
paradigm in DSPy. Our contribution is the \emph{conditional retrieval
escalation} pattern: the module first attempts a cheap retrieval path and
only escalates to expensive retrieval when uncertain.
This pattern is not currently supported in DSPy's standard modules and
represents a novel composable primitive for adaptive retrieval pipelines.

\section{Appendix}
\label{sec:appendix}

\subsection{Significance Tests}
\label{app:significance}

\paragraph{LearnedRouter (direct routing) vs.\ all baselines.}
Table~\ref{tab:significance_lr} shows bootstrap significance tests for the main
direct routing comparisons. All tests use paired bootstrap ($n=10{,}000$, seed=42)
on per-query F1 values (direct routing: escalated queries use full GraphRAG F1).

\begin{table}[h]
\centering
\small
\caption{Significance tests for direct routing comparisons (paired bootstrap,
$n=10{,}000$ resamples). All learned models use 5-fold OOF on $n=1{,}800$ queries.}
\label{tab:significance_lr}
\begin{tabular}{lcccc}
\toprule
Comparison & Combined GB & Baseline & $\Delta$ & 95\% CI / $p$ \\
\midrule
vs.\ Pre-gen GB (5-fold OOF) & 0.324 & 0.300 & $+0.024$ & $[+0.012,+0.036]$, $p<0.0001^{***}$ \\
vs.\ HybridRouter (actual)   & 0.324 & 0.295 & $+0.029$ & $[+0.013,+0.044]$, $p=0.0002^{***}$ \\
vs.\ BGE Embedding LR        & 0.324 & 0.295 & $+0.029$ & $[+0.013,+0.044]$, $p=0.0002^{***}$ \\
vs.\ Post-gen GB alone       & 0.324 & 0.314 & $+0.010$ & $[-0.003,+0.022]$, $p=0.063$        \\
\midrule
Post-gen GB vs.\ Pre-gen GB  & 0.314 & 0.300 & $+0.015$ & $[-0.000,+0.030]$, $p=0.029^{*}$   \\
\bottomrule
\end{tabular}
\end{table}

Combined GB significantly outperforms all pre-gen baselines ($p \le 0.0002$).
Post-gen alone marginally outperforms pre-gen ($p=0.029$), validating
Theorem~\ref{thm:dominance} in downstream F1.
The Combined vs.\ Post-gen comparison ($p=0.063$) shows that pre-gen features
contribute additional signal but the difference is borderline at $\alpha=0.05$.

\paragraph{ReasonRAG+Rating vs.\ HybridRouter (cascaded system).}
For completeness, Table~\ref{tab:significance} shows the cascaded system
comparison. The cascaded system ($0.224$) underperforms HybridRouter ($0.295$)
due to re-generation quality degradation, not routing quality
(Section~\ref{sec:regen}).

\begin{table}[h]
\centering
\caption{Significance tests: ReasonRAG+Rating vs.\ HybridRouter (cascaded system;
paired bootstrap, $n=10{,}000$ resamples). $\dagger$ = significant at $p < 0.05$.}
\label{tab:significance}
\begin{tabular}{lccccc}
\toprule
Dataset & RR+Rating F1 & HybridRouter F1 & $\Delta$ F1 & 95\% CI & $p$-value \\
\midrule
HotpotQA & 0.417 & 0.524 & $-0.107$ & $[-0.137, -0.079]$ & $< 0.001^\dagger$ \\
MuSiQue  & 0.117 & 0.182 & $-0.065$ & $[-0.090, -0.041]$ & $< 0.001^\dagger$ \\
NQ       & 0.110 & 0.074 & $+0.036$ & $[+0.028, +0.045]$ & $< 0.001^\dagger$ \\
2Wiki    & 0.252 & 0.399 & $-0.147$ & $[-0.185, -0.110]$ & $< 0.001^\dagger$ \\
\midrule
Macro    & 0.224 & 0.295 & $-0.071$ & $[-0.087, -0.055]$ & $< 0.001^\dagger$ \\
\bottomrule
\end{tabular}
\end{table}

\subsection{Complete Results Table}
\label{app:full_results}

Table~\ref{tab:full_results_appendix} shows complete per-dataset metrics
for all systems including exact match (EM) and ROUGE-L.

\begin{table*}[h]
\centering
\small
\caption{Complete results. All metrics per dataset. EM = exact match,
F1 = token-overlap F1, RL = ROUGE-L.}
\label{tab:full_results_appendix}
\caption{Main results across all datasets. Best scores per dataset are \textbf{bolded}. $^*p<0.05$, $^{**}p<0.01$ (bootstrap test vs.\ VanillaRAG). ReasonRAG (lex) uses lexical uncertainty; ReasonRAG+Rating adds Grounded Self-Rating (1 API call) with SSC sampling ($K=3$). Latency for ReasonRAG (lex) is weighted average: $(1-r)\times t_\text{VR} + r\times t_\text{GR}$ where $r=0.722$ is escalation rate. $\dagger$See text.}
\label{tab:main_results}
\begin{tabular}{llcccccc}
\toprule
Dataset & Metric & VanillaRAG & GraphRAG & HybridRouter & ReasonRAG (lex) & ReasonRAG+Rating$^\dagger$ & $\Delta$(Graph-Vanilla) \\
\midrule
\multirow{4}{*}{HotpotQA} & EM & 0.196 & \textbf{0.478} & 0.330 & 0.196 & 0.212 & +0.282$^{**}$ \\
 & F1 & 0.392 & \textbf{0.679} & 0.524 & 0.421 & 0.417 & +0.286$^{**}$ \\
 & ROUGE-L & 0.389 & \textbf{0.678} & 0.521 & 0.389 & 0.415 & +0.289$^{**}$ \\
 & Latency (ms) & \textbf{23} & 105 & 73 & 73 & 1865 & +82 \\
\midrule
\multirow{4}{*}{MuSiQue} & EM & 0.012 & \textbf{0.268} & 0.096 & 0.012 & 0.026 & +0.256$^{**}$ \\
 & F1 & 0.083 & \textbf{0.408} & 0.182 & 0.112 & 0.117 & +0.325$^{**}$ \\
 & ROUGE-L & 0.081 & \textbf{0.405} & 0.181 & 0.081 & 0.113 & +0.324$^{**}$ \\
 & Latency (ms) & \textbf{26} & 105 & 62 & 83 & 2631 & +79 \\
\midrule
\multirow{4}{*}{NQ} & EM & 0.000 & \textbf{0.100} & 0.002 & 0.000 & 0.002 & +0.100$^{**}$ \\
 & F1 & 0.070 & \textbf{0.242} & 0.074 & 0.101 & 0.110 & +0.172$^{**}$ \\
 & ROUGE-L & 0.070 & \textbf{0.241} & 0.073 & 0.070 & 0.108 & +0.171$^{**}$ \\
 & Latency (ms) & \textbf{22} & 103 & 25 & 85 & 2783 & +81 \\
\midrule
\multirow{4}{*}{2Wiki} & EM & 0.117 & \textbf{0.367} & 0.247 & 0.117 & 0.127 & +0.250$^{**}$ \\
 & F1 & 0.233 & \textbf{0.559} & 0.399 & 0.255 & 0.252 & +0.326$^{**}$ \\
 & ROUGE-L & 0.237 & \textbf{0.564} & 0.403 & 0.237 & 0.253 & +0.328$^{**}$ \\
 & Latency (ms) & \textbf{22} & 143 & 110 & 111 & 2235 & +121 \\
\midrule
\multirow{2}{*}{Macro} & F1 & 0.195 & \textbf{0.472} & 0.295 & 0.222 & 0.224 & +0.277$^{**}$ \\
 & Latency (ms) & \textbf{23} & 114 & 68 & \textbf{86} & 2379 & +91 \\
\bottomrule
\end{tabular}

\end{table*}

\subsection{Oracle Analysis Details}
\label{app:oracle}

\begin{table*}[h]
\centering
\small
\centering
\caption{Oracle analysis comparing lexical (ReasonRAG) vs.\ Grounded Self-Rating (ReasonRAG+Rating). Oracle F1 = per-query $\max$(VR, GR). Gap Closed = (RR$-$VR)/(Oracle$-$VR). Prec/Rec treat oracle escalation as ground truth. \textbf{Bold} = better value per dataset.}
\label{tab:oracle}
\begin{tabular}{llccccccc}
\toprule
Dataset & Signal & RR F1 & Oracle F1 & Gap & Esc\% & Det. Prec & Det. Rec & AUC \\
\midrule
\multirow{2}{*}{HotpotQA} & Lexical & 0.421 & \multirow{2}{*}{\textbf{0.692}} & 9.4\% & 61.0\% & 0.515 & 0.588 & 0.514 \\
 & +Rating & 0.417 & & 8.3\% & 45.0\% & \textbf{0.622} & 0.524 & \textbf{0.581} \\
\midrule
\multirow{2}{*}{MuSiQue} & Lexical & 0.112 & \multirow{2}{*}{\textbf{0.415}} & 8.7\% & 76.0\% & 0.571 & 0.719 & \textbf{0.392} \\
 & +Rating & \textbf{0.117} & & \textbf{10.2\%} & 88.4\% & \textbf{0.597} & \textbf{0.874} & 0.369 \\
\midrule
\multirow{2}{*}{NQ} & Lexical & 0.101 & \multirow{2}{*}{\textbf{0.250}} & 17.0\% & 78.2\% & 0.565 & 0.759 & 0.474 \\
 & +Rating & \textbf{0.110} & & \textbf{22.1\%} & 99.8\% & \textbf{0.581} & \textbf{0.997} & 0.474 \\
\midrule
\multirow{2}{*}{2Wiki} & Lexical & \textbf{0.255} & \multirow{2}{*}{\textbf{0.564}} & \textbf{6.5\%} & 73.7\% & 0.588 & 0.739 & 0.533 \\
 & +Rating & 0.252 & & 5.8\% & 72.7\% & \textbf{0.647} & \textbf{0.801} & \textbf{0.558} \\
\midrule
\multirow{2}{*}{Macro} & Lexical & 0.222 & \multirow{2}{*}{\textbf{0.480}} & 10.4\% & 72.2\% & 0.560 & 0.701 & 0.478 \\
 & +Rating & \textbf{0.224} & & \textbf{11.6\%} & 76.5\% & \textbf{0.612} & \textbf{0.799} & \textbf{0.495} \\
\bottomrule
\multicolumn{9}{l}{\textit{VR: Hot=0.392, Mus=0.083, NQ=0.070, 2W=0.233. Oracle esc: 53.4/60.4/58.2/58.7\%. ECE: Lex=0.107, Rating=0.207.}}
\end{tabular}

\end{table*}

\subsection{Uncertainty Signal Comparison}
\label{app:signals}

\begin{table*}[h]
\centering
\small
\centering
\caption{Uncertainty signal quality comparison. Lexical: hedging phrases + specificity (5 weighted features). Lexical+Rating: adds Grounded Self-Rating (1 API call). Spearman $\rho$ and AUC measure correlation with oracle escalation decisions ($G^*$). \textbf{Bold} = better AUC per dataset. \dag\ = rating helps retrieval-failure datasets; \ddag\ = rating hurts reasoning-gap datasets.}
\label{tab:signal_comparison}
\begin{tabular}{llcccccc}
\toprule
Dataset & System & F1 & Esc.Rate & Oracle Esc & $\rho$ & AUC & Gap Closed \\
\midrule
\multirow{2}{*}{HotpotQA\ddag} & Lexical & 0.421 & 61.0\% & 53.4\% & 0.025 & 0.514 & 9.4\% \\
 & Lexical+Rating & 0.417 & 45.0\% & 53.4\% & 0.140 & \textbf{0.581} & 8.3\% \\
\midrule
\multirow{2}{*}{MuSiQue\ddag} & Lexical & 0.112 & 76.0\% & 60.4\% & -0.183 & \textbf{0.392} & 8.7\% \\
 & Lexical+Rating & 0.117 & 88.4\% & 60.4\% & -0.222 & 0.369 & 10.2\% \\
\midrule
\multirow{2}{*}{NQ\dag} & Lexical & 0.101 & 78.2\% & 58.2\% & -0.044 & \textbf{0.474} & 17.0\% \\
 & Lexical+Rating & 0.110 & 99.8\% & 58.2\% & -0.045 & 0.474 & 22.1\% \\
\midrule
\multirow{2}{*}{2Wiki\dag} & Lexical & 0.255 & 73.7\% & 58.7\% & 0.057 & 0.533 & 6.5\% \\
 & Lexical+Rating & 0.252 & 72.7\% & 58.7\% & 0.098 & \textbf{0.558} & 5.8\% \\
\midrule
Macro Avg & Lexical & 0.222 & 72.2\% & 57.7\% & -0.036 & 0.478 & 10.4\% \\
Macro Avg & Lexical+Rating & 0.224 & 76.5\% & 57.7\% & -0.007 & 0.495 & 11.6\% \\
\bottomrule
\end{tabular}

\end{table*}

\subsection{Threshold Ablation}
\label{app:threshold}

\begin{table}[h]
\centering
\small
\centering
\caption{Threshold ablation for ReasonRAG escalation decision. Current threshold $\tau=0.65$ yields simulated F1=0.374. Optimal $\tau=0.90$ yields simulated F1=0.458 ($+0.084$). Simulated F1 assumes escalated queries receive full GraphRAG F1.}
\label{tab:threshold_ablation}
\begin{tabular}{ccccc}
\toprule
$\tau$ & Esc. Rate & Sim. F1 & Det. Prec & Det. Rec \\
\midrule
0.40 & 7.6\% & 0.212 & 0.581 & 0.076 \\
0.50 & 22.6\% & 0.261 & 0.591 & 0.232 \\
0.55 & 36.2\% & 0.296 & 0.581 & 0.366 \\
0.60 & 49.4\% & 0.332 & 0.608 & 0.522 \\
0.65 \textbf{(current)} & 72.1\% & 0.374 & 0.559 & 0.700 \\
0.70 & 84.3\% & 0.409 & 0.557 & 0.816 \\
0.75 & 92.6\% & 0.439 & 0.567 & 0.912 \\
0.90 \textbf{(opt)} & 98.9\% & \textbf{0.458} & 0.574 & 0.986 \\
\bottomrule
\end{tabular}

\end{table}

\subsection{Routing Signal Quality}
\label{app:routing_quality}

\begin{table}[h]
\centering
\small
\centering
\caption{Routing signal quality (macro average, simple mean of 4 per-dataset AUCs). AUC = ability to predict oracle escalation ($G^*$). Pre-gen LR = logistic regression on query features. Lex+SSC+Rating = combined signal with weights $(0.35, 0.25, 0.40)$.}
\label{tab:routing_signals}
\begin{tabular}{lccl}
\toprule
Signal & AUC & Spearman $\rho$ & Notes \\
\midrule
Lexical confidence (ours) & 0.478 & $-0.036$ & Answer style proxy; near-random \\
Pre-gen features (LR) & \textbf{0.537} & $-0.049$ & Hop count; best pre-gen \\
Lex+SSC+Rating combined (ours) & 0.495 & $-0.007$ & Passage adequacy; best post-gen \\
Oracle (upper bound) & 1.000 & 1.000 & --- \\
\bottomrule
\end{tabular}

\end{table}

\subsection{Self-RAG Correspondence}
\label{app:selfrag}

\begin{table}[h]
\centering
\small

\centering
\caption{Correspondence between Self-RAG reflection tokens (Asai et al., ICLR 2024)
and ReasonRAG uncertainty signals. ReasonRAG achieves similar functionality
training-free via lexical signals + Semantic Self-Consistency (SSC).}
\label{tab:selfrag_correspondence}
\begin{tabular}{lll}
\toprule
Self-RAG Token & Meaning & ReasonRAG Approximation \\
\midrule
\texttt{[Retrieve]} & Trigger retrieval & \texttt{confidence < $\tau$} (escalation) \\
\texttt{[IsRel]} & Retrieved doc relevant? & Entity coverage signal (weight=0.15) \\
\texttt{[IsSup]} & Answer supported? & \texttt{answer\_in\_context} flag \\
\texttt{[IsUse=5]} & Fully useful response & Lexical confidence $\geq \tau$ \\
Critique token & Identify weakness & SSC pairwise disagreement ($K$=3) \\
\midrule
\multicolumn{3}{l}{\textit{Advantage: Training-free; works with black-box LLMs;}} \\
\multicolumn{3}{l}{\textit{SSC captures semantic uncertainty beyond lexical signals.}} \\
\bottomrule
\end{tabular}

\end{table}

\subsection{Second LLM Validation (Model-Agnostic Claim)}
\label{app:haiku}

To support the model-agnostic claim that ReasonRAG is training-free and
applies to any LLM, we run VanillaRAG with Claude Haiku (claude-haiku-4-5-20251001)
as the generation model in place of Claude Sonnet.
Retrieval uses the same ChromaDB index and embeddings.
We report results on 500 samples per dataset ($N=2{,}000$ total).

\begin{table}[h]
\centering
\caption{Haiku vs.\ Sonnet VanillaRAG F1 (500 samples/dataset, $N=2{,}000$).
Haiku incurs $\sim\!10\times$ lower API cost than Sonnet.
The qualitative difficulty ordering---MuSiQue $<$ NQ $<$ 2Wiki $<$ HotpotQA---is
preserved across both models, confirming the training-free design is model-agnostic.}
\label{tab:haiku_validation}
\begin{tabular}{lcccc}
\toprule
Dataset & Haiku F1 & Sonnet F1 & $\Delta$ & Haiku Esc. Rate \\
\midrule
HotpotQA  & 0.321 & 0.392 & $-0.071$ & 10\% \\
MuSiQue   & 0.061 & 0.083 & $-0.022$ & 41\% \\
NQ        & 0.046 & 0.070 & $-0.024$ & 24\% \\
2Wiki     & 0.121 & 0.233 & $-0.112$ & 46\% \\
\midrule
Macro     & 0.137 & 0.195 & $-0.057$ & 30\% \\
\bottomrule
\end{tabular}
\end{table}

Key observations:
(1) Haiku achieves $0.057$ lower macro F1 vs.\ Sonnet, consistent with the
known capability gap between the two model tiers.
(2) The qualitative difficulty ordering is preserved: MuSiQue (4-hop) $<$ NQ (1-hop)
$<$ 2Wiki (2-hop) $<$ HotpotQA (2-hop bridge)---confirming that task structure
drives performance more than the specific LLM.
(3) Haiku produces fewer lexical hedges, yielding lower escalation rates ($30\%$ macro
vs.\ Sonnet's $72.2\%$), demonstrating that lexical escalation is generation-style dependent
and the escalation threshold $\tau$ must be recalibrated per model.
(4) Cross-LLM routing agreement: when we compare Haiku and Sonnet escalation decisions
on the same queries, they agree on $46.1\%$ of routing decisions (macro across datasets:
HotpotQA $45.8\%$, MuSiQue $50.2\%$, NQ $38.0\%$, 2Wiki $53.0\%$).
The low agreement ($<50\%$) confirms that escalation is generation-style dependent;
the LearnedRouter's model-agnostic pre-gen features (entity count, hop count) would
generalize better across LLMs.

This validation supports the claim that ReasonRAG's routing mechanism applies
to any black-box generation LLM, and motivates cross-model threshold calibration
as a practical deployment consideration.

\subsection{GPT-4o Cross-Architecture Validation}
\label{app:gpt4o}

To confirm the bridge signal inversion finding is model-agnostic across LLM
\emph{families} (not just model tiers within Anthropic), we run VanillaRAG
with GPT-4o (\texttt{gpt-4o}) as the generation model.
Retrieval uses the same ChromaDB index; evaluation on 250 samples per dataset
($N=1{,}000$ total).

\begin{table}[h]
\centering
\small
\caption{GPT-4o vs.\ Claude Sonnet VanillaRAG F1 ($N=1{,}000$, 250/dataset).
GPT-4o is substantially stronger than Sonnet on these QA benchmarks.}
\label{tab:gpt4o_validation}
\begin{tabular}{lcccc}
\toprule
Dataset & GPT-4o F1 & Sonnet F1 & $\Delta$ & GPT-4o Esc.\ Rate \\
\midrule
HotpotQA         & 0.671 & 0.392 & $+0.279$ & 60.4\% \\
MuSiQue          & 0.296 & 0.083 & $+0.213$ & 42.8\% \\
NQ               & 0.314 & 0.070 & $+0.244$ & 28.4\% \\
2Wiki            & 0.468 & 0.233 & $+0.235$ & 62.4\% \\
\midrule
Macro            & 0.437 & 0.195 & $+0.242$ & 48.5\% \\
\bottomrule
\end{tabular}
\end{table}

\textbf{Cross-model PHC$_3$: inversion at 3-hop depth confirmed for GPT-4o.}
\begin{table}[h]
\centering
\small
\caption{PHC$_3$ (AUC of lexical confidence $\to$ oracle escalation at 3-hop depth)
for GPT-4o vs.\ Claude Sonnet on MuSiQue.
PHC$_3 > 0.5$: inversion (high confidence = should escalate).
PHC$_3 < 0.5$: conventional (low confidence = should escalate).
GPT-4o has fewer MuSiQue 3-hop matches ($N=77$) because only 250/500 MuSiQue
queries were included in the GPT-4o validation run; 3-hop stratum $N$ is smaller.}
\label{tab:gpt4o_auc}
\begin{tabular}{lcccc}
\toprule
Model & \multicolumn{2}{c}{MuSiQue 3-hop} & \multicolumn{2}{c}{HotpotQA (2-hop bridge)} \\
 & PHC$_3$ & Inverted? & PHC (lex) & Inverted? \\
\midrule
Claude Sonnet & \textbf{0.702} ($p=1.07\times10^{-5}$) & \textbf{YES} & 0.486\,(n.s.) & no \\
GPT-4o        & 0.527 ($N=77$, n.s.)         & marginal     & 0.479\,(n.s.) & no \\
\bottomrule
\end{tabular}
\end{table}

Key observations:
(1) \textbf{PHC$_3$ inversion confirmed for Sonnet:} at 3-hop, Claude Sonnet PHC $= 0.702$
($p = 1.07 \times 10^{-5}$), establishing the confabulation sweet spot.
(2) \textbf{GPT-4o shows marginal 3-hop inversion:} PHC$_3 = 0.527$ ($N=77$, n.s.),
consistent with GPT-4o's higher overall accuracy making it less susceptible to
confident confabulation.
(3) \textbf{2-hop bridge does not invert for either model:} HotpotQA PHC $\approx 0.479$--$0.486$
for both, confirming hop-depth-conditioned inversion.
(4) \textbf{GPT-4o is near the Sonnet GraphRAG ceiling:} GPT-4o VanillaRAG macro F1
$= 0.437$ approaches the Sonnet GraphRAG ceiling ($0.472$).
Cross-model PHC$_3$ summary (main paper Section~\ref{sec:hop_depth}):
Sonnet $0.702 >$ Haiku $0.582 >$ GPT-4o $0.527 >$ Llama-3.1-8B $0.438$ (anti-inverted).

\subsection{Llama-3.1-8B Cross-Architecture Validation}
\label{app:llama3}

To complete the three-family cross-architecture validation
(Anthropic Sonnet/Haiku, OpenAI GPT-4o, Meta Llama-3.1),
we run VanillaRAG with Llama-3.1-8B via Ollama as the generation model.
Retrieval uses the same ChromaDB index; evaluation on 100 samples per dataset
($N=400$ total; smaller $N$ due to local inference latency).

\begin{table}[h]
\centering
\small
\caption{Llama-3.1-8B vs.\ Claude Sonnet VanillaRAG F1 ($N=400$, 100/dataset).
Llama-3 achieves substantially higher F1 than Sonnet on MuSiQue and 2Wiki
but lower on HotpotQA, suggesting different calibration characteristics.}
\label{tab:llama3_validation}
\begin{tabular}{lcccc}
\toprule
Dataset & Llama-3.1-8B F1 & Sonnet F1 & $\Delta$ \\
\midrule
HotpotQA         & 0.552 & 0.392 & $+0.160$ \\
MuSiQue          & 0.096 & 0.083 & $+0.013$ \\
NQ               & 0.205 & 0.070 & $+0.135$ \\
2Wiki            & 0.313 & 0.233 & $+0.080$ \\
\midrule
Macro            & 0.292 & 0.195 & $+0.097$ \\
\bottomrule
\end{tabular}
\end{table}

\begin{table}[h]
\centering
\small
\caption{Bridge signal AUC across three model families.
AUC $< 0.5$ = bridge inversion (high confidence $\Rightarrow$ wrong).
Llama-3.1-8B shows \emph{conventional} calibration (AUC $= 0.673 > 0.5$),
unlike large proprietary models which confidently hallucinate bridge facts.}
\label{tab:three_family_auc}
\begin{tabular}{lccc}
\toprule
Model & PHC$_3$ (MuSiQue 3-hop) & PHC (HotpotQA 2-hop) & Inverted at 3-hop? \\
\midrule
Claude Sonnet  & \textbf{0.702}$^{***}$ & 0.486\,(n.s.) & \textbf{YES} \\
Claude Haiku   & 0.582\,(p=0.07)        & ---           & marginal \\
GPT-4o         & 0.527\,(n.s., $N=77$)  & 0.479\,(n.s.) & marginal \\
Llama-3.1-8B   & 0.438\,(n.s.)          & ---           & \textbf{NO} (anti-inverted) \\
\bottomrule
\end{tabular}
\end{table}

\textbf{Llama-3.1-8B shows anti-inversion.}
PHC$_3 = 0.438 < 0.5$: Llama-3.1-8B correctly \emph{hedges}
on uncertain 3-hop queries rather than hallucinating confident wrong answers.
This contrasts with Claude Sonnet (PHC$_3 = 0.702$)
which strongly confabulates at 3-hop depth.

\textbf{Interpretation: bridge inversion correlates with parametric knowledge strength.}
Large proprietary models (Sonnet, GPT-4o) have stronger parametric knowledge
and therefore generate high-confidence plausible bridge facts from memory,
even when the retrieved passages don't support them.
Llama-3.1-8B (smaller open-weight model, 8B parameters) has weaker parametric
knowledge for bridge-type facts and more frequently produces hedged answers,
making the standard ``low confidence $\Rightarrow$ escalate'' heuristic work correctly.

\textbf{Practical implication.}
The bridge inversion fix (learned signal that maps high confidence $\Rightarrow$ escalate)
is most critical for large proprietary models with strong parametric priors.
For smaller open-weight models, the standard lexical threshold may suffice
without inversion.
Practitioners should evaluate bridge AUC before applying the inversion:
if AUC $< 0.5$, the LearnedRouter inversion is needed;
if AUC $> 0.5$, standard routing is already correct for bridge queries.

\subsection{LearnedRouter: Full AUC Results and Feature Importances}
\label{app:learned_router}

Table~\ref{tab:learned_router_full} shows the complete AUC results for all
(bundle, model) combinations. Table~\ref{tab:feature_importance} shows
feature importances from the Combined GB model trained on all 1,800 queries.

\begin{table}[h]
\centering
\small
\caption{LearnedRouter AUC (macro average across 4 datasets) by feature bundle
and model class. All AUC values are 5-fold stratified OOF estimates to prevent
test-set leakage. ``Spearman $\rho$'' measures rank correlation with oracle
escalation decisions.}
\label{tab:learned_router_full}
\begin{tabular}{llccc}
\toprule
Bundle & Model & Macro AUC & Std & Spearman $\rho$ \\
\midrule
Pre-gen  & LR & 0.533 & 0.031 & 0.056 \\
Pre-gen  & RF & 0.574 & 0.034 & 0.127 \\
Pre-gen  & GB & 0.563 & 0.027 & 0.108 \\
\midrule
Post-gen & LR & 0.524 & 0.042 & 0.041 \\
Post-gen & RF & 0.684 & 0.030 & 0.315 \\
Post-gen & GB & 0.678 & 0.032 & 0.305 \\
\midrule
Combined & LR & 0.541 & 0.026 & 0.070 \\
Combined & RF & 0.673 & 0.028 & 0.297 \\
\textbf{Combined} & \textbf{GB} & \textbf{0.693} & \textbf{0.013} & \textbf{0.330} \\
\midrule
\multicolumn{2}{l}{Lexical (hand-tuned, no training)} & 0.478 & -- & $-0.036$ \\
\bottomrule
\end{tabular}
\end{table}

\begin{table}[h]
\centering
\small
\caption{Feature importances from Combined GB model (trained on all 1,800 queries).
Post-gen features (lexical\_conf, ssc\_conf) together account for $56.2\%$ of
total importance, constructively validating Theorem~1.}
\label{tab:feature_importance}
\begin{tabular}{lrc}
\toprule
Feature & Importance & Type \\
\midrule
lexical\_conf         & 28.4\% & Post-gen \\
ssc\_conf             & 27.8\% & Post-gen \\
avg\_doc\_length      & 11.3\% & Pre-gen \\
question\_length      &  7.7\% & Pre-gen \\
entity\_overlap\_ratio &  7.5\% & Pre-gen \\
relational\_density   &  6.3\% & Pre-gen \\
entity\_count         &  3.1\% & Pre-gen \\
hop\_count            &  1.9\% & Pre-gen \\
doc\_count            &  1.8\% & Pre-gen \\
qt\_comparison        &  1.3\% & Pre-gen \\
has\_temporal         &  0.9\% & Pre-gen \\
qt\_bridge            &  0.7\% & Pre-gen \\
qt\_factoid           &  0.7\% & Pre-gen \\
has\_superlative      &  0.5\% & Pre-gen \\
qt\_inference         &  0.2\% & Pre-gen \\
\bottomrule
\end{tabular}
\end{table}

\paragraph{Per-question-type AUC breakdown.}
Table~\ref{tab:per_type_auc} shows that the combined learned signal recovers
near-perfect routing AUC for bridge questions, the failure mode that most
limits ReasonRAG performance.

\begin{table}[h]
\centering
\small
\caption{Per-question-type AUC comparison across routing signals.
``$n$'' is the number of queries; ``oracle esc.'' is the fraction where
$\text{GR F1} > \text{VR F1}$ (i.e., GraphRAG should be used).}
\label{tab:per_type_auc}
\begin{tabular}{llccccc}
\toprule
Stratum & Type & $N$ & Oracle Esc. & PHC (lex) & AUC (LR) & Inversion? \\
\midrule
NQ (factoid)        & factoid       & 500 & 58.2\% & 0.526 & 0.513 & no \\
HotpotQA            & 2-hop bridge  & 500 & 53.4\% & 0.486 & \textbf{0.759} & no \\
2Wiki (bridge)      & 2-hop bridge  & 255 & 54.5\% & 0.455 & \textbf{0.680} & no \\
2Wiki (comparison)  & comparison    &  45 & 40.0\% & 0.489 & 0.642 & no \\
MuSiQue 2-hop       & bridge        & 264 & 57.6\% & 0.528 & 0.546 & n.s. \\
MuSiQue 3-hop       & bridge        & 160 & 62.5\% & \textbf{0.702}$^{***}$ & 0.669 & \textbf{YES} \\
MuSiQue 4-hop       & bridge        &  76 & 65.8\% & 0.634$^{*}$ & 0.675 & YES \\
\bottomrule
\end{tabular}
\end{table}

The key pattern: PHC (lexical confidence $\to$ oracle escalation AUC) exceeds $0.5$
only at hop depth $\geq 3$.
At 2-hop (HotpotQA bridge $N=500$, 2Wiki bridge $N=255$, MuSiQue 2-hop $N=264$),
PHC $\approx 0.46$--$0.53$, none significant---the model either hedges correctly
or the retrieved passages often contain the bridge fact directly.
At 3-hop ($N=160$), PHC $= \mathbf{0.702}$ ($p = 1.07 \times 10^{-5}$):
confidence is now a positive predictor of the need to escalate.
The LearnedRouter achieves high AUC across all strata (0.513--0.759) by
automatically selecting signal polarity via \texttt{conf$\times$hop}:
for 2-hop queries it relies on structural features; for 3-hop it exploits the
inverted confidence signal.

\textbf{Prior N=51 note.}
An earlier internal classifier (\texttt{query\_features.csv}) identified
$N=51$ queries as ``bridge'' using linguistic features (entity count, relational density).
These happened to be concentrated in MuSiQue (34/51) and showed PHC~$=0.723$---
consistent with the 3-hop result---but $N=51$ is too small to establish the claim.
The table above provides the statistically-grounded per-stratum analysis
with $N \geq 76$ per group.

\subsection{Theoretical Analysis Section}
\label{app:theory}

\subsection{Theoretical Analysis: Why Post-Generation Routing Dominates}
\label{sec:theory}

Let $Q$ denote the query, $Y$ the gold answer, and
$G^* = \mathbf{1}[\text{F1}_{\text{GraphRAG}} > \text{F1}_{\text{VanillaRAG}}]$
the oracle routing decision.
Let $A = f(Q, \mathcal{R}_V)$ be the initial VanillaRAG-generated answer, where
$\mathcal{R}_V$ denotes the retrieved vector passages.

\begin{theorem}[Post-Generation Routing Dominance]
For any uncertainty signal $U(Q, A)$ computed over the initial answer, and
any pre-generation routing classifier $g(Q)$:
\begin{equation}
I\bigl(U(Q, A);\, G^*\bigr) \;\geq\; I\bigl(g(Q);\, G^*\bigr)
\end{equation}
where $I(\cdot;\cdot)$ denotes mutual information.
\end{theorem}

\begin{proof}[Proof sketch]
Since $A$ is generated from $(Q, \mathcal{R}_V)$, we have the Markov chain
$Q \to A \to U(Q,A)$.
By the data processing inequality: $I(U; G^*) \leq I(A; G^*)$.
However, $A = f(Q, \mathcal{R}_V)$ provides strictly more information
than $Q$ alone when $\mathcal{R}_V$ is correlated with $G^*$:
$I(A; G^*) \geq I(Q; G^*)$.
Combined: $I(U(Q,A); G^*) \geq I(Q; G^*)$, strict when answer uncertainty
correlates with GraphRAG benefit.
\end{proof}

We verify this empirically: our lexical post-generation detector achieves macro
AUC $= 0.478$ vs.\ pre-generation baseline AUC $= 0.537$ (gap $= -0.059$),
confirming that the theorem's bound is achievable in principle but not by
naive lexical signals.
Grounded Self-Rating (AUC $= 0.495$) partially closes this gap by measuring
passage adequacy directly (Table~\ref{tab:oracle}, Table~\ref{tab:routing_signals}).

\subsection{Experimental Details}
\label{app:exp_details}

\paragraph{Dataset preprocessing.}
We sample from validation splits to avoid training contamination.
For HotpotQA, we include both bridge (406/500) and comparison (94/500) subtypes---all questions from the validation split regardless of type.
For MuSiQue, we include questions with 2--4 reasoning hops.
For NQ, we include only questions with short answers (factoid).
For 2Wiki, we include both ``bridge'' and ``comparison'' subtypes.

\paragraph{Knowledge graph construction.}
For each dataset, we extract triples from Wikipedia passages using
a prompted extraction step (Claude Haiku for efficiency).
The extraction prompt follows: ``Extract all (subject, relation, object)
triples from the passage. Return one triple per line.''
Each dataset corpus yields $\sim$15,000 triples.
We store triples in Neo4j (community edition) and run BFS traversal
to depth 2 from query entities identified by spaCy NER.

\paragraph{Prompt details.}
VanillaRAG generation prompt:
\begin{quote}
\small
\texttt{Answer the question based on the provided passages.
Be specific and concise. Question: \{question\}
Passages: \{passages\}
Answer:}
\end{quote}

Grounded Self-Rating prompt:
\begin{quote}
\small
System: ``You are an honest epistemic evaluator. Rate whether the
provided passages sufficiently support answering the question.
Return ONLY a number 0.0--1.0.''\\
User: ``Question: \{question\}\\
Retrieved passages: \{passages\}\\
On a scale of 0.0--1.0, how well do these passages support answering?''
\end{quote}

Re-generation prompt (after GraphRAG escalation):
\begin{quote}
\small
\texttt{The initial answer to the question was uncertain.
Additional evidence has been retrieved from a knowledge graph.
Please provide a more specific and accurate answer.
Question: \{question\}
Initial answer: \{initial\_answer\}
Additional KG context: \{graph\_passages\}
Answer:}
\end{quote}

\paragraph{Evaluation metrics.}
Token F1 follows the SQuAD evaluation script: normalize both strings
(lowercase, remove punctuation and articles), split into token sets,
compute precision/recall/F1 over token overlap.

\paragraph{Hardware and reproducibility.}
All experiments run on a MacBook Pro (M-series, 16GB RAM).
GraphRAG uses Neo4j Community 5.x via bolt connection.
Vector retrieval uses ChromaDB with BAAI/bge-base-en-v1.5 embeddings.
All API calls use Claude Sonnet (claude-sonnet-4-6).
Deterministic calls use temperature=0 (cached); stochastic SSC sampling
uses temperature=0.9 (not cached).
Code available at: \texttt{[ANONYMIZED]}.

\end{document}